\RequirePackage{fix-cm}
\documentclass[twocolumn]{svjour3}          
\smartqed  
\usepackage{subfigure}
\usepackage{graphicx}
\usepackage{booktabs}
\usepackage{caption}

\usepackage{amscd}
\usepackage{amssymb}
\usepackage{amsmath}
\usepackage{mathrsfs}

\graphicspath{{./Figures/}}

\def\d{\mathrm{d}}

\DeclareMathOperator*{\argmax}{arg\,max}
\DeclareMathOperator*{\erf}{erf}
\DeclareMathOperator*{\Real}{Re}
\DeclareMathOperator*{\Imag}{Im}
\DeclareMathOperator*{\Imagpos}{Im^+}
\DeclareMathOperator*{\cond}{cond}

\newcommand{\Gh}{\gothic{h}}

\newcommand{\tU}{\tilde{U}}
\newcommand{\tW}{\tilde{W}}

\newcommand{\cR}{\mathcal{R}}
\newcommand{\cA}{\mathcal{A}}
\newcommand{\cB}{\mathcal{B}}
\newcommand{\cW}{\mathcal{W}}
\newcommand{\cF}{\mathcal{F}}
\newcommand{\cV}{\mathcal{V}}
\newcommand{\cN}{\mathcal{N}}

\newcommand{\R}{\mathbb{R}}
\newcommand{\C}{\mathbb{C}}
\newcommand{\LL}{\mathbb{L}}

\newcommand{\ve}{\mathbf{e}}
\newcommand{\vc}{\mathbf{c}}
\newcommand{\vd}{\mathbf{d}}

\newcommand{\vx}{\mathbf{x}}
\newcommand{\vy}{\mathbf{y}}

\newcommand{\vn}{\mathbf{n}}
\newcommand{\vomega}{\boldsymbol{\omega}}

\newcommand{\mR}{\mathbf{R}}

\newcommand{\atg}{|_{g}}
\newcommand{\gNoPar}{g}

\spnewtheorem{result}{Result}{\bf}{\it}

\newcommand{\cFRRR}{\mathcal{F}}

\newcommand*\conj[1]{\overline{#1}}

\DeclareMathAlphabet\gothic{U}{euf}{m}{n}
\twocolumn

\begin{document}

\title{Design and Processing of Invertible Orientation Scores of 3D Images for Enhancement of Complex Vasculature}

\titlerunning{Design and Processing of Invertible Orientation Scores}

\author{M.H.J. Janssen \and A.J.E.M. Janssen \and E.J. Bekkers \and J. Oliv\'an Besc\'os \and R. Duits}

\authorrunning{Janssen, Janssen, Bekkers, Oliv\'an Besc\'os, Duits}

\institute{
		M.H.J. Janssen \at
		CASA,
		Eindhoven University of Technology, \\
		Tel.: +31-40-2478897,        %
		\email{M.H.J.Janssen@tue.nl}
	\and
		A.J.E.M. Janssen \at  Eindhoven University of Technology, \\ Department of Mathematics and Computer Science, \\
		\email{A.J.E.M.Janssen@tue.nl}  \\
	\and
		E.J. Bekkers \at
		CASA,
		Eindhoven University of Technology, \\
		\email{E.J.Bekkers@tue.nl}
	\and
		J. Oliv\'an Besc\'os \at  Philips, Interventional X-ray, Eindhoven \\
		\email{Javier.olivan.bescos@philips.com}  \\
	\and
		R. Duits \at  CASA,
		Eindhoven University of Technology, \\
		Tel.: +31-40-2472859,
		\email{R.Duits@tue.nl}  \\
	 }


\maketitle

\begin{abstract}
The enhancement and detection of elongated structures in noisy image data is relevant for many biomedical imaging applications. To handle complex crossing structures in 2D images, 2D orientation scores $U: \mathbb{R} ^ 2\times S ^ 1 \rightarrow \mathbb{C}$ were introduced, which already showed their use in a variety of applications. Here we extend this work to 3D orientation scores $U: \mathbb{R} ^ 3 \times S ^ 2\rightarrow \mathbb{C}$. First, we construct the orientation score from a given dataset, which is achieved by an invertible coherent state type of transform. For this transformation we introduce 3D versions of the 2D cake-wavelets, which are complex wavelets that can simultaneously detect oriented structures and oriented edges. Here we introduce two types of cake-wavelets, the first uses a discrete Fourier transform, the second is designed in the 3D generalized Zernike basis, allowing us to calculate analytical expressions for the spatial filters. Finally, we show two applications of the orientation score transformation. In the first application we propose an extension of crossing-preserving coherence enhancing diffusion via our invertible orientation scores of 3D images which we apply to real medical image data. In the second one we develop a new tubularity measure using 3D orientation scores and apply the tubularity measure to both artificial and real medical data.
\end{abstract}

\section{Introduction}

The enhancement and detection of elongated structures is important in many biomedical image analysis applications. These tasks become problematic when multiple elongated structures cross or touch each other in the data. In these cases it is useful to work with multi-orientation representations of image data. Such multi-orientation representations can be made using various techniques, such as invertible orientation scores (which is obtained via a coherent state transform) \cite{Ali2000,DuitsQAM2010Part1,BekkersJMathImagingVis2014,ThesisFranken,IshamJMathPhys1991,BarbieriAnalAppl2014}, continuous wavelet transforms \cite{DuitsIJCV2007,DuitsQAM2010Part1,SharmaACHA2015,BekkersJMathImagingVis2014}, orientation lifts \cite{ZweckJMathImagingVis2004,BoscainSIAMJImagingSci2014}, or orientation channel representations \cite{Felsberg2012}. Here we focus on constructing an invertible orientation score. In order to separate the crossing or touching structures (Fig. \ref{fig:2DOS}), we extend the domain of the data to include orientation. This is achieved by correlating our 3D data $f:\R^ 3\rightarrow \R$ with a set of oriented filters to construct a 3D orientation score $U:\R^ 3\times S ^ 2 \rightarrow \C$. As the transformation between image and orientation score is stable, due to our design of anisotropic wavelets, we can robustly relate operators on the score to operators on images. To take advantage of the multi-orientation decomposition, we consider operators on orientation scores, and process our data via orientation scores (Fig. \ref{fig:OverviewOperators}).

Regarding the invertibility of the transform from image to orientation score, we note that in comparison to continuous wavelet transforms (see e.g. \cite{Lee1996,Antoine1999,Mallat1999,Louis1997} and many others) on the group of 3D rotations, translations and scalings, we use all scales simultaneously and exclude the scaling group from the wavelet transform and its adjoint, yielding a coherent state type of transform \cite{AliJMathPhys1998}. This makes it harder to design appropriate wavelets, but has the computational advantage of only needing one all-scale transformation.

The 2D orientation scores have already showed their use in a variety of applications. In \cite{FrankenIJCV2009,SharmaACHA2015} the orientation scores were used to perform crossing-preserving coherence-enhancing diffusions. These diffusions greatly reduce the noise in the data, while preserving the elongated crossing structures. Next to these generic enhancement techniques, the orientation scores also showed their use in retinal vessel tracking \cite{BekkersJMathImagingVis2014,BekkersSIAMJImagingSci2015,Chen2014}, in vessel segmentation \cite{Zhang2016} and biomarker analysis \cite{Bekkers2015,Rajan2016}, where they were used to better handle crossing vessels. Here we aim to extend such techniques to 3D data.

To perform detection and enhancement operators on the orientation score we first need to transform a given greyscale image or 3D dataset to an orientation score in an invertible way. In previous works various wavelets were introduced to perform a 2D orientation score transform. Some of these wavelets did not allow for an invertible transformation (e.g. Gabor wavelets \cite{Lee1996}). A wavelet that allows an invertible transformation was proposed by Kalitzin \cite{KalitzinIJCV1999}. A generalization of these wavelets was found by Duits \cite{ThesisDuits} who derived formal unitarity results and expressed the wavelets in a basis of eigenfunctions of the harmonic oscillator. This type of wavelet was also extended to 3D. This wavelet, however, has some unwanted properties such as poor spatial localization (oscillations) and the fact that the maximum of the wavelet does not lie at its center. In \cite{ThesisDuits} a class of 2D cake-wavelets were introduced, that have a cake-piece shaped form in the Fourier domain. The cake-wavelets simultaneously detect oriented structures and oriented edges by constructing a complex orientation score $U:\R ^ 2 \times S^1 \rightarrow \C$. Because the family of rotated cake-wavelets cover the full Fourier spectrum, invertibility is guaranteed.

In this article we propose a 3D version of the cake-wavelets. A preliminary attempt to generalize these filters was done in \cite{JanssenSSVM2015}, where the plate-detectors in \cite{ThesisDuits} were extended to complex-valued cake-wavelets with a line-detector in the real part. Compared to these previous works, the filters in this work are now exact until sampling in the Fourier domain. For these filters we have no analytical description in the spatial domain as filters are obtained via a discrete inverse Fourier transform (DFT). Therefore we additionally consider expressing filters of this type in the 3D generalized Zernike basis. For this basis we have analytical expressions for the inverse Fourier transform, allowing us to find analytical expressions for the spatial filters. This has the additional advantage that they allow for a steerable implementation. These analytical expressions are then used to validate the filters obtained using the DFT method. We also show applications of these filters and the orientation score transformation in 3D vessel analysis. That is, we present crossing-preserving diffusions for denoising 3D rotational Xray of blood vessels in the abdomen and we present a tubularity measure via orientation scores and features based on this tubularity measure, which we apply to cone beam CT data of the brain. An overview of the applications is presented in Fig. \ref{fig:OverviewApplications}.
Regarding our non-linear diffusions of 3D rotational Xray images via invertible orientation scores, we observe that complex geometric structures in the vasculature (involving multiple orientations) are better preserved than with non-linear diffusion filtering directly in the image domain. This is in line with previous findings for non-linear diffusion filtering of 2D images \cite{FrankenIJCV2009} and related works \cite{Scharr2006,steidl2009anisotropic,muhlich2009analysis} that rely on other more specific orientation decompositions. 

For the sake of general readability we avoid Lie group theoretical notations, until Section 6.1 where it is strictly needed.
Let us nevertheless mention that our work fits in a larger Lie group theoretical framework, see for example \cite{Ali2000,FuhrBook2005,DuitsSSVM2007,ThesisDuits,BatardJMathImagingVis2014} that has many applications in image processing.
Besides the special cases of the Heisenberg group \cite{BarbieriJMIV2014,PetitotJPhysiolParis2003,DuitsApplComputHarmonAnal2013}, the 2D Euclidean motion group \cite{CittiJMathImagingVis2006,BarbieriAnalAppl2014,Mashtakov2013,Prandi2015,CittiSIAMJImagingSci2016,DuitsQAM2010Part1,ThesisBekkers}, the similitude group \cite{Antoine1999,SharmaACHA2015,Pechaud2009} and the 3D rotation group \cite{MashtakovJMathImagingVis2017},
we now consider invertible orientation scores  on the 3D Euclidean motion group (in which the coupled space of positions and orientations is embedded),
which relate to coherent states from physics \cite{IshamJMathPhys1991} for $n=3$,  with a specific (semi-)analytic design for our image processing purposes.

\subsection{Contributions of the Article}
The main contributions per section of the article are:
\begin{itemize}
	\item In Section \ref{sect:OSTransform} we give an overview of the discrete and continuous 3D orientation score transformation. Additionally we present a transformation which is splitted in low and high frequencies and quantify the stability of the transformation  in Lemma \ref{lemma:StabilitySplitting}.

	\item In Section \ref{sect:WaveletDFT} we present the cake-wavelets obtained using the DFT method and give an efficient implementation using spherical harmonics which is summarized in Result \ref{res:Result1}. Furthermore we analyze the stability of the transformation for these filters (Proposition \ref{proposition:FastReconstruction}).

	\item In Section \ref{sect:WaveletZernike} we present the analytical versions of the cake-wavelets obtained by expansion in the Zernike basis followed by a continuous Fourier transform which is summarized in Result \ref{res:Result2}.

	\item In Section \ref{sect:Experiments} we compare the two types of filters and show that the DFT filters approximate their analytical counterparts well.

	\item In Section \ref{sect:Applications} we show two applications of the orientation score transformation:
	\begin{enumerate}
	 	\item We propose an extension of coherence enhancing diffusion via our invertible orientation scores of 3D images. Compared to the original idea of coherence enhancing diffusion acting directly on image-data \cite{WeickertIJCV1999,BurgethBook2009,Burgeth2012}, there is the advantage of preserving crossings. Here we applied our method to real medical image data (3D rotational Xray) of the abdomen containing renal arteries. We show quantitatively that our method effectively reduces noise (quantified using contrast to noise ratios (CNR)) while preserving the complex vessel geometry and the vessel widths. Furthermore, qualitative assessment indicates that our denoising method is very useful as preprocessing for 3D visualization (volume rendering).
	 	\item We develop a new tubularity measure in 3D orientation score data. This extends previous work on tubularity measures using 2D orientation scores \cite{Chen2014}\cite[Ch. 12]{ThesisBekkers} to 3D data. We show qualitatively that our measure gives sharp responses at vessel centerlines and show its use for radius extraction and complex vessel segmentation in cone beam CT data of the brain.
	 \end{enumerate}
\end{itemize}

\subsection{Outline of the Article}

First, we discuss the theory of invertible orientation score transforms in Section \ref{sect:OSTransform}. Then we construct 3D cake-wavelets and give a new efficient implementation using spherical harmonics in Section \ref{sect:WaveletDFT}, followed by their analytical counterpart expressed in the generalized Zernike basis in Section \ref{sect:WaveletZernike}. Then we compare the two types of filters and validate the invertibility of the orientation score transformation in Section \ref{sect:Experiments}. Finally, we address two application areas for 3D orientation scores in Section \ref{sect:Applications} and show results and practical benefits for both of them. In the first application (Subsection \ref{ssect:CEDOS}), we present a natural extension of the crossing preserving coherence enhancing diffusion on invertible orientation scores (CEDOS) \cite{FrankenIJCV2009} to the 3D setting. In the second application (Subsection \ref{ssect:Tubularity}) we use the orientation score to define a tubularity measure, and show experiments applying the tubularity measure to both synthetic data and brain-CT data. Both application sections start with a treatment of related methods.

\begin{figure}[ht]
\includegraphics[width=.99 \columnwidth]{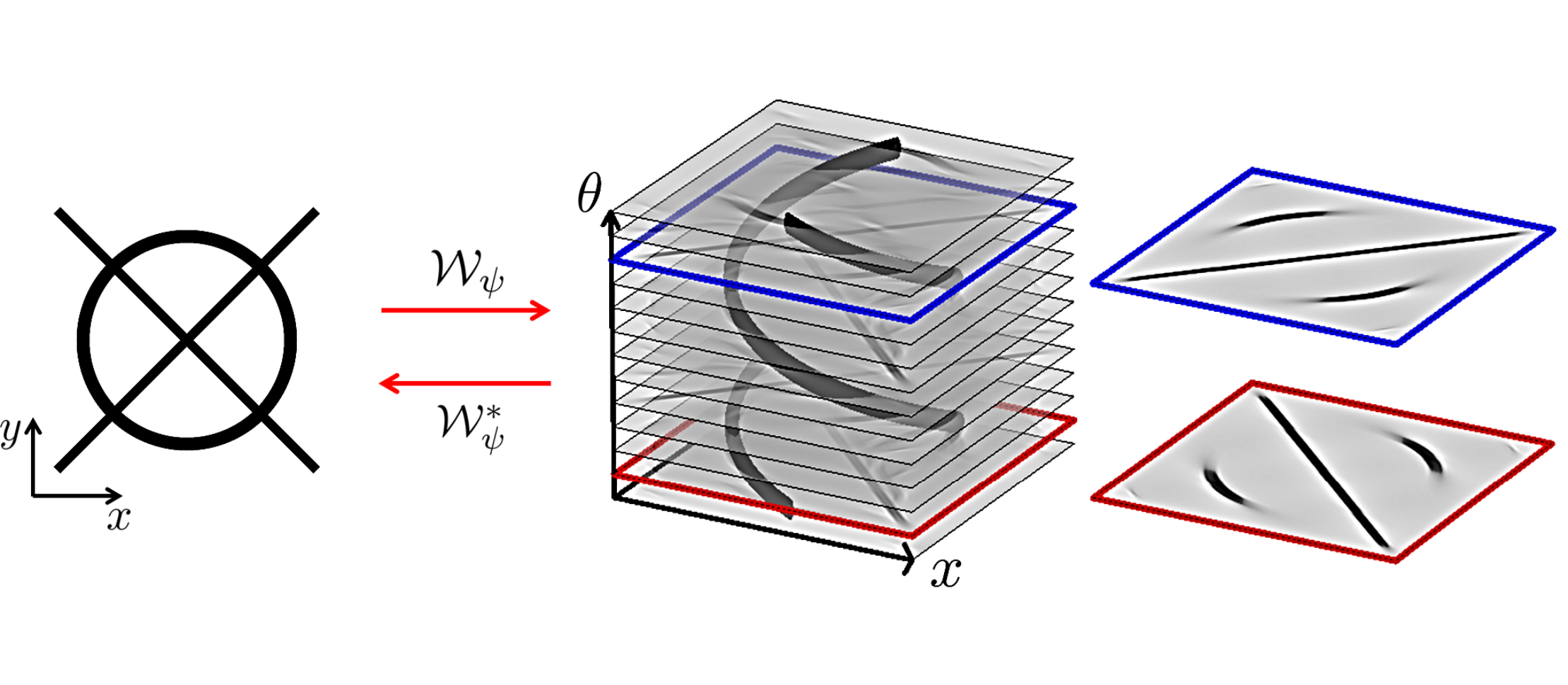}
	\caption{2D Orientation score for an exemplary image. In the orientation score crossing structures are disentangled because the different structures have a different orientation.}
	\label{fig:2DOS}
\end{figure}

\begin{figure}[ht]
\includegraphics[width=.99 \columnwidth]{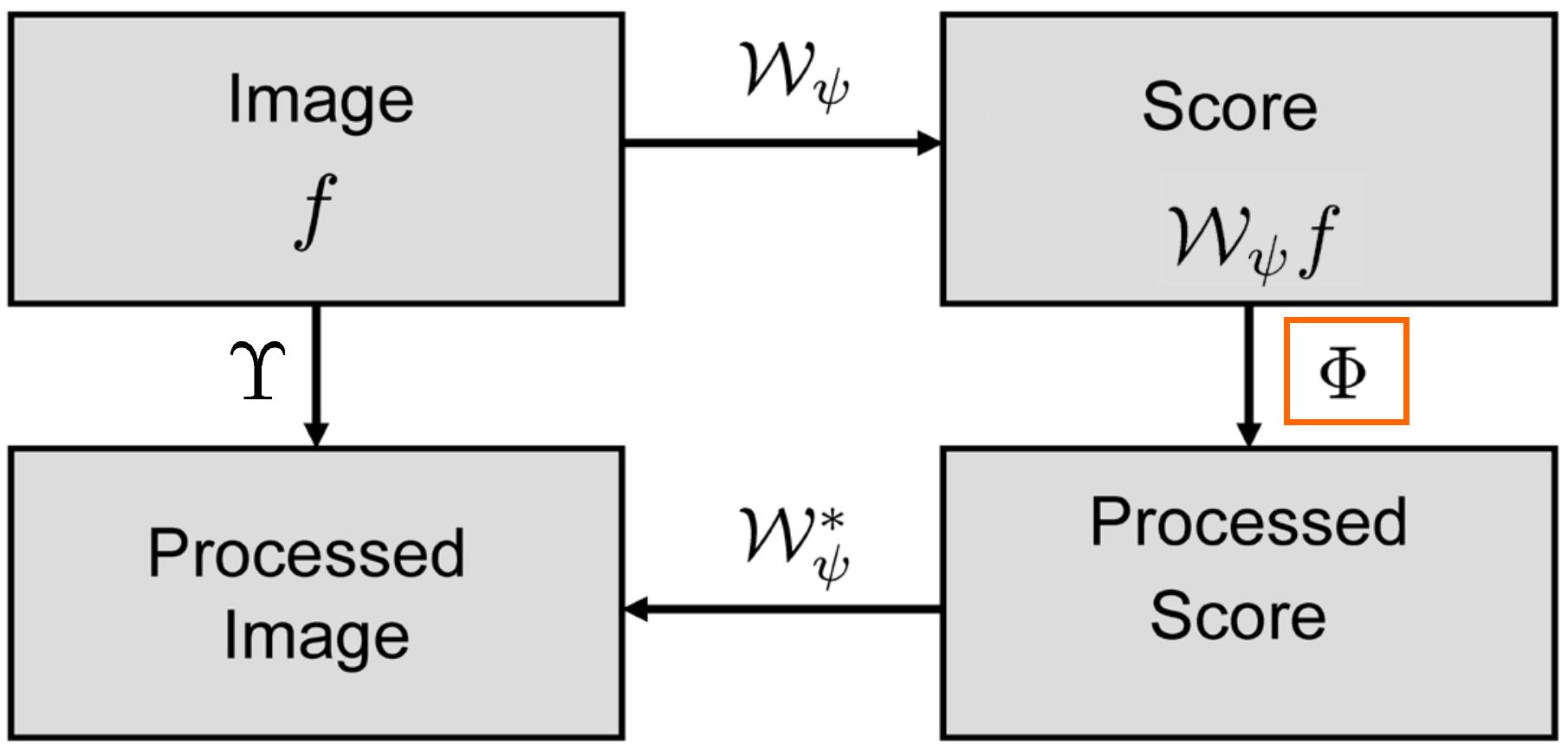}
\caption{A schematic view of image processing via invertible orientation scores.}%
\label{fig:OverviewOperators}
\end{figure}

\begin{figure*}[ht]
\centering
	\includegraphics[width=0.99 \hsize]{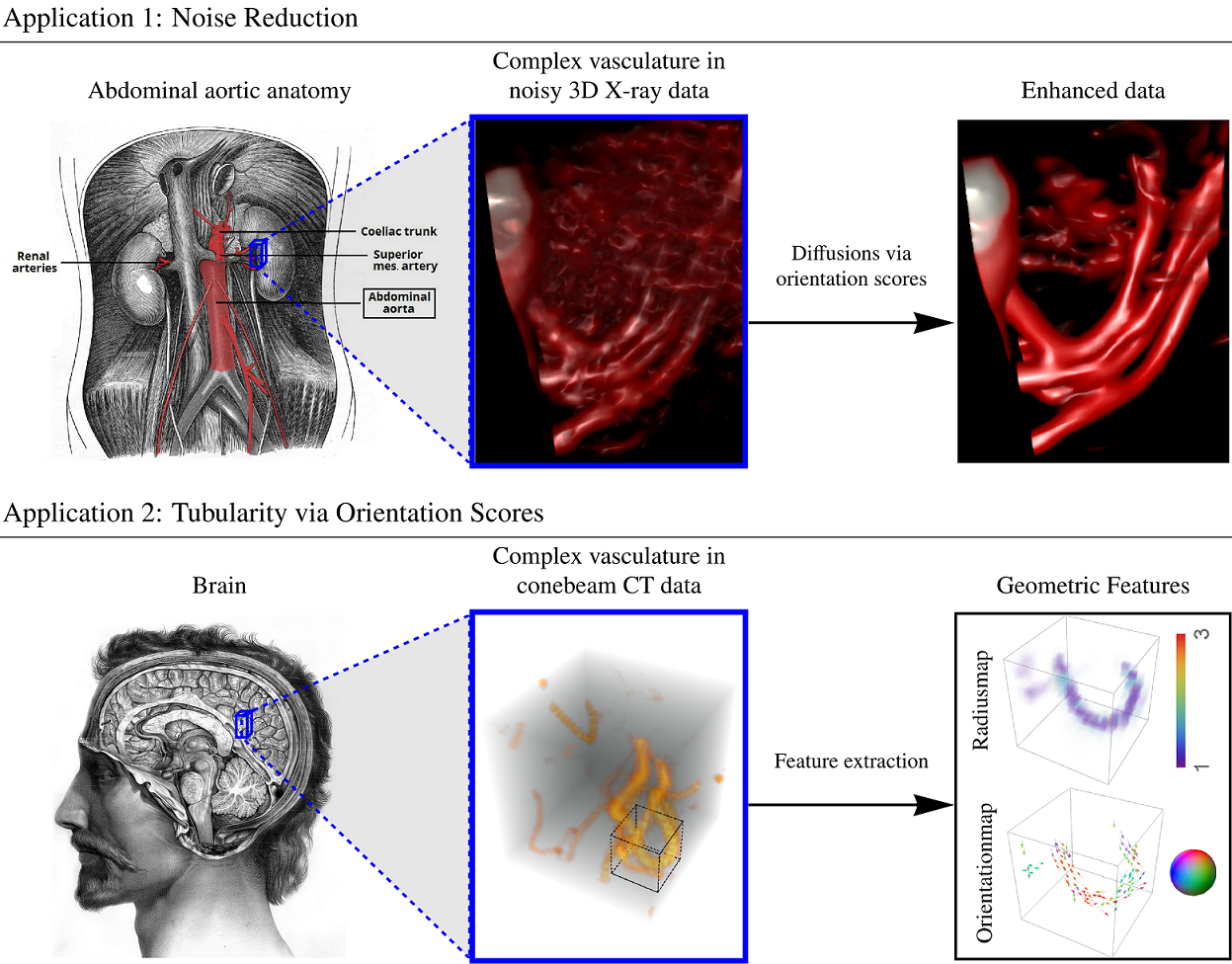}
	\caption{Overview of applications of processing via orientation scores. \emph{Top:} We reduce noise in real medical image data (3D rotational Xray) of the abdomen containing renal arteries by applying diffusions via 3D orientation scores. \emph{Bottom:} Geometrical features can be directly extracted from our tubularity measure via 3D orientation scores. We apply this method to cone beam CT data of the brain.}
	\label{fig:OverviewApplications}
\end{figure*}

\begin{figure*}[ht]%
\centering
\includegraphics[width=0.6\textwidth]{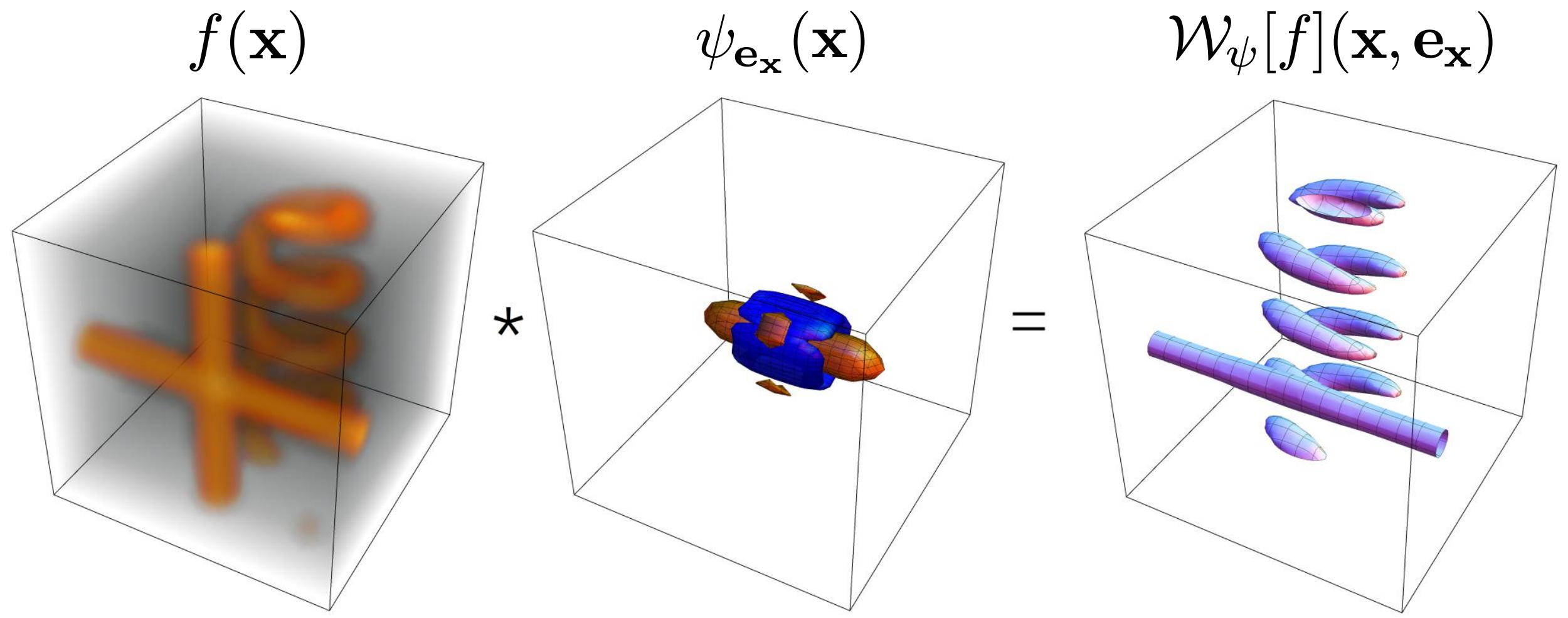}\\
\includegraphics[width=0.7\textwidth]{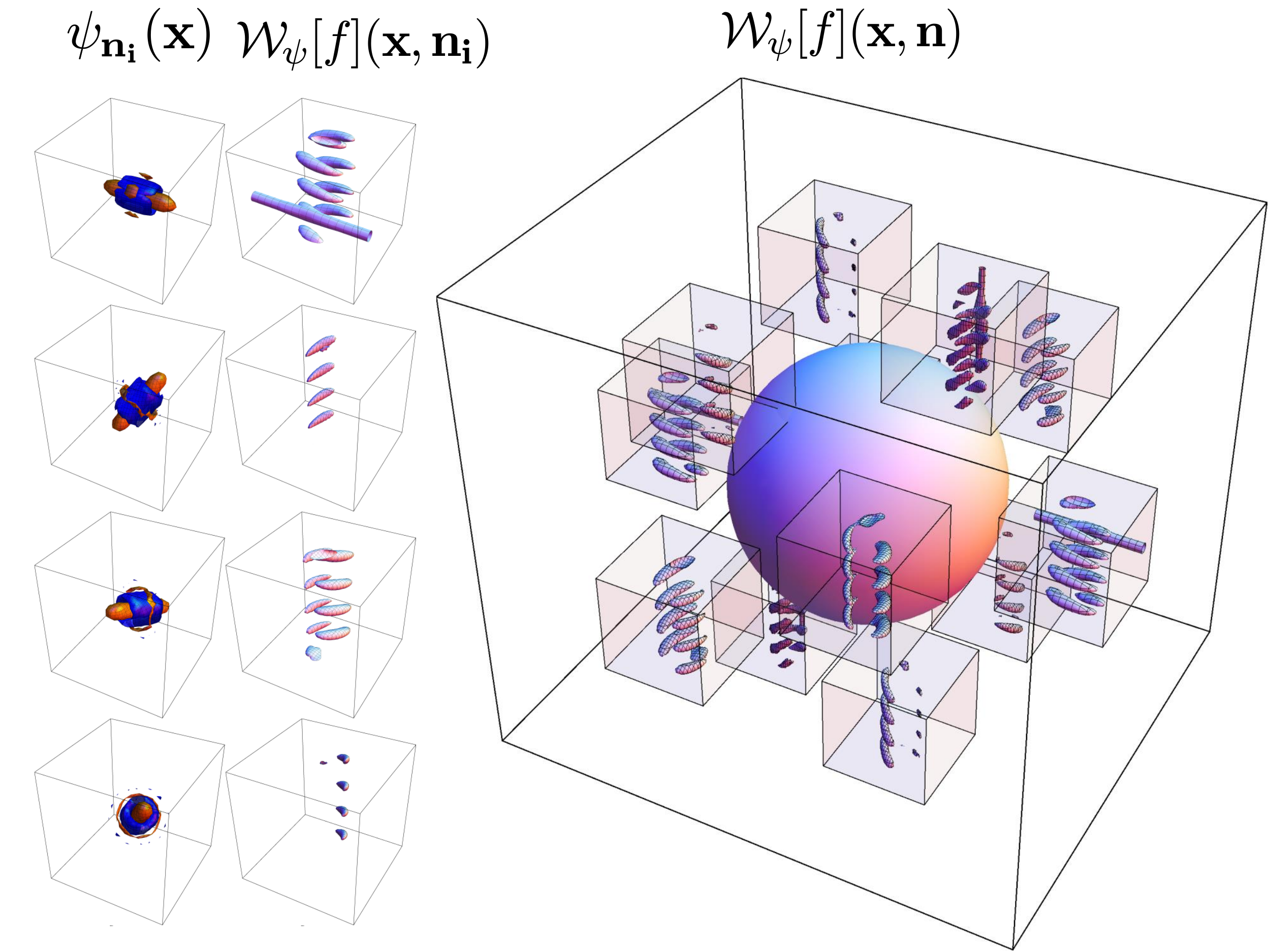}
\caption{Construction of a 3D orientation score. \emph{Top:} The data $f$ is correlated with an oriented filter $\psi_{\ve_x}$ to detect structures aligned with the filter orientation $\ve_x$. \emph{Bottom left:} This is repeated for a discrete set of filters with different orientations. \emph{Bottom right:} The collection of 3D datasets constructed by correlation with the different filters is an orientation score and is visualized by placing a 3D dataset on a number of orientations.}%
\label{fig:3DOSIntro}%
\end{figure*}

\section{Invertible Transformation}\label{sect:OSTransform}
\subsection{Continuous Orientation Score Transform}
Throughout this article we use the following definition of the Fourier transform on $\R^3$:
 \begin{equation}
 	\hat{f}(\vomega) = (\cFRRR f)(\vomega) = \int_{\R ^ 3 } e ^ {-i \vomega\cdot \vx} f(\vx) \d \vx.
\end{equation}
An invertible orientation score $\cW_{\psi}[f]:\R^3 \times S^2 \rightarrow \C$ is constructed from a given ball-limited 3D dataset
\begin{equation}
	f \in \LL_2^{\varrho} (\R^3)=\{f \in \LL_2 (\R^3) \; | \; \textrm{supp}(\cF f) \subset B_{\varrho}\},
	\label{eq:ballLimitedData}
\end{equation}
with ball $B_{\varrho} = \{\vomega \in \R^3 \;|\; \|\vomega\| < \varrho\}$ of radius $\varrho>0$, by correlation $\star$ with an anisotropic kernel
\begin{equation}
	\begin{split}
(\cW_{\psi}[f])(\vx,\vn) 	&=(\overline{\psi_\vn} \star f)(\vx) \\
							&=\int_{\R ^ 3}  \overline {\psi_\vn (\vx'-\vx)}f(\vx')\d \vx '.
\end{split}
\label{eq:Construction1}
\end{equation}
Here $\psi \in \LL_2(\R^3) \cap \LL_1(\R^3)$ is a wavelet aligned with and rotationally symmetric around the $z$-axis, and $\psi_{\vn}\in \LL_2(\R^3)$ the rotated wavelet aligned with $\vn$ given by
\begin{equation}\label{eq:RotatedWavelet}
	 \psi_{\vn} (\vx)=\psi (\mR_{\vn}^T \vx).
 \end{equation}
Here $\mR_{\vn} \in SO(3)$ is any 3D rotation which rotates the $z$-axis onto $\vn$ where the specific choice of rotation does not matter because of the rotational symmetry of $\psi$. The overline denotes a complex conjugate. The exact reconstruction formula for this transformation is given by
\begin{equation}
\begin{array}{l}
f(\vx) = (\cW_{\psi}^{-1}[\cW_{\psi}[f]])(\vx)= \\
\cFRRR^{-1} \left[ M_\psi^{-1}  \cFRRR \left[ \tilde{\vx} \mapsto \!\! \int \limits_{S^2} (\check {\psi}_\vn \star \cW_{\psi}[f](\cdot,\vn))(\tilde{\vx}) \d \sigma(\vn) \right] \right] (\vx),
\end{array}
\label{eq:Reconstruction1}
\end{equation}
with $\check {\psi}_\vn (\vx ) =\psi _\vn(-\vx)$.
The function $M_\psi: \R^3 \rightarrow \R^+ $ is given by
\begin{equation}
M_\psi(\vomega) = \int_{S^2} \left| \hat{\psi}_\vn(\vomega) \right|^2 \d \sigma(\vn),
\label{eq:Mpsi}
\end{equation}
and vanishes at $\infty$, where the circumflex $(\,\hat{}\,)$ again denotes Fourier transformation. Due to our restriction to ball-limited data \eqref{eq:ballLimitedData} this does not cause problems in reconstruction \eqref{eq:Reconstruction1}. The function $M_\psi$ quantifies the stability of the inverse transformation \cite{ThesisDuits}, since $M_\psi(\vomega)$ specifies how well frequency component $\vomega$ is preserved by the cascade of construction and reconstruction when $M_\psi^{-1}$ would not be included in Eq.~\!(\ref{eq:Reconstruction1}). An exact reconstruction is possible as long as
\begin{equation}
	\exists_{M> 0,\delta> 0}\forall_{\vomega \in B_{\varrho}} \; : \; \;  0<\delta \leq M_\psi (\vomega) \leq M<\infty.
\label{eq:AdmissibilityRequirement}
\end{equation}
In practice it is best to aim for $M_\psi (\vomega) \approx 1,$  in view of the condition number of transformation $\cW_\psi:\LL_2^\varrho (\R^3)\rightarrow \LL_2^\varrho (\R^3\times S^2)$ given by:
\begin{equation}
	\cond(\cW_{\psi}) = \|\cW_\psi\|\|\cW_\psi^{-1}\| = \frac{M}{\delta},
\end{equation}
where in the codomain spatial frequencies are again limited to the ball:
\begin{equation}
	\LL_2^\varrho (\R^3\times S^2)=\{U \in \LL_2 (\R^3 \times S^2)| \forall_{\vn \in S^2} \, U(\cdot,\vn) \in \LL_2^\varrho (\R^3)\}.
\end{equation}

Also, in the case we have  $M_\psi (\vomega) = 1$ for $\vomega \in B_{\varrho}$ we have $\LL_2$-norm preservation
\begin{equation}
\|f\|_{\LL_2(\R^3)}^2 = \|\cW_{\psi} f\|_{\LL_2 (\R^3\times S^2)}^2,\;\; \textrm{for all } f \in \LL_2^\varrho (\R^3),
\end{equation}
and reconstruction Eq.~\!(\ref{eq:Reconstruction1}) simplifies to
\begin{equation}
	f(\vx) = \int_{S^2} (\check{\psi}_\vn \star \cW_{\psi}[f](\cdot,\vn))(\vx) \d \sigma(\vn).
\end{equation}
We can further simplify the reconstruction for wavelets for which the following additional property holds:
\begin{equation}
		N_\psi(\vomega) = \int_{S^2} \hat{\psi}_\vn(\vomega) \, \d \sigma(\vn) \approx 1.
		\label{eq:Npsi}
\end{equation}
In that case the reconstruction formula is approximately an integration over orientations only:
\begin{equation}
f(\vx) \approx \int_{S^2} \cW_{\psi}[f](\vx,\vn) \, \d \sigma(\vn).
\label{eq:Reconstruction2Approximation}
\end{equation}
For the reconstruction by integration over angles only we can analyze the stability via the condition number of the  mapping that maps an image $f \in \mathbb{L}_{2}^{\varrho}(\R^3)$ to an orientation integrated score
\begin{equation}\label{op1}
A_{\psi}(f) = \int_{S^2} \mathcal{W}_{\psi}f(\cdot,\vn)\, {\rm d}\sigma(\vn).
\end{equation}
Its condition number is given by
{\small $\cond(A_\psi) =
\frac{\max \limits_{\vomega \in B_{\varrho}}N_{\psi}(\vomega)}{\min \limits_{\vomega \in B_{\varrho}}N_{\psi}(\vomega)}
$}.

In practice, we always use this last reconstruction because practical experiments show that performing an additional convolution with the wavelet as done in reconstruction \eqref{eq:Reconstruction1} after processing the score can lead to artifacts. It is, however, important to also consider the reconstruction \eqref{eq:Reconstruction1} and $M_\psi$ because it is used to quantify the stability and norm preservation of the transformation from image to orientation score.

The fact that we use reconstruction by integration	 while still taking into account norm-preservation by controlling $M_\psi$ leads to restrictions on our wavelets which are captured in the following definition:

\begin{definition}[Proper Wavelet] \label{def:properwavelet}
Let us set a priori bounds\footnote{In practice we choose the default values $\delta=\frac{1}{8}$ and $M=1.1$ and $\varepsilon=0.01$ and note it is actually the ratio $\frac{M}{\delta}$ that determines the condition number. It is just a convenient choice to set the upper bound close to 1.} $\delta,M>0,\  0  < \varepsilon \ll 1$. Furthermore, let $\varrho$ be an a priori maximum frequency of our ball-limited image data. Then, a wavelet $\psi \in \mathbb{L}_{2}(\mathbb{R}^{3}) \cap \mathbb{L}_{1}(\mathbb{R}^{3})$ is called a proper wavelet if
\begin{align}
1.)\  & \forall_{\alpha \in[0,2 \pi)} &&: \psi(\mR_{\ve_z, \alpha}^{-1}\vx)=\psi (\vx) ,\hspace{2.2cm} \\[5pt]
2.)\  & \forall_{\vomega \in B_{\varrho}} &&: \delta \leq M_{\psi}(\vomega) \leq M, \label{eq:AdmissibilityRequirementInDef1}
\end{align}
where $\mR_{\ve_z, \alpha} \in SO(3)$ is the 3D rotation around axis $\ve_z$ over angle $\alpha$.

If moreover, one has
\begin{equation}
	3.)\ \exists_{\frac{1}{2}\varrho <\varrho_0 < \varrho} \forall_{\vomega \in B_{\varrho_0}}\;:\; N_{\psi}(\vomega) \in [1-\varepsilon,1+\varepsilon],
\end{equation}
then we speak of a \emph{proper wavelet with fast reconstruction property}, cf.~\!(\ref{eq:Reconstruction2Approximation}).
\end{definition}
\begin{remark}
The 1st condition (symmetry around the $z$-axis) allows for an appropriate definition of an \emph{orientation score} rather than a \emph{rotation score}.
The 2nd condition ensures invertibility and stability of the (inverse) orientation score transform.
The 3rd condition, allows us to use the approximate reconstruction by integration over angles only.
\end{remark}

\begin{remark}
Because of finite sampling in practice, the constraint to ball-limited functions is reasonable. The constraint is not a necessary one when one relies on distributional transforms \cite[App. B]{BekkersJMathImagingVis2014}, but we avoid such technicalities here.
\end{remark}

\subsubsection{Low Frequency Components}\label{sssect:LowFrequencyComponents}
In practice we are not interested in the zero and lowest frequency components since they represent average value and global variations which appear at scales much larger than the structures of interest. We need, however, to store this data for reconstruction. Therefore we perform an additional splitting of our wavelets into two parts
\begin{equation}
		 \psi=\psi_0 + \psi_1, \quad \textrm{with } \hat{\psi}_0=\hat{G}_{s_{\rho}} \hat{\psi}, \;\;  \hat{\psi}_1=(1 - \hat{G}_{s_{\rho}}) \hat{\psi},
		 \label{eq:WaveletSplitting}
\end{equation}
with Gaussian window in the Fourier domain given by
\begin{equation}
	\hat{G}_{s_{\rho}}(\vomega) = e^{- s_\rho \|\vomega\|^2},
\end{equation}
After splitting, $\psi_0$ contains the average and low frequency components and $\psi_1$ the higher frequencies relevant for further processing. In continuous wavelet theory it is also common to separately store very low frequency components separately, see e.g. \cite{Mallat1999,Sifre2014}. In this case we construct two scores. One for the high-frequency components
\begin{equation}
(\cW_{\psi_1}[f])(\vx,\vn)=(\overline{\psi_{1,\vn}} \star f)(\vx),
\label{eq:Construction1WithoutLowFrequencies}
\end{equation}
and one for the low-frequency components
\begin{equation}
(\cW_{\psi_0}[f])(\vx,\vn)=(\overline{\psi_{0,\vn}} \star f)(\vx).
\label{eq:Construction1LowFrequencies}
\end{equation}
Here we again have $\psi_{i,\vn}(\vx)=\psi_{i}(\mR_{\vn}^T \vx)$, as in Eq. \eqref{eq:RotatedWavelet}. The vector transformation is then defined as
\begin{equation}
	\underline{\cW}_{\underline{\psi}} [f]={(\cW_{\psi_0}[f],\cW_{\psi_1}[f])}.
\end{equation}
For this transformation we have the exact reconstruction formula
\begin{align}
f(\vx) &= ( \underline{\cW}_{\underline{\psi}}^{-1} \underline{\cW}_{\underline{\psi}} f )(\vx) \nonumber \\
&= \cFRRR^{-1} \Bigg[ M_{\underline{\psi}}^{-1}  \cFRRR \Bigg[ \tilde{\vx} \mapsto \int_{S^2} (\check {\psi}_{1,\vn} \star \cW_{\psi_1}[f](\cdot,\vn))(\tilde{\vx}) + \nonumber \\
& \quad (\check {\psi}_{0,\vn} \star \cW_{\psi_0}[f](\cdot,\vn))(\tilde{\vx}) \d \sigma(\vn) \Bigg] \Bigg] (\vx)
\label{eq:ReconstructionLowAndHighFrequencies}
\end{align}
with
\begin{equation}
		M_{\underline{\psi}}(\vomega)
		=   \int_{S^2} \left( \left| \hat{\psi}_{0,\vn}(\vomega) \right|^2 + \left| \hat{\psi}_{1,\vn}(\vomega) \right|^2 \right) \d \sigma(\vn).
\end{equation}
Again, $M_{\underline{\psi}}$ quantifies the stability of the transformation. The next lemma shows us that the stability of the transformation is maintained after performing the additional splitting.

\begin{lemma}\label{lemma:StabilitySplitting}
	Let $\psi \in \LL_2(\R^3) \cap \LL_1(\R^3)$ such that Eq. \eqref{eq:AdmissibilityRequirement} holds, $\delta=\min_{\vomega \in B_\varrho} M_\psi(\vomega)$ and $M=\max_{\vomega \in B_\varrho} M_\psi(\vomega)$. Then the condition number of $\cW_\psi: \LL_2^\varrho(\R^3) \rightarrow \LL_2^\varrho(\R^3 \times S^2)$ is given by
	\begin{equation}
		|\cond(\cW_{\psi})|^2 = \|\cW_\psi\|^2\|\cW_\psi^{-1}\|^2 = \frac{M}{\delta}.
	\end{equation}
	 The condition number of $\underline{\cW}_{\underline{\psi}}: \LL_2^\varrho(\R^3) \rightarrow \underline{\LL}_2^\varrho(\R^3 \times S^2)$ obtained from $\cW_\psi$ by performing an additional splitting in low and high frequency components is given by
	\begin{equation}
		|\cond(\underline{\cW}_{\underline{\psi}})|^2 = \|\underline{\cW}_{\underline{\psi}}\|^2\|\underline{\cW}_{\underline{\psi}}^{-1}\|^2 = \frac{2M}{\delta},
	\end{equation}
	thereby guaranteeing that stability is maintained after performing the splitting.
\end{lemma}
\begin{proof}
First, we find the condition number of $\cW_\psi$:
\begin{equation}
\begin{array}{ll}
		|\cond(\cW_{\psi})|^2 &=\!\!\! \sup \limits_{f \in \LL_2^\varrho(\R^3)} \frac{ \| f  \|_{\LL_2}^2}{\| \cW_{\psi} f  \|_{\LL_2}^2} \cdot \!\! \sup \limits_{f \in \LL_2^\varrho(\R^3)} \!\! \frac{\| \cW_{\psi} f \|_{\LL_2}^2}{\| f \|_{\LL_2}^2}. \\
		\end{array}
		\label{eq:conditionNumberProof}
\end{equation}
For the first factor in Eq. \eqref{eq:conditionNumberProof} we find
\begin{equation}
	\begin{split}
			\sup \limits_{f \in \LL_2^\varrho(\R^3)} \frac{ \| f  \|_{\LL_2}^2}{\| \cW_{\psi} f  \|_{\LL_2}^2} &= \sup \limits_{f \in \LL_2^\varrho(\R^3)} \frac{ \| \mathcal{F} f  \|_{\LL_2}^2}{\| \mathcal{F} \cW_{\psi} f  \|_{\LL_2}^2}   \\
			 &\hspace{-2.1cm} = \sup \limits_{f \in \LL_2^\varrho(\R^3)} \frac{ \int_{\R^3} |\hat{f}(\vomega)|^2  \d \vomega}{ \int_{S^2} \int_{{\R^3}} |\hat{\psi}_{\vn} (\vomega)|^2 |\hat{f}(\vomega)|^2 \d \vomega \d \sigma(\vn)}  \\
		&\hspace{-2.1cm} = \sup \limits_{f \in \LL_2^\varrho(\R^3)} \frac{ \int_{\R^3} |\hat{f}(\vomega)|^2  \d \vomega}{ \int_{{\R^3}} M_{\psi}(\vomega) |\hat{f}(\vomega)|^2 \d \vomega} \\	&\hspace{-2.1cm} = \sup_{\vomega \in B_\varrho} \frac{1}{M_{\psi}(\vomega)}.
	\end{split}
\end{equation}
Similarly, we get $\sup_{\vomega \in B_\varrho} M_{\psi}(\vomega)$ for the second factor in Eq. \eqref{eq:conditionNumberProof}. Then we obtain
\begin{equation}
			\cond(\cW_{\psi})  = \sup_{\vomega \in B_\varrho} \frac{1}{M_{\psi}(\vomega)} \cdot \sup_{\vomega \in B_\varrho} M_{\psi}(\vomega) =   \frac{M}{\delta}.
\end{equation}
Similarly the condition number of $\underline{\cW}_{\underline{\psi}}$ is given by
\begin{equation}
			\cond(\underline{\cW}_{\underline{\psi}})  = \sup_{\vomega \in B_\varrho} \frac{1}{M_{\underline{\psi}}(\vomega)} \cdot \sup_{\vomega \in B_\varrho} M_{\underline{\psi}}(\vomega).
\end{equation}
Next we express $M_{\underline{\psi}}$ in $M_{\psi}$ of the original wavelet as
\begin{equation}
\begin{split}
		M_{\underline{\psi}}(\vomega)
		&=   \int_{S^2} \left| \hat{\psi}_{0,\vn}(\vomega) \right|^2  + \left| \hat{\psi}_{1,\vn}(\vomega) \right|^2 \d \sigma(\vn) \\
		&=   \int_{S^2} \left| \hat{\psi}_{0,\vn}(\vomega) + \hat{\psi}_{1,\vn}(\vomega) \right|^2 \d \sigma(\vn) \\ & \quad -  \int_{S^2} 2 \operatorname{Re} \left( \hat{\psi}_{0,\vn} (\vomega)  \conj{\hat{\psi}_{1,\vn}(\vomega)} \right) \d \sigma(\vn) \\
		&= M_{\psi}(\vomega) - I(\vomega).
\end{split}
\end{equation}
So it remains to quantify $I(\vomega)$. For a wavelet splitting according to \eqref{eq:WaveletSplitting} we have
\begin{align}
		 I(\vomega) &=   \int_{S^2} 2 \operatorname{Re} \left( \hat{\psi}_{0,\vn} (\vomega)  \conj{\hat{\psi}_{1,\vn}(\vomega)} \right) \d \sigma(\vn)  \nonumber\\
		&=   \int_{S^2} 2 \operatorname{Re} \left( \hat{G}_{s_{\rho}}(\vomega) \hat{\psi}_\vn(\vomega)   (1 - \hat{G}_{s_{\rho}}(\vomega)) \conj{\hat{\psi}_\vn(\vomega)} \right)  \d \sigma(\vn) \nonumber\\
		&=  2 (\hat{G}_{s_{\rho}}(\vomega) (1 - \hat{G}_{s_{\rho}}(\vomega))) M_{\psi}(\vomega).
\end{align}
Hence
\begin{equation}
	M_{\underline{\psi}}(\vomega) = \Big(1 - 2 \big(\hat{G}_{s_{\rho}}(\vomega) \big(1 - \hat{G}_{s_{\rho}}(\vomega)\big)\big)\Big)  M_{\psi}(\vomega).
\end{equation}
And since $\frac{1}{2} \leq 1-2x(1-x)) \leq 1$ for $0 \leq x \leq 1$ we have for $M_\psi$ satisfying \eqref{eq:AdmissibilityRequirement}
the following bounds on $M_{\underline{\psi}}$:
\begin{equation}
0<\delta/2 \leq M_{\underline{\psi}} (\vomega) \leq M<\infty, \quad \textrm {for all } \vomega=B_{\varrho},
\label{eq:AdmissibilityRequirementNew}
\end{equation}
thereby guaranteeing stability after the splitting \eqref{eq:WaveletSplitting}.
$\hfill \Box$
\end{proof}

For this vector transformation we can also use the approximate reconstruction by integration (for $N_\psi \approx 1$) over orientations. Thus we have
\begin{align}\label{eq:Reconstruction2Approximation2}
	f(\vx) &\approx \int_{S^2} \cW_{\psi}[f](\vx,\vn) \d \sigma(\vn)\\
	&=  \int_{S^2} \cW_{\psi_1}[f](\vx,\vn) \d \sigma(\vn) + \underbrace{\int_{S^2} \cW_{\psi_0}[f](\vx,\vn) \d \sigma(\vn)}_{L_{\psi_0}[f](\vx) }. \nonumber
\end{align}
As said we are only interested in processing of $\cW_{\psi_1}[f]$ and not in processing of $\cW_{\psi_0}[f]$, and so we directly calculate $L_{\psi_0}[f]$ via
\begin{equation}
	L_{\psi_0}[f](\vx) =(\overline{\phi_0} \star f)(\vx), \;  \textrm{with } \phi_0 = \! \int_{S^2} \! \psi_{0,\vn} \,\d \sigma(\vn).
\end{equation}
For a design with $N_\psi(\vomega)=1$ for all $\vomega \in B_\varrho$, we have $\hat{\phi}_0=\hat{G}_{s_{\rho}}$ and so
\begin{equation}\label{eq:LowFrequencyPhi0}
	 \phi_0(\vx) =   G_{s_{\rho}}(\vx) = \frac{1}{(4\pi s_{\rho})^{3/2}} e^{- \frac{\|\vx\|^2}{4 s_\rho}}.
\end{equation}
Then Eq. \eqref{eq:Reconstruction2Approximation2} becomes
\begin{equation}
	 f(\vx) \approx  \int_{S^2} \cW_{\psi_1}[f](\vx,\vn) \, \d \sigma(\vn) + (G_{s_{\rho}} * f) (\vx).
\end{equation}

\subsection{Discrete Orientation Score Transform}
In the previous section, we considered a continuous orientation score transformation. In practice, we have only a finite number of orientations. To determine this discrete set of orientations we uniformly sample the sphere using an electrostatic repulsion model \cite{CaruyerMagnResonMed2013}.

Assume we have a number $N_o$ of orientations $\cV =\{\vn_1,\vn_2,...,\vn_{N_o}\}\subset S^2$, and define the discrete invertible orientation score $\cW_\psi^d[f]:\R^3\times \cV \rightarrow \C$ by
\begin{equation}
(\cW_\psi^d[f])(\vx,\vn_i)=(\overline{\psi_{\vn_i}} \star f)(\vx).
\label{eq:construction1Discrete}
\end{equation}
The exact reconstruction formula is in the discrete setting given by
\begin{equation}
\begin{split}
f(\vx) &= ((\cW_\psi^d)^{-1}[\cW_\psi^d[f]])(\vx) \\
&= \cFRRR^{-1} \bigg[ (M_\psi^d)^{-1}  \cFRRR \bigg[ \\
&\qquad\! \tilde {\vx} \rightarrow \sum_{i=1}^{N_o} (\check {\psi}_{\vn_{i}} \star \cW_\psi^d[f](\cdot,\vn_i))(\tilde {\vx}) \, \Delta_i \bigg] \bigg] (\vx),
\end{split}
\label{eq:Reconstruction1Discrete}
\end{equation}
with $\Delta_i$ the discrete spherical area measure ($\sum \limits_{i=1}^{N_o} \Delta_i =4\pi$) which for reasonably uniform spherical sampling can be approximated by $\Delta_i\approx \frac{4 \pi}{N_o}$ (otherwise one could use \cite[Eq. (83)]{DuitsIntJComputVis2010}), and
\begin{equation}\label{eq:MpsiDiscrete}
M_\psi^d(\vomega) =   \sum_{i=1}^{N_o} \left| \hat{\psi}_{\vn_i}(\vomega) \right|^2 \Delta_i.
\end{equation}
Again, an exact reconstruction is possible iff $0<\delta\leq M_\psi^d (\vomega)\leq M<\infty$ and we have norm preservation when $M_\psi^d$=1.
Again, for the wavelets for which
\begin{equation}\label{eq:NpsiDiscrete}
	N_\psi^d =   \sum_{i=1}^{N_o} \hat{\psi}_{\vn_i}(\vomega) \Delta_i \approx 1,
\end{equation}
the image reconstruction can be simplified by a summation over orientations:
\begin{equation}
\begin{split}
f(\vx) &\approx \sum_{i=1}^{N_o} \cW_\psi^d[f](\vx,\vn_i) \, \Delta_i.
\end{split}
\label{eq:ReconstructionSumDiscrete}
\end{equation}
For this reconstruction by summation we can analyze the stability via the condition number of the  mapping that maps an image $f \in \mathbb{L}_{2}^{\varrho}(\R^3)$ to an orientation integrated score
\begin{equation}\label{op2}
A_\psi^d (f) = \sum \limits_{i=1}^{N_o} \overline{\psi}_{\vn_i} \star f \;\Delta_i.
 \end{equation}
This transformation has condition number
{\small $ \cond(A_\psi^d) =
\frac{\max \limits_{\vomega \in B_{\varrho}}N_{\psi}^d(\vomega)}{\min \limits_{\vomega \in B_{\varrho}}N_{\psi}^d(\vomega)}
$}.

Similar to Definition \ref{def:properwavelet} for the continuous case, the reconstruction properties of a set of filters is captured in the following definition:

\begin{definition}[Proper Wavelet Set] \label{def:properwaveletSet}
Let us again set a priori bounds $\delta,M>0,\  0  < \varepsilon \ll 1$. Let $\varrho$ be an a priori maximum frequency of our ball-limited image data. Then, a set of wavelets $\{ \psi_{\vn_i} \in \mathbb{L}_{2}^{\varrho}(\mathbb{R}^{3}) \cap \mathbb{L}_{1}(\mathbb{R}^{3}) \,|\, i=1,\dots,N_o \}$, with a reasonable uniform spherical sampling ($\Delta_i \approx \frac{4\pi}{N_o}$), constructed as rotated versions of $\psi$ is called a proper wavelet set if
\begin{align}
1.)\  & \forall_{\alpha \in[0,2 \pi)} &&: \psi(\mR_{\ve_z, \alpha}^{-1}\vx)=\psi (\vx) ,\hspace{2.2cm} \\[5pt]
2.)\  & \forall_{\vomega \in B_{\varrho}} &&: \delta \leq M_{\psi}^d(\vomega) \leq M, \label{eq:AdmissibilityRequirementInDef2}
\end{align}
where $\mR_{\ve_z, \alpha} \in SO(3)$ is a 3D rotation around axis $\ve_z$ over angle $\alpha$.

If moreover, one has
\begin{equation}
	3.)\ \exists_{\frac{1}{2}\varrho <\varrho_0 < \varrho} \forall_{\vomega \in B_{\varrho_0}}\;:\; N_{\psi}^d(\vomega) \in [1-\varepsilon,1+\varepsilon],
\end{equation}
then we speak of a \emph{proper wavelet with fast reconstruction property}, cf.~\!(\ref{eq:ReconstructionSumDiscrete}).
\end{definition}

\subsubsection{Low Frequency Components}
For the discrete transformation we will also perform a splitting in low and high frequency components as explained in Section \ref{sssect:LowFrequencyComponents}. The reconstruction formula by summation in Eq. \eqref{eq:ReconstructionSumDiscrete} is now given by
\begin{equation}
\begin{split}
f(\vx) &\approx \sum_{i=1}^{N_o} \cW_{\psi_1}^d[f](\vx,\vn_i) \Delta_i + (G_{s_{\rho}} * f) (\vx).
\end{split}
\label{eq:ReconstructionSumDiscreteAfterSplitting}
\end{equation}

\subsection{Steerable Orientation Score Transform}
Throughout this article we shall rely on spherical harmonic decomposition of the angular part of proper wavelets in spatial and Fourier domain. This has the benefit that one can obtain steerable \cite{FreemanPAMI1991,ThesisFranken,ThesisReisert} implementations of orientation scores, where rotations of the wavelets are obtained via linear combination of the basis functions. As such, computations are exact and no interpolation (because of rotations) takes place. Details are provided in Appendix~\ref{app:steerable}.

\section{Wavelet Design using a DFT}\label{sect:WaveletDFT}

A class of 2D cake-wavelets, see \cite{BekkersJMathImagingVis2014,FrankenIJCV2009,DuitsPRIA2007}, was used for the 2D orientation score transformation. We now present 3D versions of these cake-wavelets. Thanks to the splitting in Section \ref{sssect:LowFrequencyComponents} we no longer need the extra spatial window used there. Our 3D transformation using the 3D cake-wavelets should fulfill a set of requirements, compare \cite{FrankenIJCV2009}:
\begin{enumerate}
		\item The orientation score should be constructed for a finite number ($N_o$) of orientations.
		\item The transformation should be invertible and reconstruction should be achieved by summation. Therefore we aim for $N_\psi^d \approx 1$. Additionally, to guarantee all frequencies are transferred equally to the orientation score domain we aim for $M_\psi^d \approx 1$. The set should be a proper wavelet set with fast reconstruction property (Def. \ref{def:properwaveletSet}) \label{requirement:ReconstructionSummation}
		\item The kernel should be strongly directional.
		\item The kernel should be separable in spherical coordinates in the Fourier domain, i.e., $(\cF\psi) (\vomega) =g(\rho) h (\vartheta,\varphi)$, with
		\begin{equation}
			\begin{split}
				\vomega &= (\omega_x,\omega_y,\omega_z) \\
					&= (\rho \sin \vartheta \cos \varphi,\rho \sin \vartheta \sin \varphi,\rho \cos \vartheta).
			\end{split}
			\label{eq:CoordinatesFourierSpace}
		\end{equation}
		Because by definition the wavelet $\psi$ has rotational symmetry around the $z$-axis we have $h (\vartheta,\varphi) = \Gh (\vartheta)$.
		\item The kernel should be localized in the spatial domain, since we want to pick up local oriented structures.
		\item The real part of the kernel should detect oriented structures and the imaginary part should detect oriented edges. The constructed orientation score is therefore a complex orientation score. For an intuitive preview, see the boxes in Fig. \ref{fig:CakeWavelets}.
\end{enumerate}

\subsection{Construction of Line and Edge Detectors}\label{ssect:CakeDesign}
We now discuss the procedure used to make 3D cake-wavelets before splitting in low and high frequencies according to \eqref{eq:WaveletSplitting} in Section \ref{sssect:LowFrequencyComponents} takes place. Following requirement 4 we only consider polar separable wavelets in the Fourier domain, so that $(\mathcal{F}\psi) (\vomega) = g(\rho) \Gh (\vartheta)$. To satisfy requirement 2 we should choose radial function $g (\rho) = 1$ for $\rho \in [0,\varrho]$. In practice, this function should go to 0 when $\rho$ tends to the Nyquist frequency $\rho_\cN$ to avoid long spatial oscillations. For the radial function $g (\rho)$ we use,
\begin{equation}\label{eq:radialFunctionOfPsi}
g (\rho)=\frac{1}{2}(1 - \erf(\frac{\rho - \varrho}{\sigma_{erf}})),
\end{equation}
with $\erf(z)= \frac{2}{\sqrt{\pi}} \int_{0}^{z} e^{-x^2} \;{\rm d}x$, which is approximately equal to one for largest part of the domain and then smoothly goes to 0 when approaching the Nyquist frequency. We fix the inflection point of this function $g$ and set the fundamental parameter for ball-limitedness to
\begin{equation}\label{eq:GammaNyquist}
	\varrho = \gamma \, \rho_\cN,
\end{equation}
with $0 \ll \gamma < 1$. The steepness of the decay when approaching $\rho_\cN$ is controlled by the parameter $\sigma_{erf}$ which we by default set to $\sigma_{erf} = \frac{1}{3} (\rho_\cN-\varrho)$. The additional splitting in low and high frequencies according to Section \ref{sssect:LowFrequencyComponents} effectively causes a splitting of the radial function, see Fig. \ref{fig:radialFunction}.

In practice the frequencies in our data are limited by the Nyquist frequency (we have $\varrho \approx \rho_\cN$), and because radial function $g$ causes $M_\psi^d$ to become really small close to the Nyquist frequency, reconstruction Eq.\eqref{eq:Reconstruction1Discrete} becomes unstable. We solve this by using approximate reconstruction Eq.\eqref{eq:Reconstruction2Approximation}. Alternatively, one could replace $M_\psi^d$ by $\max(M_\psi^d,\epsilon)$ in Eq. \eqref{eq:Reconstruction1}, with $\epsilon$ small. Both make the reconstruction stable at the cost of not completely reconstructing the highest frequencies which causes a small amount of blurring.

\begin{figure}[htbp]
	\centering
	\includegraphics[width=0.95\columnwidth]{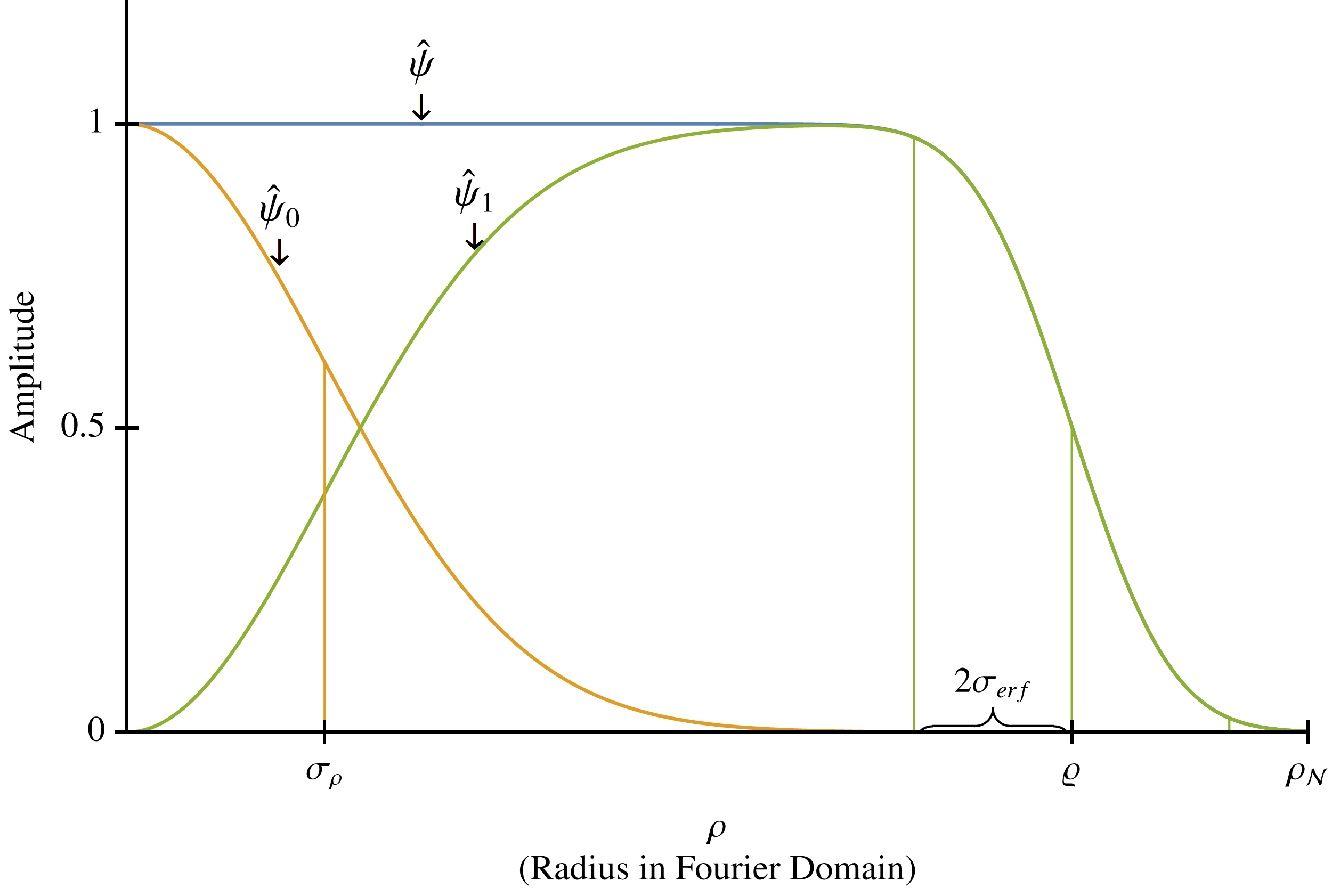}
	\caption{Radial part $g$ of $\hat{\psi}$, see Eq. \eqref{eq:radialFunctionOfPsi} and radial parts $g_0$ and $g_1$ of $\hat{\psi}_0$ and $\hat{\psi}_1$ after splitting according to Section \ref{sssect:LowFrequencyComponents}. The parameter $\gamma$ controls the inflection point of the error function, here $\gamma=0.8$. The steepness of the decay when approaching $\rho_\cN$ is controlled by the parameter $\sigma_{erf}$ with default value $\sigma_{erf} = \frac{1}{3} (\rho_\cN-\varrho)$. At what frequency the splitting of $\hat{\psi}$ in $\hat{\psi}_0$ and $\hat{\psi}_1$ is done is controlled by parameter $\sigma_\rho = \sqrt{2 s_\rho}$, see Eq. \eqref{eq:WaveletSplitting}.}
	\label{fig:radialFunction}
\end{figure}

We now need to find an appropriate angular part $\Gh$ for the cake-wavelets. First, we specify an orientation distribution $A:S^2\rightarrow \R^+$, which determines what orientations the wavelet should measure. To satisfy requirement 3 this function should be a localized spherical window, for which we propose the spherical diffusion kernel \cite{Chung2006}:
\begin{equation}\label{eq:orientationDistribution}
	A(\vn(\vartheta,\varphi)) = G_{s_o}^{S^2}(\vn(\vartheta,\varphi)),
\end{equation}
with $s_o>0$ and $\vn(\vartheta,\varphi) = (\sin \vartheta \cos \varphi,\sin \vartheta \sin \varphi,\cos \vartheta)$. The parameter $s_o$ determines the trade-off between requirements 2 and 3 listed in the beginning of Section \ref{sect:WaveletDFT}, where higher values give a more uniform $M_\psi^d$ at the cost of less directionality.

First consider setting $h=A$ so that $\psi$ has compact support within a convex cone in the Fourier domain. The real part of the corresponding wavelet would, however, be a plate detector and not a line detector (Fig. \ref{fig:cakePieceCreatesPlateDetector}). The imaginary part is already an oriented edge detector, and so we set
\begin{equation}
\begin{split}
	\Gh_ {Im} (\vartheta) &= \frac{1}{2}\left(A(\vn(\vartheta,\varphi)) -A(-\vn(\vartheta,\varphi))\right) \\
	&= \frac{1}{2} \left( G_{s_o}^{S^2}(\vn(\vartheta,\varphi))- G_{s_o}^{S^2}(-\vn(\vartheta,\varphi) \right) ,
\end{split}
\label{eq:Antisymmetrize}
\end{equation}
where the real part of the earlier found wavelet vanishes by anti-symmetrization of the orientation distribution $A$ while the imaginary part is unaffected. As to the construction of $h_ {Re}$, there is the general observation that we detect a structure that is perpendicular to the shape in the Fourier domain, so for line detection we should aim for a plane detector in the Fourier domain. To achieve this we apply the Funk transform to $A$, and we define
\begin{equation}
\begin{split}
	h_{Re}(\vartheta,\varphi) &= F A(\vn(\vartheta,\varphi)) \\
	&= \frac{1}{2\pi} \int_{S_p(\vn (\vartheta,\varphi))} \! A(\vn') \, \d s(\vn'),
\end{split}
\label{eq:CakeWaveletRe}
\end{equation}
where integration is performed over $S_p (\vn)$ denoting the great circle perpendicular to $\vn$. This transformation preserves the symmetry of $A$, so we have $h_{Re} (\vartheta,\varphi) =\Gh_{Re} (\vartheta)$. Thus, we finally set
\begin{equation}
\Gh(\vartheta) =\Gh_{Re}(\vartheta) +\Gh_{Im} (\vartheta).
\end{equation}
For an overview of the transformations see  Fig. \ref{fig:CakeWavelets}.

As discussed before, the additional splitting in low and high frequencies as described in Section \ref{sssect:LowFrequencyComponents} effectively causes a splitting in the radial function. How this affects the coverage of the Fourier domain is visualized in Fig. \ref{fig:CakeWaveletsSplitting}.

\subsection{Efficient Implementation via Spherical Harmonics}
In Subsection \ref{ssect:CakeDesign} we defined the real part and the imaginary part of the wavelets in terms of a given orientation distribution. In order to efficiently implement the various transformations (e.g. Funk transform), and to create the various rotated versions of the wavelet we express our orientation distribution $A$ in a spherical harmonic basis $\{Y_l^m\}$ up to order $L$:
\begin{equation}
A(\vn(\vartheta,\varphi)) =\sum_{l = 0} ^L \sum_{m= -l} ^ {l} a_l^m Y_l ^m(\vartheta,\varphi), \quad L \in \mathbb{N}.
\label{eq:ODSphericalHarmonics}
\end{equation}
The spherical harmonics are given by
\begin{equation}
 	Y_l^m(\vartheta ,\varphi ) = \epsilon \sqrt{\frac{2 l+1}{4 \pi }} \sqrt{\frac{(l-|m|)!}{(l+|m|)!}} e^{i m \varphi } P_l^{|m|}(\cos \vartheta),
 \label{eq:definitionSphericalHarmonics}
\end{equation}
where $P_l^m$ is the associated Legendre function, $\epsilon = (-1)^m$ for $m<0$ and $\epsilon = 1$ for $m>0$ and with integer $l\geq0$ and integer $m$ satisfying $-l \leq m \leq l$.
For the diffusion kernel, which has symmetry around the $z$-axis we only need the spherical harmonics with $m = 0$, and we have the coefficients \cite{Chung2006}:
\begin{equation}\label{eq:coefficientsDifussionKernel}
		a_{l}^{m} = \begin{cases}
				0                               & m\neq 0, \\
				\sqrt{\frac{2 l+1}{4 \pi }} e^{-l(l+1) s_o }    & m=0,
		\end{cases}
\end{equation}
and Eq. \eqref{eq:ODSphericalHarmonics} reduces to
\begin{equation}
	 A(\vn(\vartheta,\varphi)) = \sum \limits_{l=0}^L  a_{l}^{0} Y_l^0(\vartheta,\varphi).
\end{equation}
\subsubsection {Funk Transform}
According to \cite{DescoteauxMagnResonMed2007}, the Funk transform of a spherical harmonic equals
\begin{equation}
\begin{split}
	F Y_l^m  (\vartheta,\varphi) &= \frac{1}{2\pi} \int_{S_p(\vn(\vartheta,\varphi))} \! Y_l^m(\vn') \, \d s(\vn') \\
	&= P_l(0)\,Y_l^m  (\vartheta,\varphi),
\end{split}
\end{equation}
with  $P_l (0) $ the Legendre polynomial of degree $l$ evaluated at $0$. We can therefore apply the Funk transform to a function expressed in a spherical harmonic basis by a simple transformation of the coefficients $a_l^m \rightarrow  P_l(0)\,a_l^m$.

\subsubsection {Anti-Symmetrization}
We have $Y_l^m(\pi-\vartheta,\varphi+\pi)=(-1)^l Y_l^m(\vartheta,\varphi)$. We therefore anti-symmetrize the orientation distribution, see Eq.~\!(\ref{eq:Antisymmetrize}), via $a_l^m\rightarrow \frac{(1-(-1)^l)}{2} a_l ^ m$.

\subsubsection {Making Rotated Wavelets}\label{sssect:rotate}
To make the rotated versions $\psi_\vn$ of wavelet $\psi$ we have to find $h_\vn$ in $\hat{\psi}_\vn (\vomega) =g (\rho)\,h_\vn (\vartheta,\varphi)$. To achieve this we use the steerability of the spherical harmonic basis. Spherical harmonics rotate according to the irreducible representations of the SO(3) group $D_{m,m '} ^ l (\gamma,\beta,\alpha)$ (Wigner-D functions \cite{Griffiths2016}):
\begin{equation}
\!\!\!\! \left(\cR_ {\mR_{\gamma ,\beta ,\alpha}}Y_l^m \right) (\vartheta ,\varphi )= \!\!\!\! \sum _{m'=-l}^l \!\!\! D_{m,m'}^l(\gamma ,\beta ,\alpha )Y_l^{m'}(\vartheta ,\varphi ).
\end{equation}
Here $\alpha, \beta $ and $\gamma$ denote the Euler angles with counterclockwise rotations, where we rely on the convention $\mR_{\gamma ,\beta ,\alpha} =\mR_ {\ve_z,\gamma} \mR_ {\ve_y,\beta} \mR_ {\ve_z,\alpha}$. This gives
\begin{equation}
\begin{split}
	h_\vn (\vartheta,\varphi)&=\left( \cR_{\mR_{\gamma ,\beta ,\alpha }} h \right) (\vartheta ,\varphi ) \\
	& \!\!\!\!\!\!\!\!\!\!\!\!\!\!= \sum _{l=0}^L  \sum _{m=-l}^l \sum _{m'=-l}^l c_l^m D_{m, m'}^l(\gamma ,\beta ,\alpha )Y_l^{m'}(\vartheta ,\varphi ),
\end{split}
\label{eq:Rotationh}
\end{equation}
where $c_l^m$ are the coefficients of $h$ given by
\begin{equation}
	 c_l^m = P_l(0)\, a_l^m + \frac{(1-(-1)^l)}{2} a_l ^ m.
 \end{equation}
Because both anti-symmetrization and Funk transform preserve the rotational symmetry of $A$, we have $h(\vartheta,\varphi) =\sum_{l = 0} ^L  c_l^0 Y_l ^ 0 (\vartheta,\varphi)$, and Eq.~\!(\ref{eq:Rotationh}) reduces to
\begin{equation}
h_\vn (\vartheta,\varphi) = \overset{L }{\sum _{l=0}} \overset{l}{\sum _{m'=-l} } c_l^0 D_{0, m'}^l(\gamma ,\beta ,0 )\,Y_l^{m'}(\vartheta ,\varphi ) .
\end{equation}

The filters from this section are summarized in the following result:
\begin{result}\label{res:Result1}
Let $A:S^2 \rightarrow \R^+$ be a function supported mainly in a sharp convex cone around the $z$-axis and symmetricaly around the $z$-axis and $g$ as radial function of Eq. \eqref{eq:radialFunctionOfPsi}. Then $A$ provides our wavelet $\hat{\psi}$ in the Fourier domain via
\begin{equation}\label{eq:psiHatFromASummary}
	\hat{\psi}(\vomega) = g (\rho) \left( F A (\vn_\omega) + A (\vn_\omega) - A (-\vn_\omega) \right),
\end{equation}
with $\vomega=\rho \, \vn_{\vomega} = \rho \, \vn(\vartheta,\varphi)$. The real part of $\psi$ is a tube detector given by
\begin{equation}
	\Real(\psi) = \cF^{-1} \left(\vomega \mapsto g(\rho)  (F A) (\vn_\omega ) \right).
\end{equation}
The imaginary part of $\psi_{\vn}$ is an edge detector given by
\begin{equation}
	\Imag(\psi) = \frac{1}{i} \cF^{-1} \left(\vomega \mapsto g(\rho)  \left(A (\vn_\omega) - A (-\vn_\omega) \right) \right).
\end{equation}
When expanding the angular part in spherical harmonics up to order $L$ and choosing $A = G_{s_o}^{S^2}$:
\begin{equation}\label{eq:SphericalHarmonicExpansionDIfKernelInResult1}
	\begin{split}
		A(\vn(\vartheta,\varphi)) &= \sum \limits_{l=0}^L  a_{l}^{0} Y_l^0(\vartheta,\varphi), \\
		a_{l}^{0} &= \sqrt{\frac{2 l+1}{4 \pi }} e^{-l(l+1) s_o },
	\end{split}
\end{equation}
we have the following wavelet in the Fourier domain
\begin{equation}\label{eq:cakeWaveletFSummary1}
	\hat{\psi} (\vomega) = g (\rho) \overset{L }{\sum _{l=0}} c_l^0 Y_l^{0}(\vartheta ,\varphi ) ,
\end{equation}
and the coefficients of $A$ and $\hat{\psi}$ relate via
\begin{equation}
	c_l^0 = \left(P_l(0) + \frac{(1-(-1)^l)}{2} \right)a_l^0.
\end{equation}
We obtain rotated versions of our filter via
\begin{equation}\label{eq:cakeWaveletFSummary}
	\hat{\psi}_{\vn} (\vomega) = g (\rho) \overset{L }{\sum _{l=0}} \overset{l}{\sum _{m'=-l} }  c_l^0 D_{0, m'}^l(\gamma ,\beta ,0 )Y_l^{m'}(\vartheta ,\varphi ) ,
\end{equation}
with $\vn=\vn(\beta,\gamma)$.

As we do not have analytical expressions for the spatial wavelets $\psi_\vn$, we sample the filter in the Fourier domain using Eq. \eqref{eq:cakeWaveletFSummary} and apply a DFT afterwards. The wavelet $\psi$ is a proper wavelet with fast reconstruction property (Def.~\!\ref{def:properwavelet}).
\end{result}

\begin{remark}
The heat kernel on $S^2$ is given by $G^{S^2}_{s_o} (\vn(\vartheta,\varphi)) = \sum \limits_{l=0}^\infty  a_{l}^{0} Y_l^0(\vartheta,\varphi)$ with coefficients given by Eq. \eqref{eq:SphericalHarmonicExpansionDIfKernelInResult1}. Because of the exponential decay with respect to $l$ we can describe the diffusion kernel well with the first few coefficients. In all experiments we truncate at smallest $L$ such that $a_L^0/a_0^0<10^{-3}$ (e.g. $L = 21$ for $s_o = \frac{1}{2} (0.25)^2$).
\end{remark}

\begin{figure*}[p]
\centering
	\includegraphics[width=0.9 \hsize]{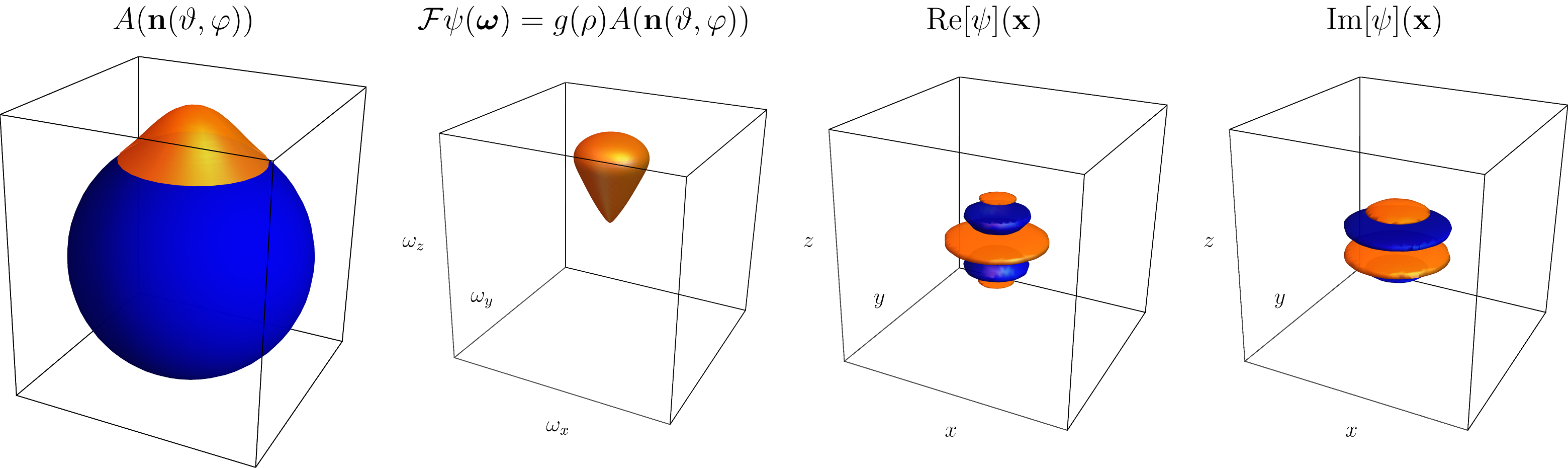}
	\caption{When directly setting orientation distribution $A$ of Eq. \eqref{eq:Reconstruction2Approximation2} as angular part of the wavelet $h$ we construct plate detectors. From left to right: Orientation distribution $A$, wavelet in the Fourier domain, the plate detector (real part) and the edge detector (imaginary part). Orange: Positive iso-contour. Blue: Negative iso-contour. Parameters used: $s_o=\frac{1}{2}(0.25)^2,\sigma_{\erf}=3,\gamma=0.85$ and evaluated on a grid of 51x51x51 pixels.}
	\label{fig:cakePieceCreatesPlateDetector}
\end{figure*}

\begin{figure*}[p]
\centering
\includegraphics[width=0.75 \hsize]{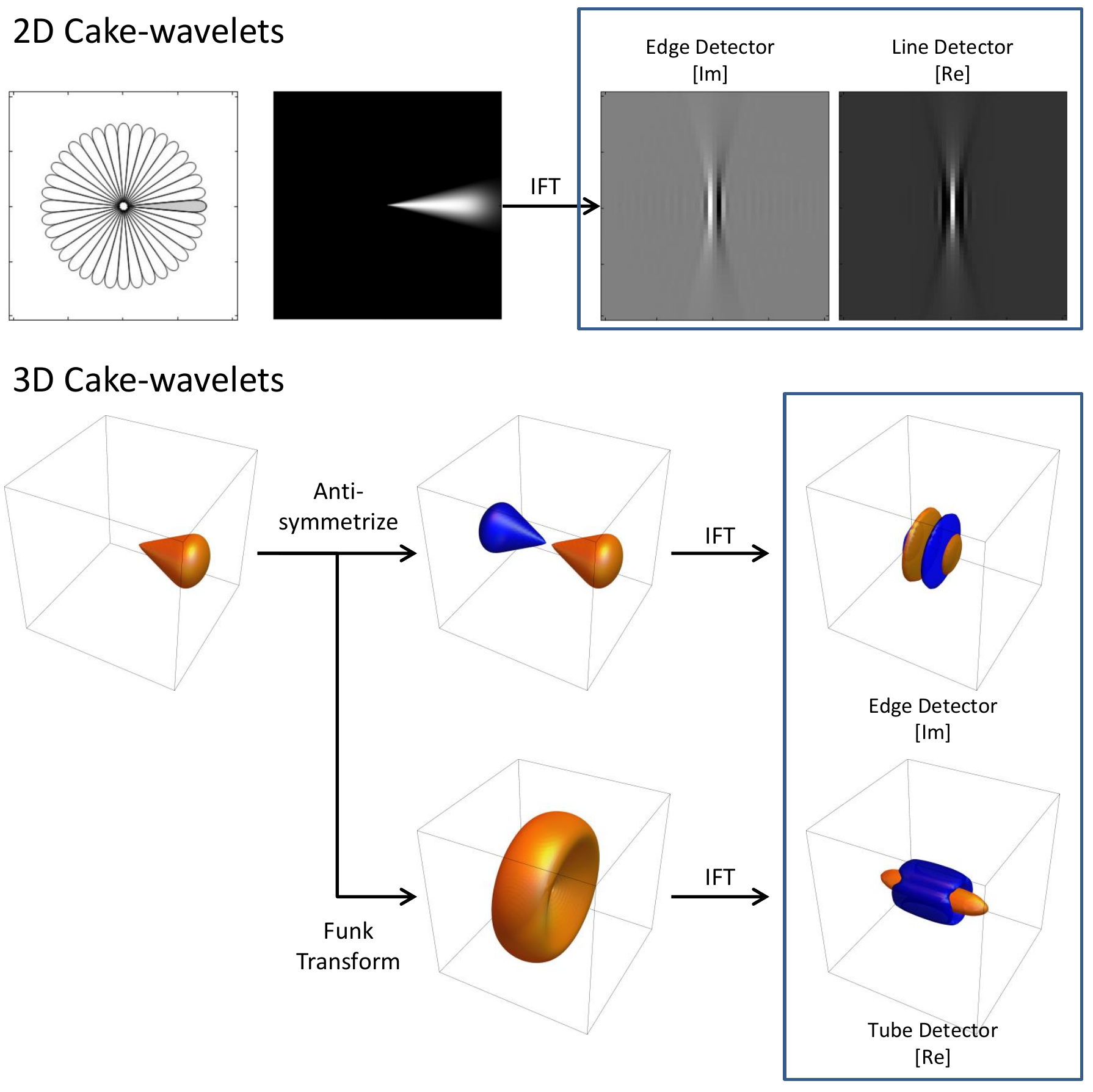}%
\caption{Cake-Wavelets. \emph{Top} 2D cake-wavelets. From left to right: Illustration of the Fourier domain coverage, the wavelet in the Fourier domain and the real and imaginary part of the wavelet in the spatial domain \cite{BekkersJMathImagingVis2014}. \emph{Bottom} 3D cake-wavelets. Overview of the transformations used to construct the wavelets from a given orientation distribution. Upper part: The wavelet according to Eq.~\!(\ref{eq:Antisymmetrize}). Lower part: The wavelet according to Eq.~\!(\ref{eq:CakeWaveletRe}). IFT: Inverse Fourier Transform. Parameters used: $s_o=\frac{1}{2}(0.25)^2,\gamma=0.85$ and evaluated on a grid of 81x81x81 pixels.}%
\label{fig:CakeWavelets}%
\end{figure*}

\begin{figure*}[ht]
\centering
\includegraphics[width=0.8 \hsize]{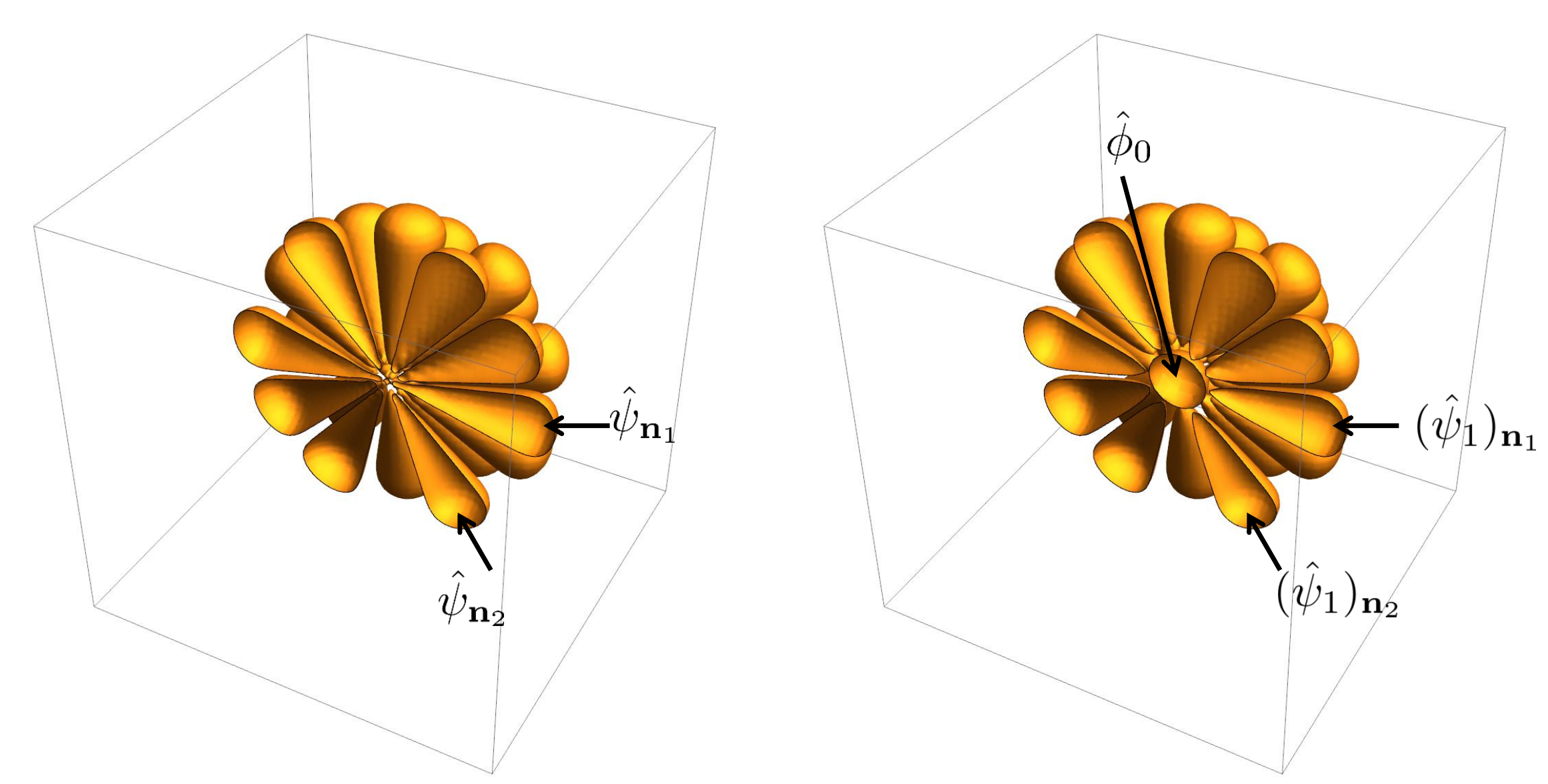}%
\caption{Coverage of the Fourier domain before and after splitting according to Section \ref{sssect:LowFrequencyComponents}. \emph{Left:} The different wavelets cover the Fourier domain. The "sharp" parts when the cones reach the center however cause the filter to be non-localized, which was solved in earlier works by applying a spatial window after filter construction. \emph{Right:} By splitting the filter in lower and higher frequencies we solve this problem. In the figure we show $g(\rho)A(\vn(\vartheta,\varphi))$ for the different filters, before applying the Funk transform to the orientation distribution $A$.}%
\label{fig:CakeWaveletsSplitting}%
\end{figure*}

\subsection{Stability of the Discrete Transformation With Fast Reconstruction for Filters of Result \ref{res:Result1}}
To make a fast reconstruction by summation possible (requirement \ref{requirement:ReconstructionSummation}) we need a proper wavelet set with the fast reconstruction property (Definition \ref{def:properwaveletSet}) with $N_\psi^d \approx 1$. We now focus on finding bounds for $N_\psi^d$ such that we can choose our parameters in a deliberate way.

\begin{proposition}\label{proposition:FastReconstruction}
Let $\{\psi_{\vn(\beta_i,\gamma_i)} \,|\, i=1 \dots N_o\}$ be a set of wavelets constructed via the procedure in Result~\ref{res:Result1}. Then we have bounds on $N_{\psi}^d$ given by
\begin{multline} \label{eq:boundsNPsiD}
1-\sum _{l=1}^L \| \vd_{l} \| \sqrt{\frac{2l+1}{4 \pi}} \leq N_\psi^d(\vomega) \leq 1 + \sum _{l=1}^L \| \vd_{l} \| \sqrt{\frac{2l+1}{4 \pi}},\\
\textrm{for all } \vomega \in B_{\rho_0}
\end{multline}
with $\vd_l=(d_{l}^{m})_{m=-l}^l$, $d_{l}^{m}= \sum \limits_{i=1}^{N_o} c_{l}^{0} \cdot \Delta_i \cdot \mathcal{D}^{l}_{0,m}(0,\beta_i,\gamma_i)$ and here the norm is the $\ell_2$-norm on $\mathbb{C}^{2l+1}$.
\end{proposition}
\begin{proof}
First we expand function $N_\psi^d$ in spherical harmonics:
\begin{align}
	N_\psi^d &(\vomega) =  \sum_{i=1}^{N_o} \cFRRR [\psi_{\vn_i}](\vomega) \Delta_i
	=  g (\rho)  \sum_{i=1}^{N_o} h_{\vn_i} (\vartheta,\varphi) \Delta_i \notag\\
	&= g (\rho)   \overset{L }{\sum _{l=0}} \overset{l}{\sum _{m'=-l} } \underbrace{ \sum_{i=1}^{N_o} c_l^0 D_{0, m'}^l(0 ,\beta_i ,\gamma_i ) \Delta_i }_{d_l^{m'}  } Y_l^{m'}\!(\vartheta ,\varphi ) \notag\\ \notag\\[-9mm]
	&=  g (\rho)   \overset{L }{\sum _{l=0}} \overset{l}{\sum _{m'=-l} } d_l^{m'} Y_l^{m'}(\vartheta ,\varphi )
\end{align}
We have $g (\rho) = 1$ for $\rho=\|\vomega\| \leq \rho_0$, but we still need to quantify the angular part.
We define
$\mathbf{Y}_{l}^N= (Y_{l}^{-l},Y_{l}^{-l+1},\dots,Y_{l}^{l-1},Y_{l}^{l})$, so that
\begin{equation}
	\begin{split}
		\sum _{l=0}^L \sum _{m'=-l}^l d_l^{m'} Y_l^{m'} & (\vartheta ,\varphi ) = \sum _{l=0}^L \vd_{l} \cdot \mathbf{Y}_l (\vartheta ,\varphi ) \\
		&\!\!\!= Y_0^0 (\vartheta ,\varphi ) d_0^0 +  \sum _{l=1}^L \vd_{l} \cdot \mathbf{Y}_l (\vartheta ,\varphi ) \\
		&\!\!\!= 1+  \sum _{l=1}^L \vd_{l} \cdot \mathbf{Y}_l (\vartheta ,\varphi )
	\end{split}
\end{equation}
This varying component should remain small. We use the Cauchy-Schwarz inequality for each order l:
\begin{multline}
	 \left| \sum_{l=1}^L \vd_{l} \cdot \mathbf{Y}_l (\vartheta ,\varphi ) \right|  \leq   \sum _{l=1}^L \left| \vd_{l} \cdot \mathbf{Y}_l (\vartheta ,\varphi ) \right|  \\
	  \leq   \sum _{l=1}^L \| \vd_{l} \| \| \mathbf{Y}_l (\vartheta ,\varphi ) \|  = \sum _{l=1}^L \| \vd_{l} \| \sqrt{\frac{2l+1}{4 \pi}},
\end{multline}
 from which (\ref{eq:boundsNPsiD}) follows. $\hfill \Box$
\end{proof}
 See Fig. \ref{fig:figureBounds} for visual inspection of bounds of $M_\psi^d$ and $N_\psi^d$, and numerical results for the bounds of $N_\psi^d$.

\begin{corollary}
	Given our analytical bounds \eqref{eq:boundsNPsiD} from Proposition \ref{proposition:FastReconstruction} and $N_o =42$, we can guarantee that our set of wavelets from Result \ref{res:Result1} is a proper wavelet set with fast reconstruction property according to Def. \ref{def:properwaveletSet} with $\epsilon = 0.05$ when choosing parameter $s_0 \gtrapprox  0.04$. In practice we have a proper wavelet set with fast reconstruction property already for smaller values of $s_o$ (see Fig. \ref{fig:figureBounds}).
\end{corollary}

\begin{figure*}[ht]
	\centering
	\includegraphics[width= 0.42\hsize]{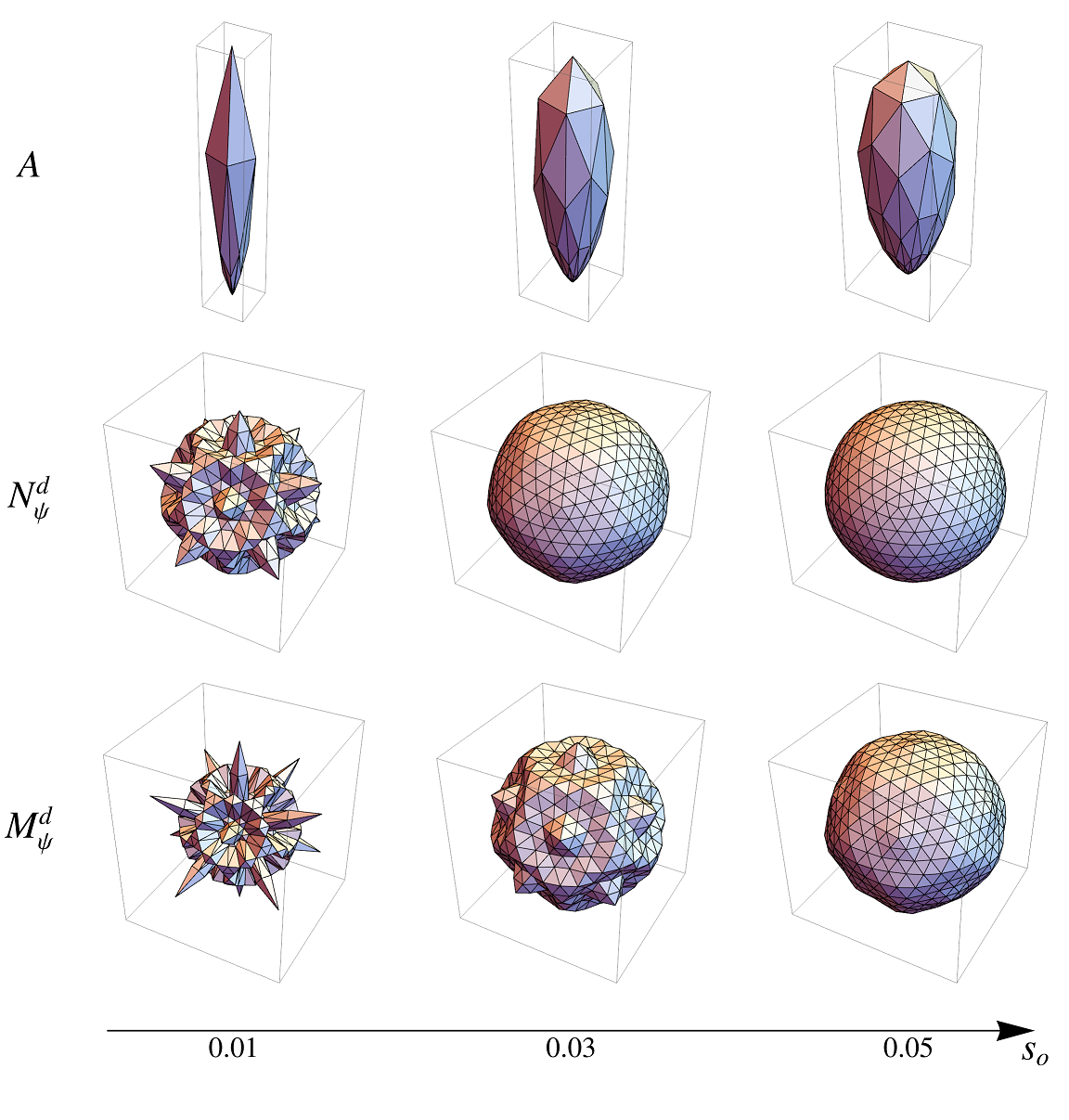}
	\includegraphics[width= 0.53\hsize]{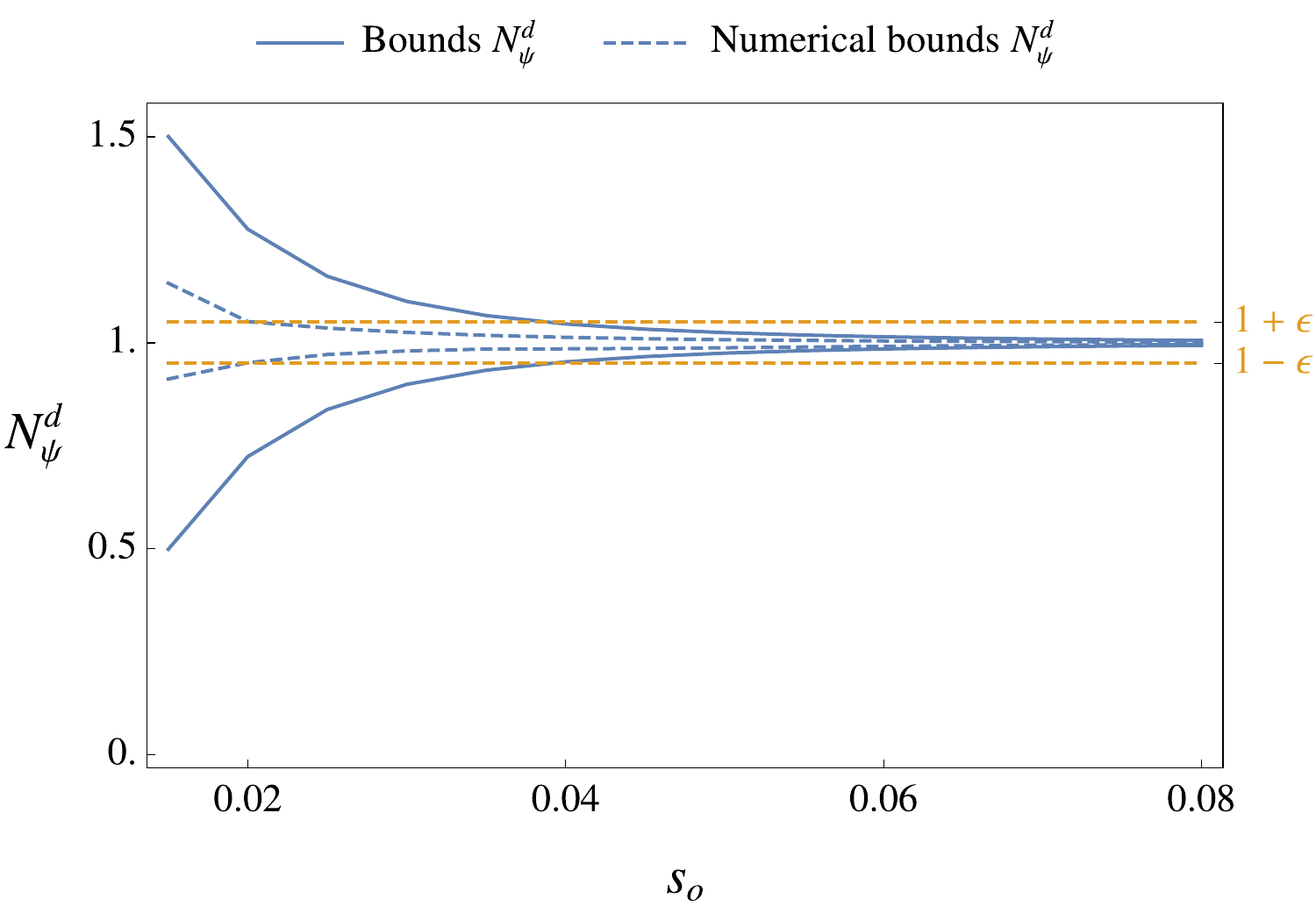}
	\caption{Inspection of the stability of the transformation for different values of $s_o$ given a orientation distribution $A = G_{s_o}^{S^2}$ and for $N_o=42$. \emph{Left:} Spherical plot of $A$ and the angular part of polar separable function $N_\psi^d$ and $M_\psi^d$.  Orientation coverage is more uniform as the plots for $N_\psi^d$ and $M_\psi^d$ look more like a ball. \emph{Right:} The upper and lower bounds of $N_\psi^d$. Comparison of the bounds according to Eq. \eqref{eq:boundsNPsiD} (filled blue line) and numerical results (dashed blue line) of the bounds by a very fine sampling of the sphere ($\approx 500$ orientations). Furthermore, we show $1+\epsilon$ and $1-\epsilon$ (orange dashed lines) for $\epsilon=0.05$.}
	\label{fig:figureBounds}
\end{figure*}

\section{Wavelet Design with Continuous Fourier Transform and Analytical Description in the Spatial Domain}\label{sect:WaveletZernike}
In the previous section we described wavelets which were analytical in the Fourier domain, and were sampled and inverse discrete Fourier transformed to find the wavelets in the spatial domain.

To get more control on the wavelet properties in both the spatial and Fourier domain it would be convenient to have an analytical description of the wavelets in both domains. This could be achieved by expressing the wavelets in a basis for which we have analytical expressions for the Fourier transform. We will now discuss 2 such options for the basis and describe filters expressed in them.
\subsection{A review on Expansion in the Harmonic Oscillator Basis}
The first basis in which we could expand our wavelets are the eigenfunctions of the harmonic oscillator $H=\|\vx\|^2-\Delta$, which was also studied in \cite{ThesisDuits,ThesisAlmsick}. We will quickly review this work, show the problems which were encountered when using this basis, before moving onto an alternative basis in the next section which aims to solve these problems.

When using the eigenfunctions of the harmonic oscillator as a basis, the idea is that operator $H$ and the Fourier transform commute ($\cF \circ H= H \circ \cF$) and eigenfunctions of $H$ are also eigenfunctions of $\cF$. We then expand our wavelets in these eigenfunctions restricting ourselves to eigenfunctions which are symmetric around the $z$-axis: the spherical harmonics with $m=0$. The wavelet is then given by
\begin{equation} \label{eq:WaveletHarmonicOscillator}
\begin{split}
		\psi(\vx)&=\sum_{n=0}^{\infty} \sum_{l=0}^{L} \alpha_{l}^n g_n^l(r) \, Y_l^0(\theta,\phi), \\
		\hat{\psi}(\vomega)&=\sum_{n=0}^{\infty} \sum_{l=0}^{L} \alpha_{l}^n (-1)^{n+l} i^l g_n^l(\rho) \, Y_l^0(\vartheta,\varphi),
\end{split}
\end{equation}
with $Y_l^m$ the spherical harmonics, $(r,\theta,\phi)$ and $(\rho,\vartheta,\varphi)$ spherical coordinates for $\vx$ and $\vomega$ respectively, i.e.,
\begin{equation} \label{eq:ballcoordinates}
\begin{split}
	\vx &= (r  \sin \theta \cos \phi, r \sin \theta \sin \phi, r \cos \theta ), \\
\vomega &= (\rho \sin \vartheta \cos \varphi,\rho \sin \vartheta \sin \varphi,\rho \cos \vartheta),
\end{split}
\end{equation}
and $g_n^l$ given by
\begin{equation} \label{eq:WaveletHarmonicOscillator2}
	\begin{split}
		g_n^l(\rho)   &=\frac{1}{\rho} E_{n}^{l+\frac{1}{2}}(\rho), \\
		E_{n}^{\nu}(\rho) &= \left(\frac{2 (n !)}{\Gamma(n+\nu+1)}\right)^{\frac{1}{2}} \rho^{\nu+\frac{1}{2}} e^{-\frac{\rho^2}{2}} L_{n}^{(\nu)}(\rho^2),
	\end{split}
\end{equation}
where $L_{n}^{(\nu)}(\rho)$ is the generalized Laguerre polynomial. We then choose the case with least radial oscillations $\alpha_l^n=\alpha^l \delta_n^0$. If we then choose
\begin{equation}
		\alpha_l=\sqrt{\frac{\Gamma(l+\frac{3}{2})}{\Gamma(l+1)}},
\end{equation}
we have that $M_\psi(\vomega)$ approximates 1 for all $\vomega\in \R^3$ in the Fourier domain as $L \rightarrow \infty$ and we get the following wavelet\footnote{The series in \eqref{eq:harmonicOscillatorWavelet} converges point-wise but not in $\LL_2$-sense. So for taking the limit $L \rightarrow \infty$ one must rely on a distributional wavelet transform, since $\psi_H^{L \rightarrow \infty} \notin \LL_2(\R^3)$ but $\psi_H^{L \rightarrow \infty} \in H_{-4}(\R^3)$, i.e., it is contained in the dual of the 4$^{th}$ order Sobolev space. Such technicalities do not occur in the Zernike basis as we will see later in Section \ref{sect:WaveletZernike}.} \cite{DuitsPRIA2007}:
\begin{equation}
\begin{split}
		\psi_H(\vx) = \sum_{l=0}^{L} \frac{1}{\sqrt{l!}} r^l e^{-\frac{r^2}{2}} Y_l^0(\theta,\phi).
\end{split}
\label{eq:harmonicOscillatorWavelet}
\end{equation}
For this wavelet we have an analytical description in both spatial and Fourier domain. This wavelet, however, has some problems: 1) the wavelet is not localized and has long spatial oscillations. 2) the wavelet is not centered. Both problems can be observed in Fig. \ref{fig:HarmonicOscillator}.

\begin{figure}[tb]
	\centering
	\mbox{
	\begin{minipage}{0.49\hsize}
		\centering $\psi_H$ \\
		\includegraphics[width=1\textwidth]{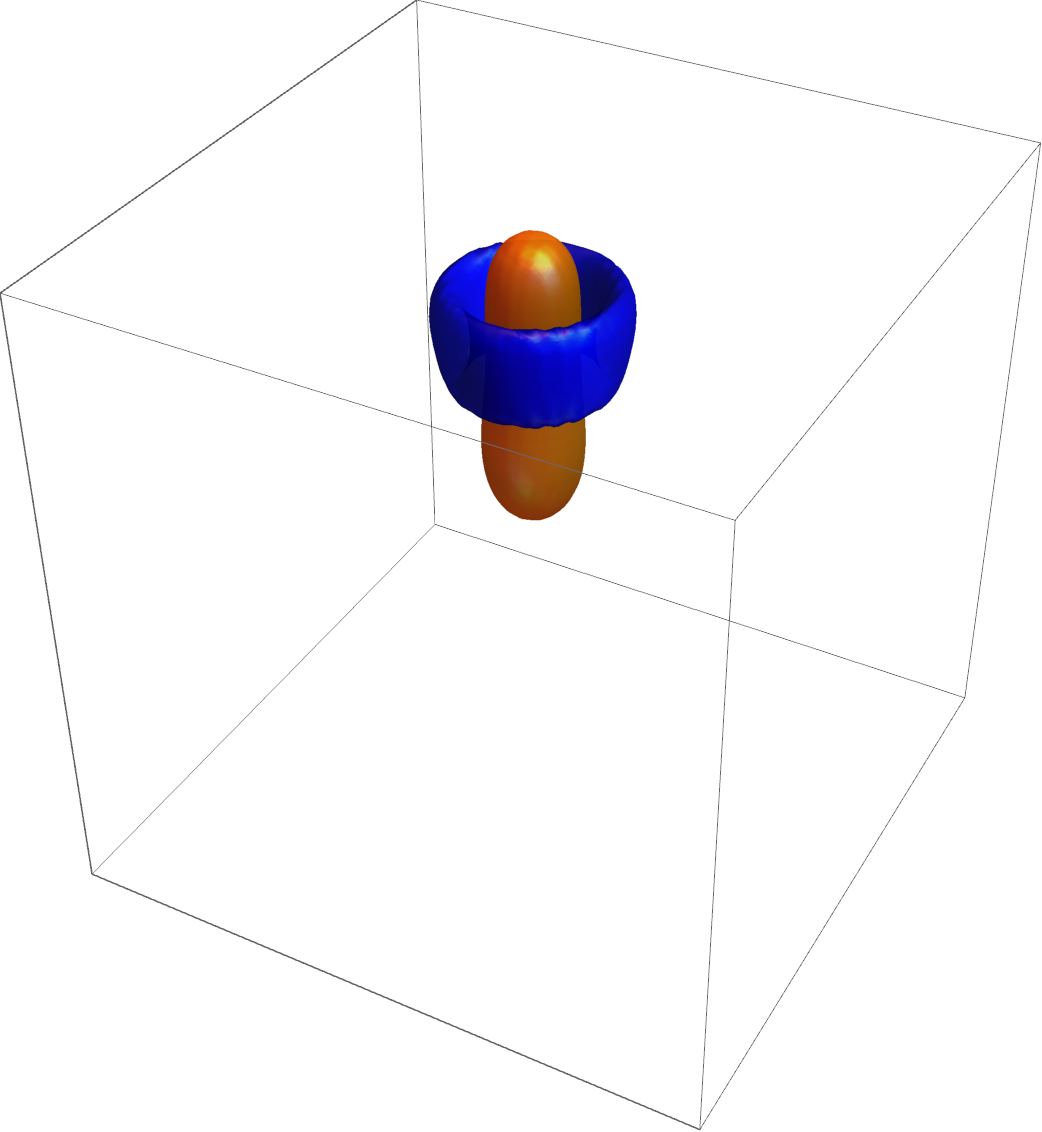}
	\end{minipage}
	\begin{minipage}{0.49\hsize}
		\centering $\psi_H(0,\cdot,\cdot)$ \\
		\includegraphics[width=1\textwidth]{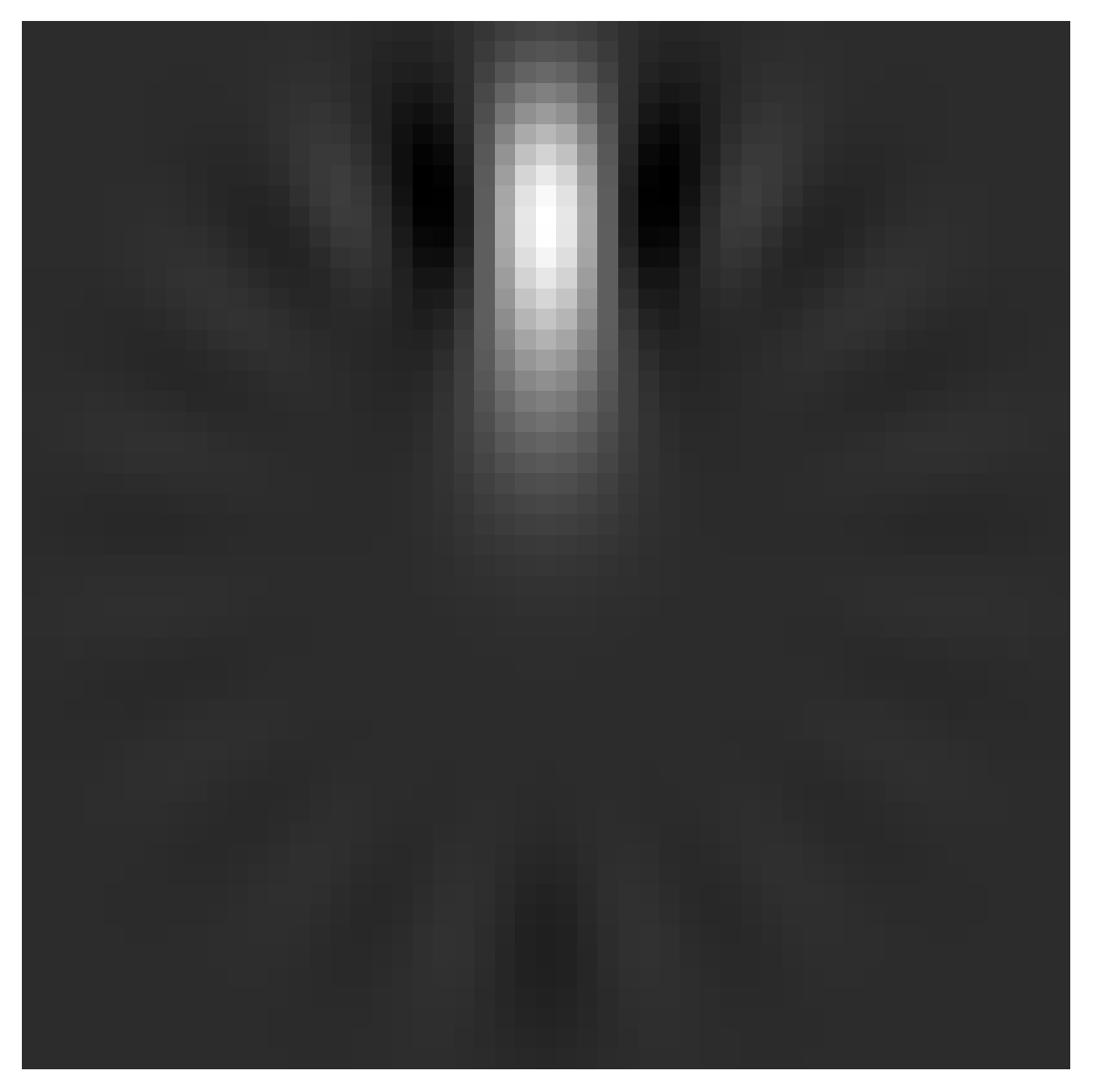}
	\end{minipage}
	}
	\caption{Wavelet expanded in the harmonic oscillator basis according to Eq. \eqref{eq:harmonicOscillatorWavelet} for $L=15$. \emph{Left} 3D visualization showing one negative (blue) and one positive (orange) isocontour. \emph{Right} Cross section of the wavelet at $x=0$.}
	\label{fig:HarmonicOscillator}
\end{figure}

\subsection{Expansion in the Zernike Basis}
The wavelets from the previous subsection have some unwanted properties such as poor spatial localization (long oscillations) and the fact that the wavelets maximum does not lie at the wavelets center. A possible explanation is that the basis used is orthogonal on the full $\LL_2(\R^3)$ space and not limited to the ball in the Fourier domain, and truncation of this basis at the Nyquist frequency could lead to oscillations. An alternative basis for the unit-ball is the Zernike basis which we can scale to be a basis in the Fourier domain for ball-limited images $f \in \LL_{2}^{\varrho}(\R^3)$, recall Eq. \eqref{eq:ballLimitedData}.

The 2D Zernike basis is often used in applications as optics, lithography and acoustics \cite{AartsJAcoustSocAm2009,Brunner1997,Born1999}, since efficient recursions can be used for calculating the basis functions and analytic formulas exist for many transformations among which the Fourier transform. The basis is therefore highly suitable for problems where the shape needs to be controlled in both domains such as in aberration retrieval where the basis is used in an inverse problem with an unknown aberration in the Fourier domain and optimization functional in the spatial domain.

Orthogonal polynomials of several variables on the unit-ball were considered by A. Louis in \cite{LouisSIAMJMathAnal1984} in the context of inversion of the Radon transform for tomographic applications. For a modern treatment of orthogonal polynomials on the unit-ball see \cite[Ch. 4]{Dunkl2014}. Here we will use the generalized Zernike functions \cite{Janssenarxiv2015}, which vanish to a prescribed degree at the boundary of the unit-ball, with explicit expansion results for for particular functions supported by the unit-ball. Since this basis is orthogonal on the unit-ball and has explicit results for the Fourier transform it is highly suitable for the application in mind. Derivations for the results used in the following section can be found in \cite{Janssenarxiv2015}.

\subsubsection{The 3D Generalized Zernike Basis}
The generalized Zernike functions are given by
\begin{equation}
Z_{n,l}^{m,\alpha}(\vomega)=R_n^{l,\alpha}(\rho) Y_l^m(\vartheta,\varphi),
\label{eq:GeneralizedZernikeFunctions}
\end{equation}
with spherical coordinates
\begin{equation}
\vomega=\rho \,\vn(\vartheta,\varphi),	
\end{equation}
integer $n,l\geq0$ such that $n=l+2p$, integer $p\geq0$ and $m=-l,-l+1,...,l$ and $\alpha>0$. The angular part is the spherical harmonic function and the radial part is given by
\begin{equation}\label{eq:RadialPartGeneralizedZernikeFunctions}
R_n^{l,\alpha}(\rho)=\rho^l(1-\rho^2)^\alpha P_{p=\frac{n-l}{2}}^{(\alpha,l+\frac{1}{2})}(2 \rho^2 -1),
\end{equation}
where $P_{p}^{(\alpha,l+\frac{1}{2})}$ denotes the Jacobi polynomial. The generalized Zernike functions are orthogonal on the unit-ball
\begin{multline}
\label{eq:GeneralizedZernikeFunctionsOrthogonality}
	\iiint\limits_{\| \vomega \| \leq 1} Z_{n_1,l_1}^{m_1,\alpha}(\vomega) \overline{Z_{n_2,l_2}^{m_2,\alpha}(\vomega)}  \frac{\d \vomega}{(1-\rho^2)^\alpha} \\ = N_{n,l}^\alpha \delta_{n_1,n_2} \delta_{m_1,m_2} \delta_{l_1,l_2},
\end{multline}
with $\delta$ the Kronecker delta and with normalization factor
\begin{equation}
N_{n,l}^\alpha=\frac{(p+1)_\alpha}{(p+l+\frac{3}{2})_\alpha} \frac{1}{2(n+\alpha+\frac{3}{2})},
\end{equation}
in which $(x)_\alpha=\frac{\Gamma(x+\alpha)}{\Gamma(x)}$  is the (generalized) Pochhammer symbol.
\paragraph {Fourier transform\\}
The inverse Fourier transform of the generalized Zernike function
\begin{equation}
(\mathcal{F}^{-1} Z_{n,l}^{m,\alpha})(\vx) =  \iiint\limits_{\| \vomega\| \leq 1} e^{2 \pi i (\vomega \cdot \vx)} R_n^{l,\alpha}(\rho) \, Y_l^m(\vartheta,\varphi) \d \vomega
\end{equation}
is given by
\begin{equation}
(\mathcal{F}^{-1} Z_{n,l}^{m,\alpha})(\vx) = 4 \pi i^l S_{n,l}^{\alpha}(2 \pi r) \, Y_l^m(\theta,\phi) ,
\end{equation}
with $\vx= r \,\vn(\theta,\phi)$ and
\begin{align}\label{eq:radialPartZernikeSpatialDomain}
S_{n,l}^{\alpha}(q) &= \int_0^1 R_n^{l,\alpha}(\rho) \, j_l (q \rho) \, \rho^2 \d \rho \\
  	&=\begin{cases}  2^\alpha (-1)^p (p+1)_\alpha \sqrt{\frac{\pi}{2 q}} \frac{J_{n+\alpha+\frac{3}{2}} (q)}{q^{\alpha+1}} & \textrm{if } q>0, \\
\frac{\sqrt{\pi} \, \Gamma(1+\alpha)}{4 \, \Gamma(\frac{5}{2}+\alpha)} \delta_{n,0} & \textrm{if } q=0. \nonumber
\end{cases}
\end{align}
Here $J_a$ and $j_a$ are the Bessel functions and spherical Bessel functions \cite{Olver2010}. For integer $\alpha$, the expression in Eq. \eqref{eq:radialPartZernikeSpatialDomain} for $q>0$ reduces to
\begin{equation}
2^\alpha (-1)^p (p+1)_\alpha \frac{j_{n+\alpha+1}(q)}{q^{\alpha+1}}.
\end{equation}

\paragraph {Expansion of separable functions\\}\label{ssect:expansionSeparableFunctions}
An additional constraint for the wavelets is that they should be separable in the Fourier domain, i.e., $(\mathcal{F}\psi)(\vomega)= F(\vomega) = A (\vartheta,\varphi) B(\rho)$. When expanding such a function in the generalized Zernike basis,
\begin{equation}
F(\vomega)=\sum_{n,l,m} c_{n,l}^{m,\alpha}(F) \, Z_{n,l}^{m,\alpha}(\vomega),
\end{equation}
we can split the coefficients in radial coefficients and angular coefficients
\begin{equation}
\begin{split}
c_{n,l}^{m,\alpha}(F)&=\frac{1}{N_{n,l}^\alpha}  \iiint\limits_{\|\vomega\| \leq 1} F(\vomega) \, \conj{Z_{n,l}^{m,\alpha}(\vomega)} \frac{\d \vomega}{(1-\rho^2)^\alpha}\\
&=a_l^m(A) \, \tilde{b}_n^{l,\alpha}(B), \; \tilde{b}_n^{l,\alpha}(B) = \frac{1}{N_{n,l}^\alpha} b_n^{l,\alpha}(B)
\end{split}
\label{eq:coefficientsSeparableFunctions}
\end{equation}
where
\begin{align}
a_l^m(A) &=  \int \limits_0^\pi  \int \limits_0^{2\pi} A(\vn(\vartheta,\varphi))\, \conj{Y_l^m(\vartheta,\varphi)} \sin \vartheta \, \d \vartheta \d \varphi, \\
b_n^{l,\alpha}(B) &= \int_0^1 B(\rho) \, R_n^{l,\alpha}(\rho) \, \frac{\rho^2 \d \rho}{(1-\rho^2)^\alpha}.
\end{align}
The coefficients $c_{n,l}^{m,\alpha}$ in \eqref{eq:coefficientsSeparableFunctions} reflect the separation of $F$ as a product of an angular and radial factor as well as a corresponding separation of the generalized Zernike basis functions in \eqref{eq:GeneralizedZernikeFunctions}. In the latter, the index $l$ appears both in the angular and radial factor. Thus we have
\begin{equation}\label{eq:SphericalHarmonicsExpansionZernike}
		A(\vn(\vartheta,\varphi))=\sum_{l,m} a_l^m  Y_l^m(\vartheta,\varphi),
\end{equation}
while for all $l=0,1,\dots$
\begin{equation}\label{eq:RadialZernikeExpansion}
    B(\rho)=\sum_{n=l,l+2,\dots} \tilde{b}_n^{l,\alpha} R_n^{l,\alpha}(\rho).
\end{equation}

For each $l$, the radial functions $R_n^{l,\alpha}$ with $n$ varying are a basis for functions defined on the interval $[ 0, 1]$. For separable functions, we expand the same radial function $B(\rho)$ for each $l$, and it can be shown that there is a recursion formula for the radial coefficients \cite{Janssenarxiv2015}.

\subsubsection{Wavelets}
We now choose an appropriate radial and angular functions for our wavelets expressed in the generalized Zernike basis.

\paragraph{Angular function for the Zernike wavelets\\}
For the angular functions we again choose orientation distribution $A(\vn(\vartheta,\varphi)) = G_{s_o}^{S^2}(\vn(\vartheta,\varphi))$ for which the spherical harmonic coefficients are given by \eqref{eq:coefficientsDifussionKernel}. After this we apply the same transformations (Funk transform and anti-symmetrization) to obtain the angular part of the wavelet.

\paragraph{Flat radial profile for all-scale transform\\}
Recall the procedure of splitting of the lowest frequencies as described in Subsection \ref{sssect:LowFrequencyComponents} resulting in filters $\psi_0$ and $\psi_1$. In this section we design a radial function for $\psi_1$ which is relevant for further processing. Furthermore we already have an analytical description for $\phi_0$, which we set to $\phi_0=G_{s_\rho}$ (see Eq. \eqref{eq:LowFrequencyPhi0}).

The radial function of $\psi_0$ should therefore approximate $B(\rho)=1-G_{s_{\rho}}(\rho)$ on the interval $[0,\varrho]$ and should smoothly go to zero when approaching the edges of the interval. For the moment, we set $\varrho=1$ and we include the scaling later. To start, we define the function
\begin{equation}\label{eq:StandardRadialProfileB}
    B_{\alpha,\beta}(\rho)=(1-\rho ^2)^{\alpha } \rho ^{\beta },
\end{equation}
see Fig. \ref{fig:StandardFlat}a for the case $\alpha=6,\beta=2$. For this function we have the following coefficients
\begin{equation}\label{eq:coefficientsRadialProfileB}
    b_{n=l+2p}^{l,\alpha,\beta} = \frac{\binom{\frac{\beta-l}{2}}{p}}{(2 \alpha+\beta+l+2 p+3) \binom{\frac{1}{2} (\beta+l+1)+\alpha+p}{\alpha+p}}.
\end{equation}
To obtain a flatter function we multiply the function $B_{\alpha,\beta}$ with a second order Taylor expansion of the reciprocal function $\rho \mapsto (B_{\alpha,\beta}(\rho))^{-1}$ around the function's maximum obtained at
\begin{equation}
     \rho_{\textrm{max}} = \left(\frac{\frac{1}{2}\beta}{\alpha + \frac{1}{2}\beta }\right)^\frac{1}{2},
 \end{equation}
see Fig. \ref{fig:StandardFlat}b. The resulting function is again a sum of functions of type \eqref{eq:StandardRadialProfileB} with different values for $\beta$, so we can find the coefficients $b_{n=l+2p}^{l,\alpha}$ for the flattened function as well. For the specific case $\beta=2$ we get the following flattened function
\begin{align}\label{eq:Bflat}
   	B^{\textrm{flat}}_{\alpha,2}(\rho) &= B_{\alpha,2}(\rho) \, B^\textrm{rec}_{\alpha,2}(\rho)\\
    &=  \frac{1}{B_{\textrm{max}}} \rho^2 (1-\rho^2)^\alpha  \left(1+ \frac{(\alpha+1)^3}{2 \alpha} (\rho^2 - \rho_{\textrm{max}}^2)^2\right), \nonumber
\end{align}
with $B_{\textrm{max}} = B_{\alpha,2}(\rho_{\textrm{max}})$. For this flattened function the coefficients are given by
\begin{equation}\label{eq:coefficientsFlatB}
    b_{n=l+2p}^{l,\textrm{flat}} = \sum \limits_{i=0}^2 c_i b_{n=l+2p}^{l,\alpha,2+2 i},
\end{equation}
with $c_i$ the coefficients of $\rho^0$, $\rho^2$ and $\rho^4$ in the second order Taylor series of the reciprocal. These coefficients follow from \eqref{eq:Bflat} and are given by
\begin{multline}
         B^\textrm{rec}_{\alpha,2}(\rho) = \sum \limits_{i=0}^2 c_i \rho^{2i} = c_0 +  c_1 \rho^2 + c_2 \rho^4, \\
\textrm{with} \quad \begin{pmatrix}
	c_0 \\ c_1 \\ c_2
\end{pmatrix}
\!\!=\!\!
\begin{pmatrix}
	 1+ \frac{(\alpha+1)^3}{2 \alpha} \rho_{\textrm{max}}^4 \\
	 -2 \frac{(\alpha+1)^3}{2 \alpha} \rho_{\textrm{max}}^2 \\
	 \frac{(\alpha+1)^3}{2 \alpha}
\end{pmatrix}.
\end{multline}

The filters from this section are summarized in the following result:
\begin{result}\label{res:Result2} (\textbf{\mbox{Analytic 3D-wavelets in Zernike basis}})\\
Let $\alpha>0$ and let $A:S^2 \rightarrow \R^+$ be a function supported mainly in a sharp convex cone around the $z$-axis and symmetrically around the $z$-axis. Then $A$ provides our wavelet $\hat{\psi}$ in the Fourier domain via Eq. \eqref{eq:psiHatFromASummary}. The real part of $\psi$ is a tube detector and the imaginary part of $\psi$ is an edge detector, see Fig. \ref{fig:CakeWavelets}. We choose radial function $g(\rho)=B^{\textrm{flat}}_{6,2}(\frac{\rho}{\rho_\cN})$ in Eq. \eqref{eq:WaveletSplitting} for $\psi_1$ and angular function $A(\vn(\vartheta,\varphi)) = G_{s_o}^{S^2}(\vn(\vartheta,\varphi))$ and expand in the generalized Zernike basis:
\begin{equation}
\hat{\psi}_1(\boldsymbol{\omega})	=
 \sum \limits_{{\scriptsize \begin{array}{c}
n\!-\!l\!=\!2p,\\
l,n \geq 0
\end{array}}}c_{n,l}^0 R_n^{l,\alpha}\left(\frac{\rho}{\rho_\cN}\right) \, Y_l^0 (\vartheta, \varphi).
\end{equation}
The coefficients $c_{n,l}^0$ follow by expanding $A$ in spherical harmonics and $B^{\textrm{flat}}_{6,2}$ in the radial Zernike polynomials, recall Eq. \eqref{eq:SphericalHarmonicsExpansionZernike} and Eq. \eqref{eq:RadialZernikeExpansion}. This yields $a_l^0$ (Eq. \eqref{eq:coefficientsDifussionKernel}) and $b_{n,l}^{\alpha}$ (Eq. \eqref{eq:coefficientsFlatB}) and coefficients $c_{n,l}^0$:
\begin{equation}
	c^{0}_{n,l}= ( P_l(0) + \frac{1-(-1)^l}{2})a_{l}^0 \tilde{b}_{n,l}^{\alpha}.
\end{equation}
The spatial wavelet is given by
\begin{equation}
\psi_1(\mathbf{x}) =
 \sum \limits_{{\scriptsize \begin{array}{c}
n\!-\!l\!=\!2p,\\
l,n \geq 0
\end{array}}}  c_{n,l}^0  4 \pi i^l S_{n,l}^{\alpha}(2 \pi r \rho_\cN) \, Y_{l}^0(\theta, \phi),
\end{equation}
with $S_{n,l}^{\alpha}$ given by Eq. \eqref{eq:radialPartZernikeSpatialDomain} and $Y_{l}^m$ the spherical harmonics of Eq. \eqref{eq:definitionSphericalHarmonics}. Then we obtain rotated filters via
\begin{align}\label{eq:ZernikeSpatialFilterSummary}
\psi_{1,\vn}(\mathbf{x}) \!&= \!\!\!\!\!\!\!
 \sum \limits_{{\scriptsize \begin{array}{c}
n\!-\!l\!=\!2p,\\
l,n \geq 0
\end{array}}} \! \sum \limits_{m'=-l}^l \! (c_\vn)_{n,l}^{m'} \, 4 \pi \, i^l \, S_{n,l}^{\alpha}(2 \pi r \rho_\cN) \, Y_{l}^{m'}\!(\theta,\phi), \nonumber \\
& \qquad \textrm{with } (c_{\vn (\beta,\gamma)})_{n,l}^{m'} =c_{n,l}^0 D_{0, m'}^l(\gamma ,\beta ,0 ).
\end{align}
Since now we do have analytical expressions for the spatial filter, in contrary to the filters from Section \ref{sect:WaveletDFT}, we sample the filters in the spatial domain using Eq. \eqref{eq:ZernikeSpatialFilterSummary}.
The filter is a proper wavelet with fast reconstruction property (recall Def.~\ref{def:properwavelet}).
\end{result}

\begin{figure}[t]
		\centering{
		\includegraphics[width=0.49\hsize]{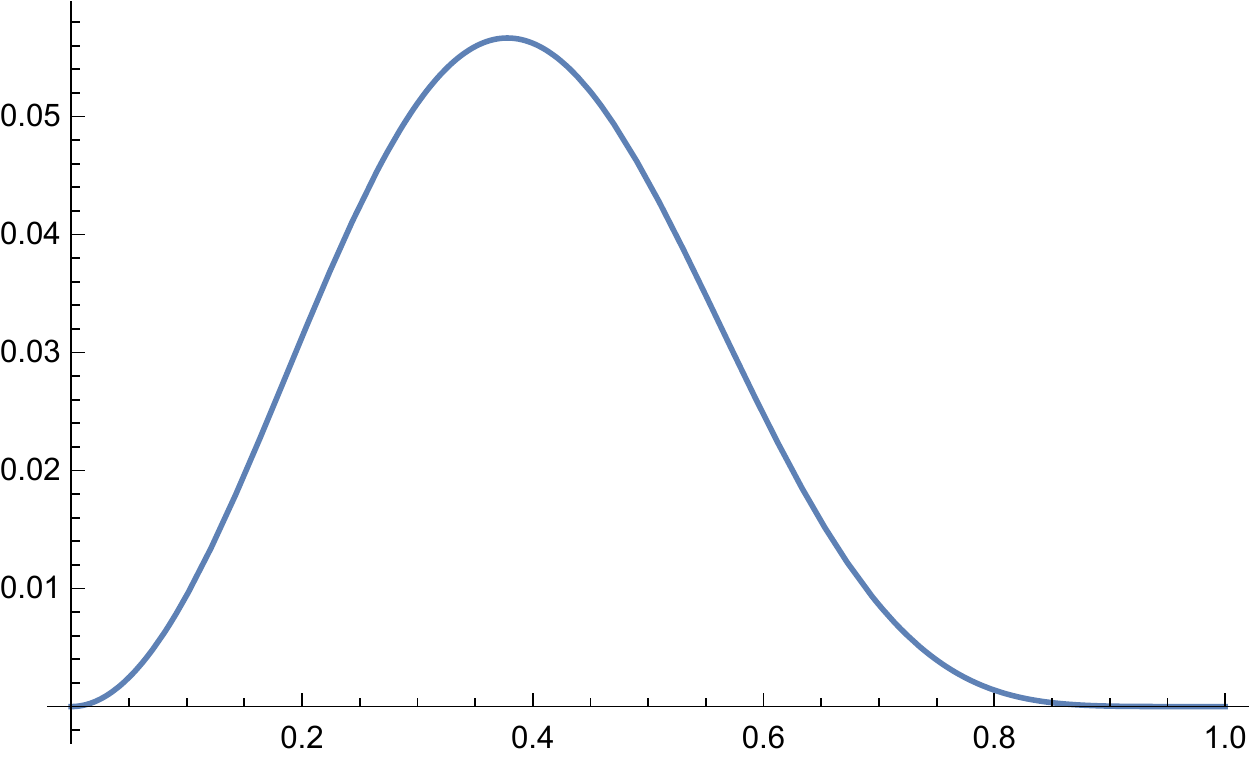}
		\includegraphics[width=0.49\hsize]{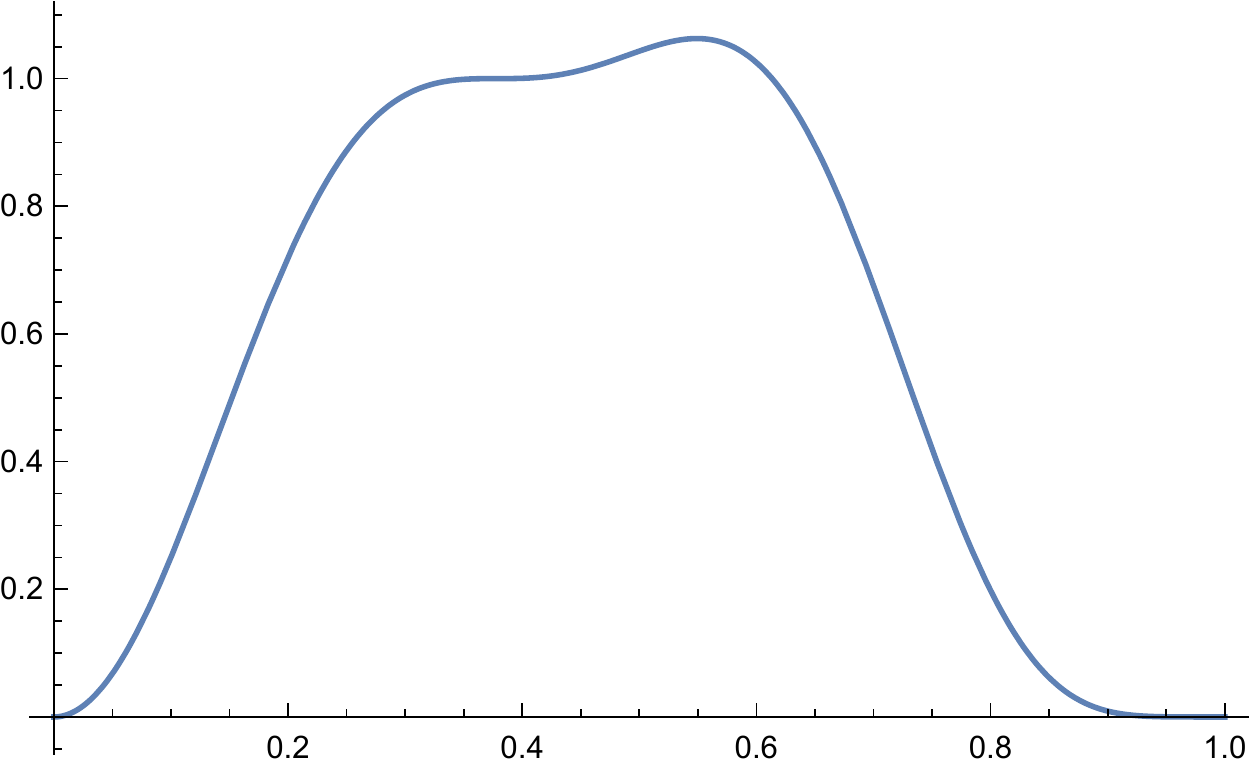}
		}
		\caption{\emph{Left} Function $B_{6,2}(\rho)=(1-\rho ^2)^{6 } \rho ^{2}$. \emph{Right} Flattened function which is obtained from $B_{\alpha,\beta}(\rho)$ by multiplying with the second order Taylor approximation of its reciprocal around the function maximum: $B_{6,2}^{\textrm{flat}}(\rho)=\frac{1}{B_\textrm{max}}(1+\frac{7^3}{12}(\rho^2-\frac{1}{7})^2)B_{6,2}(\rho)$.}
		\label{fig:StandardFlat}
\end{figure}

\section{Experiments}\label{sect:Experiments}
Before considering applications of the filters we first compare filters obtained by DFT (Section \ref{sect:WaveletDFT}) to filters expressed in the generalized Zernike basis (Section \ref{sect:WaveletZernike}) and inspect the quality of the reconstruction.
\subsection{Comparison of wavelets obtained via DFT and analytical expressions using the Zernike basis}
First we compare the filters obtained by sampling in the Fourier domain followed by a DFT (Section \ref{sect:WaveletDFT}) to the filters obtained by expansion in the Zernike basis (Section \ref{sect:WaveletZernike}). Settings were chosen such that the radial functions of both wavelets matched best and the same settings for the angular function were used. In Fig. \ref{fig:DFTvsZernikeFilter} we show that the filters are very similar in shape. We see no major artifacts caused by sampling followed by an inverse DFT.

\begin{figure}[htbp]
	\centering
	\includegraphics[width=0.95\columnwidth]{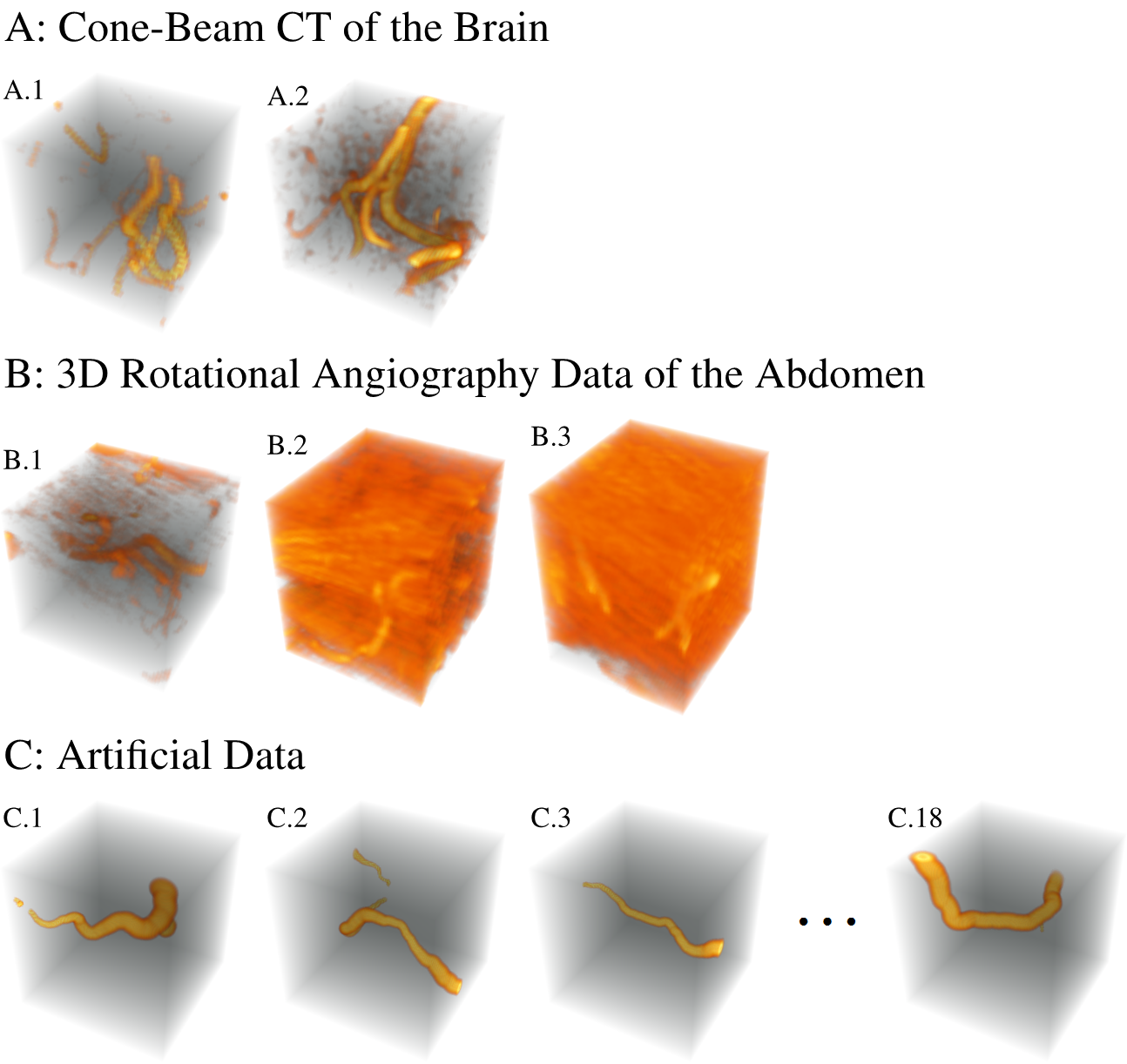}
	\caption{Overview of the datasets used in our experiments.}
	\label{fig:dataOverview}
\end{figure}

\begin{figure*}[p]
	\centering
	\includegraphics[width=0.7\textwidth]{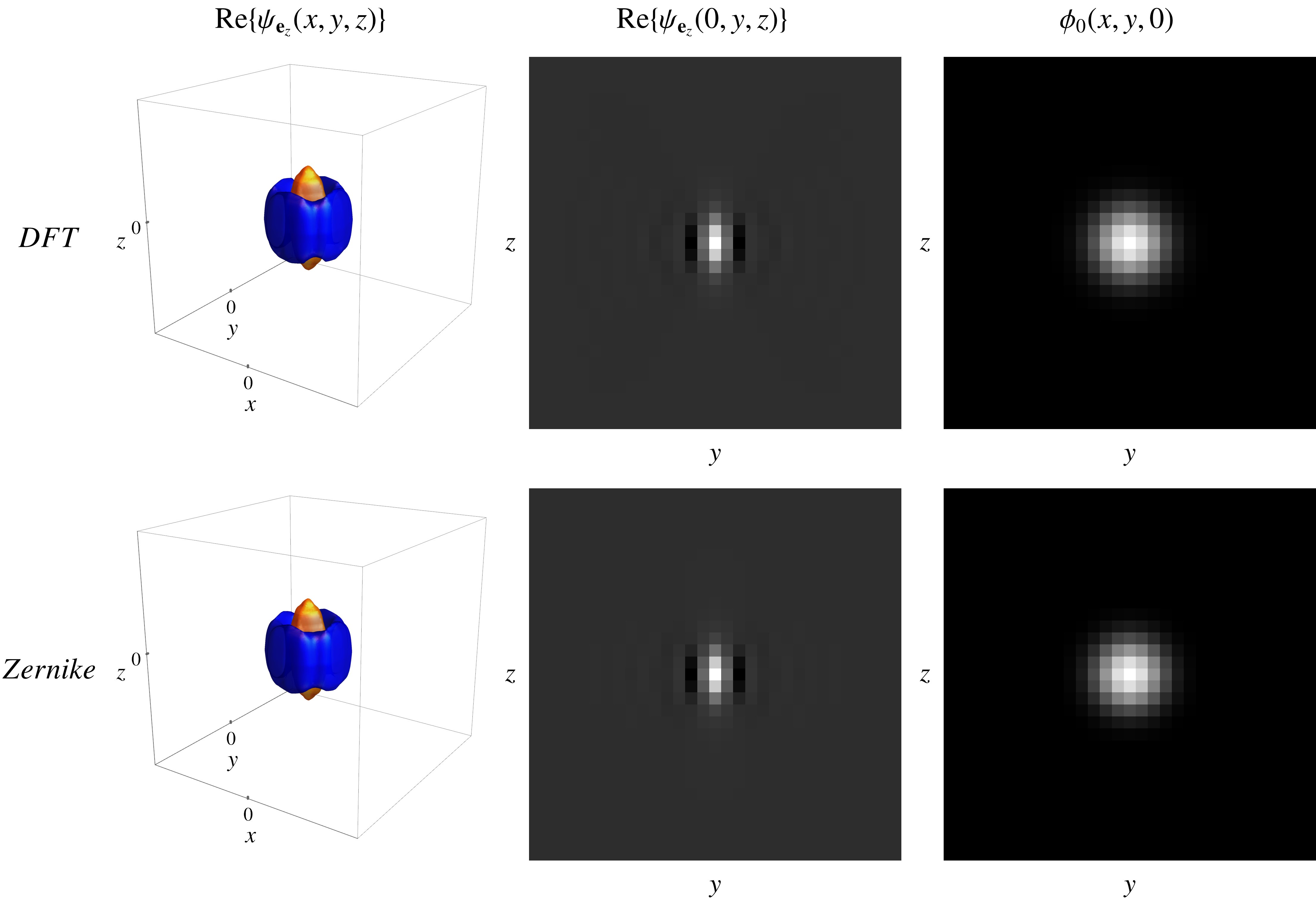}
	\caption{Comparison of the filters obtained by sampling in the Fourier domain and performing an inverse-DFT (Result \ref{res:Result1} in Section \ref{sect:WaveletDFT}) and the filters expressed in the generalized Zernike-basis (Result \ref{res:Result2} in Section \ref{sect:WaveletZernike}). \emph{Left} Iso-contour plot of the filter aligned with the x-axis showing one positive iso-contour (orange) and one negative iso-contour (blue). \emph{Middle} Cross section of the filter for $z=0$. \emph{Right} The low-pass filter. \emph{Top} Filters according to Result \ref{res:Result1} with parameters $s_{\rho}=\frac{1}{2} (1.9)^2$ and  $\gamma=0.85$.  \emph{Bottom} The filters according to Result \ref{res:Result2} with $\alpha=3$ and $\beta=2$. Both have $s_{o}=\frac{1}{2} (0.4)^2$ and are evaluated on a grid of $31 \times 31 \times 31$ voxels.}
	\label{fig:DFTvsZernikeFilter}
\end{figure*}

\subsection{Quality of the reconstruction}
A visual inspection of the reconstruction after the transformation and reconstruction procedure can be found in Fig. \ref{fig:DFTvsZernikeReconstruction}. As expected, a small amount of regularization is observed. We see no qualitative differences between the two reconstructions.

\begin{figure*}[p]
	\centering
	\includegraphics[width=0.75\textwidth]{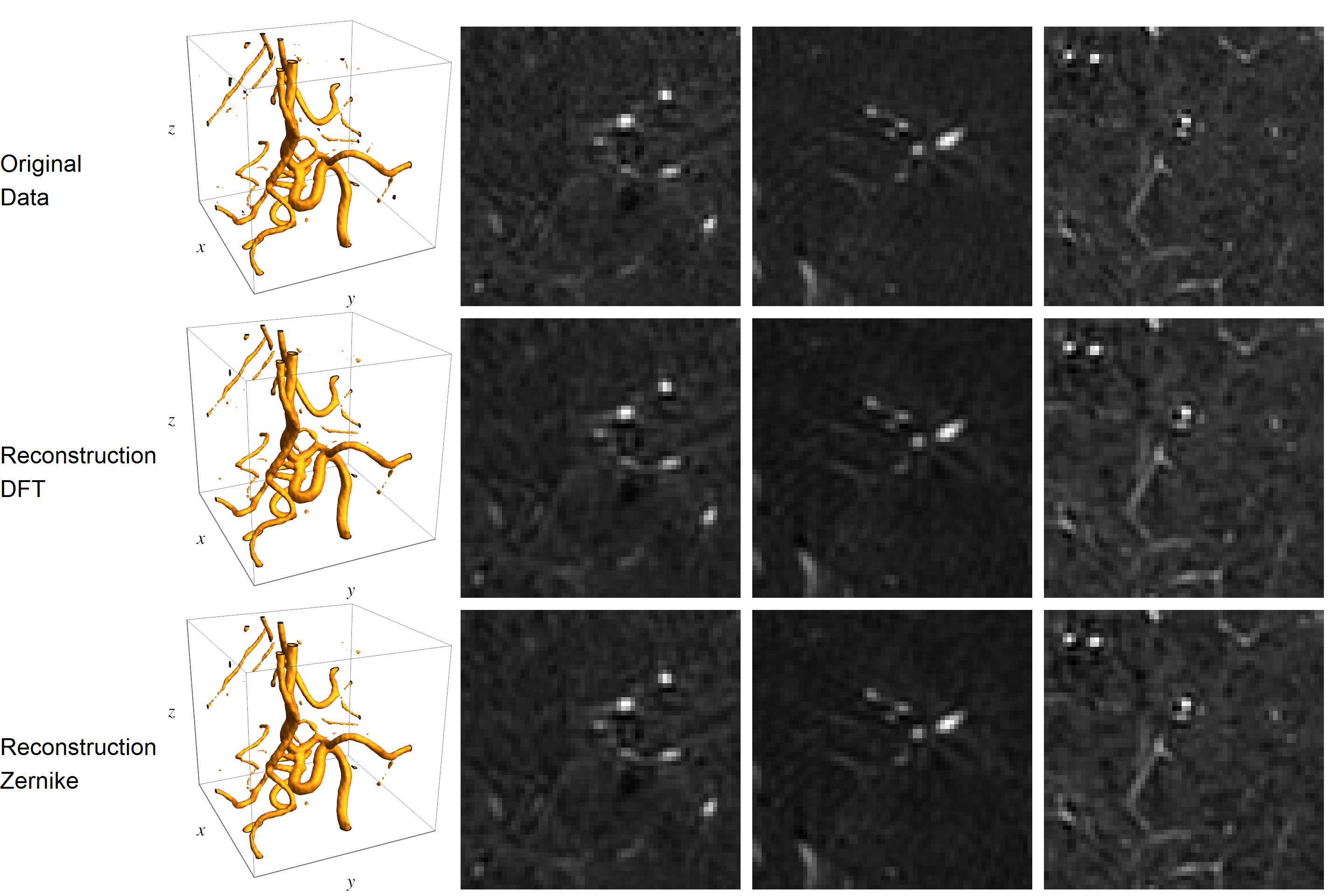}
	\caption{Comparison of construction and reconstruction of data A.1 using the different types of filters with the same settings as in Fig. \ref{fig:DFTvsZernikeFilter}. In each row, from left to right, an iso-contour of the data and 3 slices through the center of the data along the three principal axis. \emph{Top:} The original data. \emph{Middle:} the data after construction and reconstruction using the filters from Result \ref{res:Result1}. \emph{Bottom:} the data after construction and reconstruction using the filters from Result \ref{res:Result2}.}
	\label{fig:DFTvsZernikeReconstruction}
\end{figure*}

\section{Applications}\label{sect:Applications}
Next we present two applications of the orientation score transformation. For all experiments we will use the filters from Results \ref{res:Result1} and the default values of Table \ref{tab:parameters}, unless stated otherwise.


\begin{table}[h]
  \centering
  \begin{tabular}{l l r}
  	Parameter & Default value & Defining Eq. \\
	\hline
	$N_o$       & $42$     & \eqref{eq:construction1Discrete} \\
	$\gamma$      & $0.85$    & \eqref{eq:GammaNyquist} \\
	$\sigma_{erf}$       & $\frac{1}{3} (\rho_\cN-\varrho)$     & \eqref{eq:radialFunctionOfPsi} \\
	$s_\rho $ & $\frac{1}{2}(16)^2$      & \eqref{eq:WaveletSplitting} \\
	$s_o $ & $\frac{1}{2}(0.45)^2$      & \eqref{eq:orientationDistribution} \\
	Discrete wavelet size & $11 \times 11 \times 11$      & -\;\;\,	 \\
\hline
  \end{tabular}
  \caption{Default values for the parameters of the orientation score transform used in the application section.
 }
  \label{tab:parameters}
\end{table}

\subsection{Diffusion via invertible orientation scores}\label{ssect:CEDOS}

\subsubsection{Background and related methods}
Many methods exist for enhancing elongated structures based on non-linear diffusion equations. Coherence enhancing diffusion (CED) filtering \cite{WeickertIJCV1999} uses the structure tensor to steer the diffusion process to mainly apply diffusion along the elongated structures, therefore preserving the edges. One downside of this method is that at situations where multiple oriented structures occur at the same position, one of the structures gets destroyed. This renders this method not suitable for crossing structures, and in 3D data bifurcating vessels. Interesting extensions dealing with crossings by analyzing the environment using higher order derivatives have been proposed \cite{Scharr2006}.

Methods that deal with crossings by applying coherence enhancing diffusion via 2D orientation scores have been developed for 2D data \cite{FrankenIJCV2009,SharmaACHA2015}. Here, we propose an extension of coherence enhancing diffusion via 3D orientation scores to enhance elongated structures, while preserving crossings and bifurcating vessels. Preliminary results on artificial data have been shown in \cite{DuitsJMathImagingVis2016}. Here we show first results on real data, quantify the results, and furthermore add additional adaptivity to the diffusion equation.

\subsubsection{CEDOS}
We now use the invertible orientation score transformation to perform data-enhancement according to Fig. \ref{fig:OverviewOperators}. Because $\R ^ 3\times S ^ 2$ is not a Lie group, it is common practice to embed the space of positions and orientations in the Lie group of positions and rotations SE(3) by setting
\begin{equation}
\tilde{U}(\vx,\mR)=U(\vx,\mR \cdot \ve_z),\quad U(\vx,\vn)=\tilde{U}(\vx,\mR_{\vn}),
\label{eq:EquivalenceRelation}
\end{equation}
with $\mR_\vn$ any rotation for which $\mR_\vn \cdot \ve_z = \vn$. The operators $\Phi $ which we consider are scale spaces on SE(3) (diffusions), and are given by $\Phi=\Phi_t$ with
\begin{equation}
\Phi_t(U)(\vy,\vn)=\tW(\vy,\mR_\vn,t).
\label{eq:PhiDiffusion}
\end{equation}
Here $\tilde {W}$ is the solution of a non-linear diffusion equation:
\begin{equation}
\left\{\begin{split}
\frac{\partial \tW}{\partial t}(\gNoPar,t) &= \sum_{i,j=1}^6 \cA_i \atg D_{i j}(\tU) \cA_j \atg \tW(\gNoPar,t),\\
\tW(\gNoPar, 0)&=\cW_{\psi_1}[f] (\vx,\mR \cdot \ve_z), \qquad g = (\vx,\mR)
\end{split}
\right.
\label{eq:CEDOSPre}
\end{equation}
where in coherence enhancing diffusion on orientation scores (CEDOS) $D_{i j}$ is adapted locally to initial condition $\tW(\gNoPar, 0)$ based on exponential curve fits (see \cite{DuitsJMathImagingVis2016}), and with $\cA_i|_{g}=(L_g)_* \cA_i|_e$ the left-invariant vector fields on SE(3). The diffusion is better understood in locally adaptive frame $\{\cB_i\}_{i=1}^6$. Here $\cB_3$ follows from an exponential curve fits and points along our structure. $\cB_1$ and $\cB_2$ span the plane spatially perpendicular to our structure, and $\cB_4, \cB_5$ and $\cB_6$ correspond to angular diffusion (we have two angular dimensions for $\R^3 \times S^2$ but embedding in $SE(3)$ leads to a third angular dimension). Our diffusion then takes the diagonal form
\begin{equation}
\begin{split}
\frac{\partial \tW}{\partial t}(\gNoPar,t) &= \sum_{i=1}^6 D_{i i} (\tU) \cB_i \atg^2 \tW(\gNoPar,t)\\
&= D_{11}(\tU) \left( \cB_{1} \atg^2 + \cB_{2} \atg^2 \right) \tW(\gNoPar,t)  \\
&\quad + D_{33} (\tU) \cB_3 \atg^2 \tW(\gNoPar,t) \\
&\quad + D_{44} \left( \cB_{4} \atg^2 + \cB_{5} \atg^2 + \cB_{6} \atg^2 \right) \tW(\gNoPar,t)
\end{split}
\label{eq:CEDOS}
\end{equation}
where we limit ourselves to diffusion of type $D_{11} = D_{22}$, and $D_{44} = D_{55} = D_{66}$ to preserve the data symmetry of Eq. \eqref{eq:EquivalenceRelation} . For further details on how the adapted frame is obtained see \cite{DuitsJMathImagingVis2016}. We aim to enhance oriented structures and reduce noise as much as possible. Therefore, non-oriented regions are smoothed isotropically by setting
\begin{equation}\label{eq:DiffusionConstantD11}
	(D_{11}(\tU))(\gNoPar) = 1- \exp \left(- \left( \frac{c_1}{s(\tU)(\gNoPar)} \right)^2 \right),
\end{equation}
with $c_1$ a constant automatically set to the $50\%$ quantile of $s(\tU)$ and $s(\tU)(\gNoPar)$ a measure for orientation confidence given by minus the laplacian in the space orthogonal to the structure orientation $\cB_3$:
\begin{equation}
	s(\tU)(\gNoPar) = - \sum_{i \in \{1,2,4,5,6 \}} \cB_i \atg^2 \tU(\gNoPar).
\end{equation}
Since we want to stop diffusion when reaching the end of a structure, we set
\begin{equation}\label{eq:DiffusionConstantD33}
	(D_{33}(\tU))(\gNoPar) = 1- \exp \left(- \left( \frac{c_2}{\cB_3\atg \tU(\gNoPar)} \right)^2 \right),
\end{equation}
with $c_2$ a constant automatically set to the $50\%$ quantile of $|\cB_3\atg(\tU)|$.

We then obtain Euclidean invariant image processing via
\begin{equation}
\Upsilon f = \cW_\psi ^ {-1, \textrm{ext}}\circ \Phi \circ \cW_\psi f= \cW_\psi ^ {-1}\circ \mathbb {P}_\psi  \Phi \circ \cW_\psi f,
\label{eq:EuclideanInvariantImageProcessing}
\end{equation}
which includes inherent projection $\mathbb{P}_\psi$ of orientation scores, even if $\Phi=\Phi_t$ maps outside of the space of orientation scores. We write $\cW_\psi ^ {-1, \textrm{ext}}$ because we extend the inverse to $\LL_2(\R^3 \times S^2)$.

\subsubsection{Quantification via peak signal to noise ratio and contrast-to-noise ratio}
For signal $f$ with noise $N$ the noisy data is given by:
\begin{equation}
	f_N(\vx) = f (\vx) + N(\vx).
\end{equation}
Given such data, the noise is quantified via the contrast-to-noise ratio (CNR):
\begin{equation}
	CNR(f_N, f) = \frac{\max \limits_\vx f (\vx) - \min \limits_\vx f (\vx) }{\sigma(f- f_N)},
\end{equation}
where $\sigma(f- f_N)$ denotes the standard deviation of the difference signal $f-f_N$, and where the numerator denotes the contrast in our data.

For real data we do not have a ground truth $f$ but only noise signal $f_N$ and we will use the following estimation for the standard deviation of the noise and the contrast of the signal.
First we estimate the contrast by determining the average value over manually segmented parts of the vessel given by region $\Omega_S$ and background regions given by $\Omega_B$:
\begin{equation}
	\mu_S = \langle f _N |_{\Omega_S} \rangle, \qquad \mu_B = \langle f_N |_{\Omega_B} \rangle.
\end{equation}
For estimating the noise of the signal we select regions for which the signal $f$ can be expected to be constant (see Fig. \ref{fig:VesselAndBackgroundRegions}). Given such a region $\Omega_B$ we estimate the noise standard deviation by
\begin{equation}
	\sigma_N = \sigma(f_N|_{\Omega_B} ).
\end{equation}
The contrast to noise ratio (CNR) is then given by
\begin{equation}
	CNR(f_N) = \frac{\mu_S - \mu_B}{\sigma_N},
\end{equation}
where the numerator denotes the contrast in our data.

\subsubsection{Results on Cone Beam CT Data of the Abdomen}
We tested our method on real Cone Beam CT data of the abdomen (Fig. \ref{fig:dataOverview} B.). The data was acquired using a Philips Allura Xper FC20 system, using a Cone Beam CT backprojection algorithm (XperCT) to generate the final volumetric image.

To quantify our method we segmented the vessels and selected background regions, see Fig. \ref{fig:VesselAndBackgroundRegions}. We then applied CEDOS with different end times and computed the CNR for these different end times ranging from 1 to 6, see Fig. \ref{fig:CEDOSvsGaussian}. Diffusion constants $D_{11}$ and $D_{33}$ were determined using Eqs. \eqref{eq:DiffusionConstantD11} \eqref{eq:DiffusionConstantD33} and we set $D_{11}=0.001$. For the orientation score transformation we used $s_o=\frac{1}{2}(0.45)^2$ and $s_\rho = \frac{1}{2}(16)^2$. For CED we used the following settings: $\alpha=0.2$ and $c$ the $50\%$-quantile of $\kappa$ (see \cite{WeickertIJCV1999}).

As one can expect, in all cases we recognize a peak since initially noise is reduced whereas later also contrast is reduced. Compared to Gaussian diffusion and CED, we reach a higher CNR (Fig. \ref{fig:CEDOSvsGaussianCNR}). Furthermore, in the CEDOS and CED case, the CNR does not decrease much when applying more diffusion making it more robust with respect to choice in diffusion time. The fact that CED does not achieve high CNR-ratios and performs relatively bad in this test is that the diffusion matrix is not designed to reduce background noise (which is used here to quantify noise) but mainly to enhance orientated patterns, for this reason we also set $\alpha$ relatively high to still achieve noise reduction in the background regions.

For 3D visualization of the diffusion results for optimal diffusion time (determined from the CNR) see Fig. \ref{fig:CEDOSvsGaussianPhilipsViewer}. Here we see that compared to Gaussian diffusion our anisotropic diffusion reduced more noise while still maintaining the important structures. A similar thing is achieved by anisotropic diffusion in CED but bifurcating vessels are destroyed by this method, as in this method the diffusion is mainly performed along the orientation of one of the vessels at the bifurcation (see the black circles in Fig. \ref{fig:CEDOSvsGaussianPhilipsViewer}).

\begin{figure*}[h!t]
	\centering
	\includegraphics[width=0.55\textwidth]{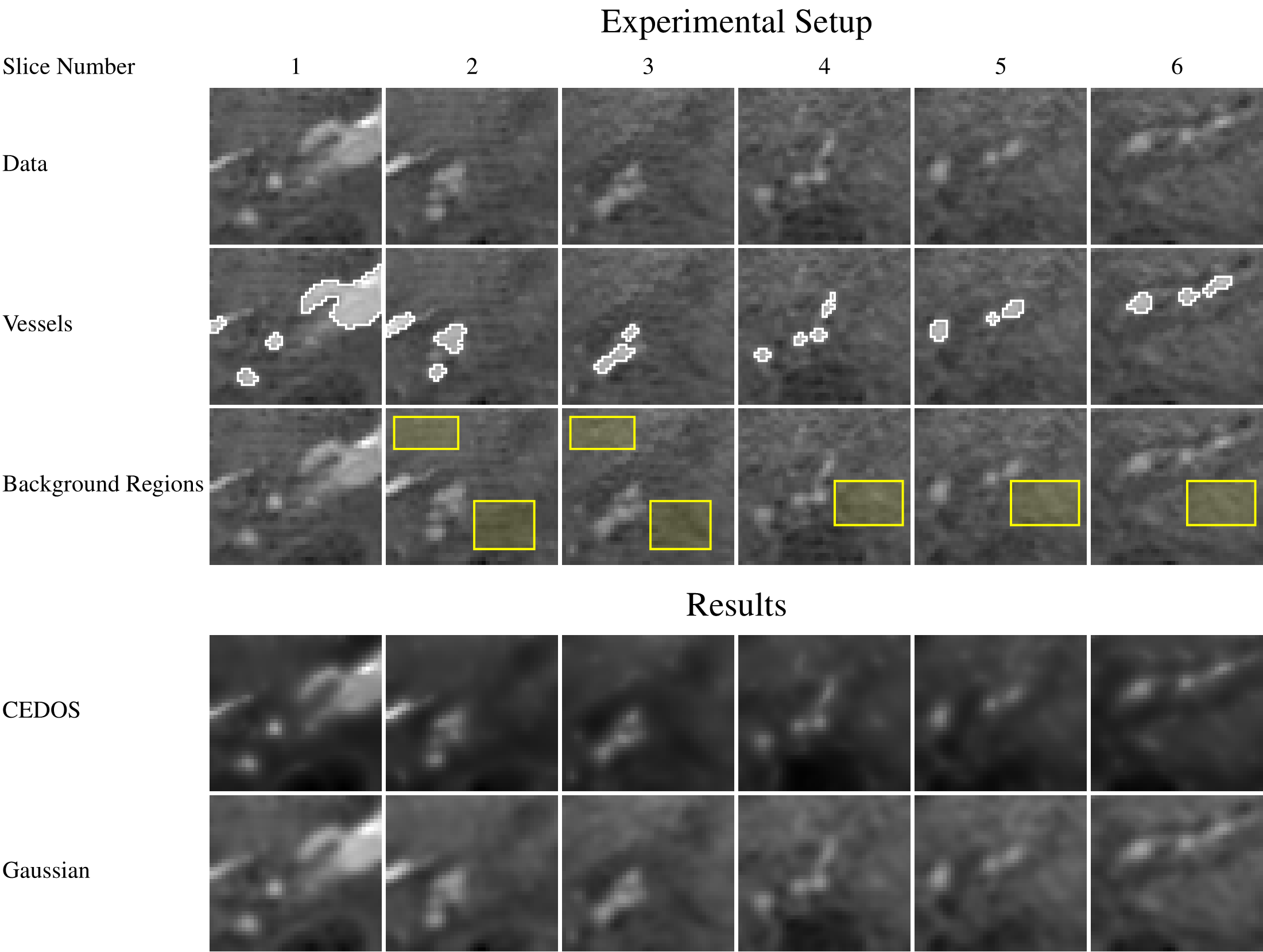}
	\includegraphics[width=0.35\textwidth]{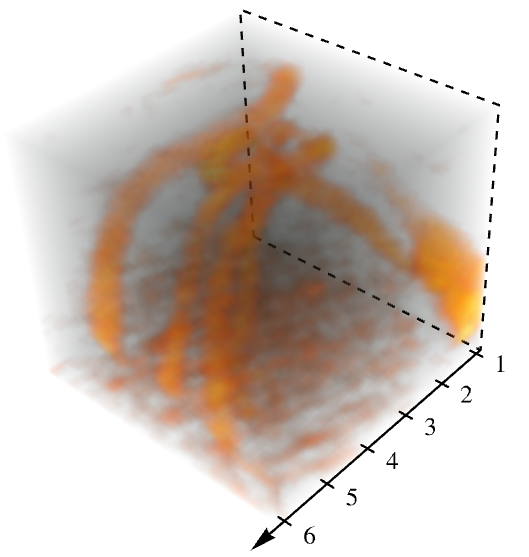}
	\caption{Selected regions for determining the contrast to noise ratios for dataset B.1. \emph{Left}: Grid containing slices of the data (top row), the same slices with segmented vessel parts (second row) and the slices with three selected background regions (third row), the slices after applying CEDOS. \emph{Right}: 3D visualization of the data.}
	\label{fig:VesselAndBackgroundRegions}
\end{figure*}



\begin{figure}[h]
	\centering
	\includegraphics[width=0.9\columnwidth]{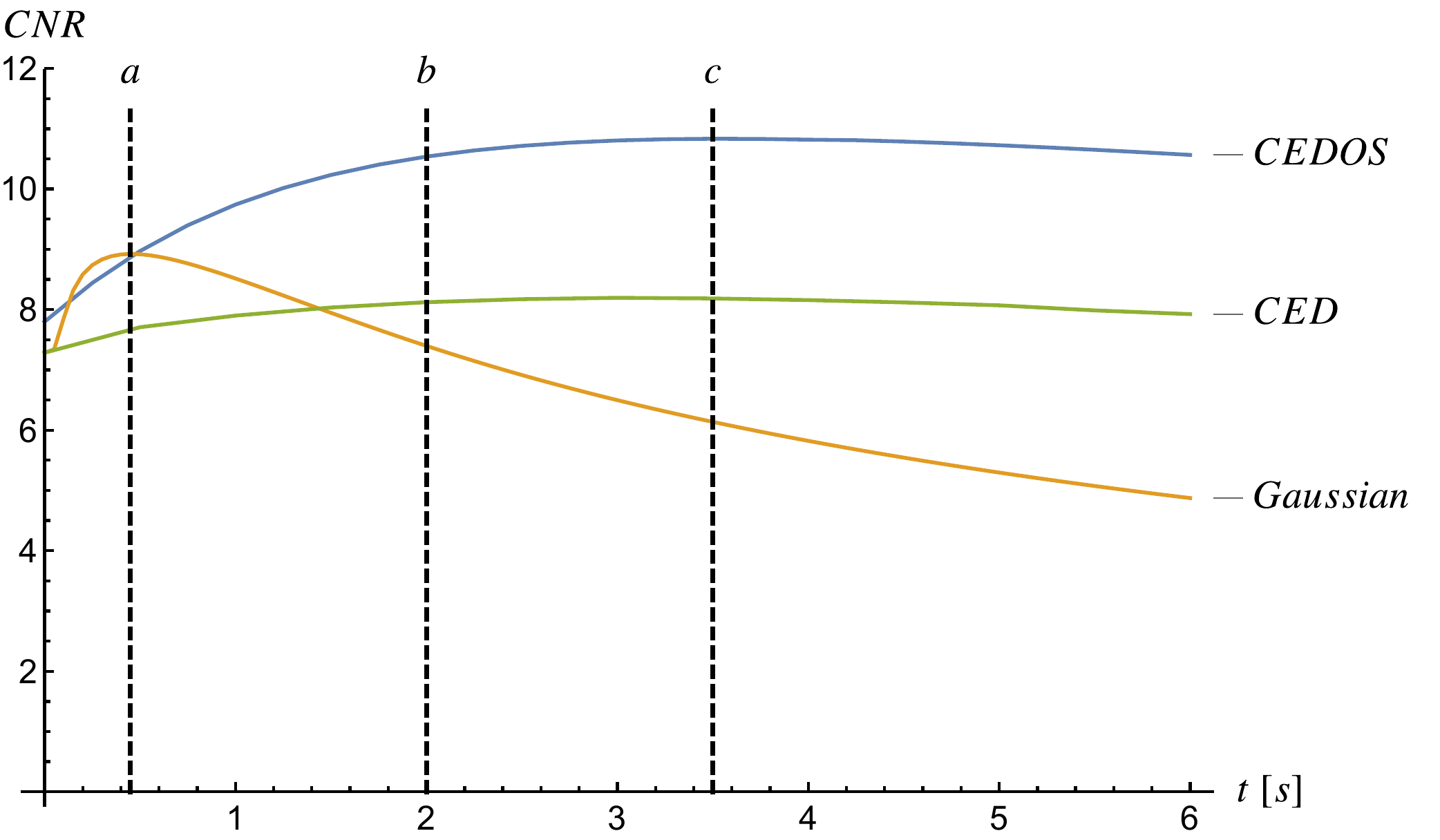}
	\caption{Contrast to noise ratio (CNR) against diffusion time for CEDOS compared to Gaussian regularization and CED of data B.1 depicted in Fig. \ref{fig:VesselAndBackgroundRegions}. The times denoted by a, b and c correspond to the diffusion times shown in Fig. \ref{fig:CEDOSvsGaussian}.}
	\label{fig:CEDOSvsGaussianCNR}
\end{figure}

\begin{figure*}[ht]
	\centering
	\includegraphics[width=0.67\textwidth]{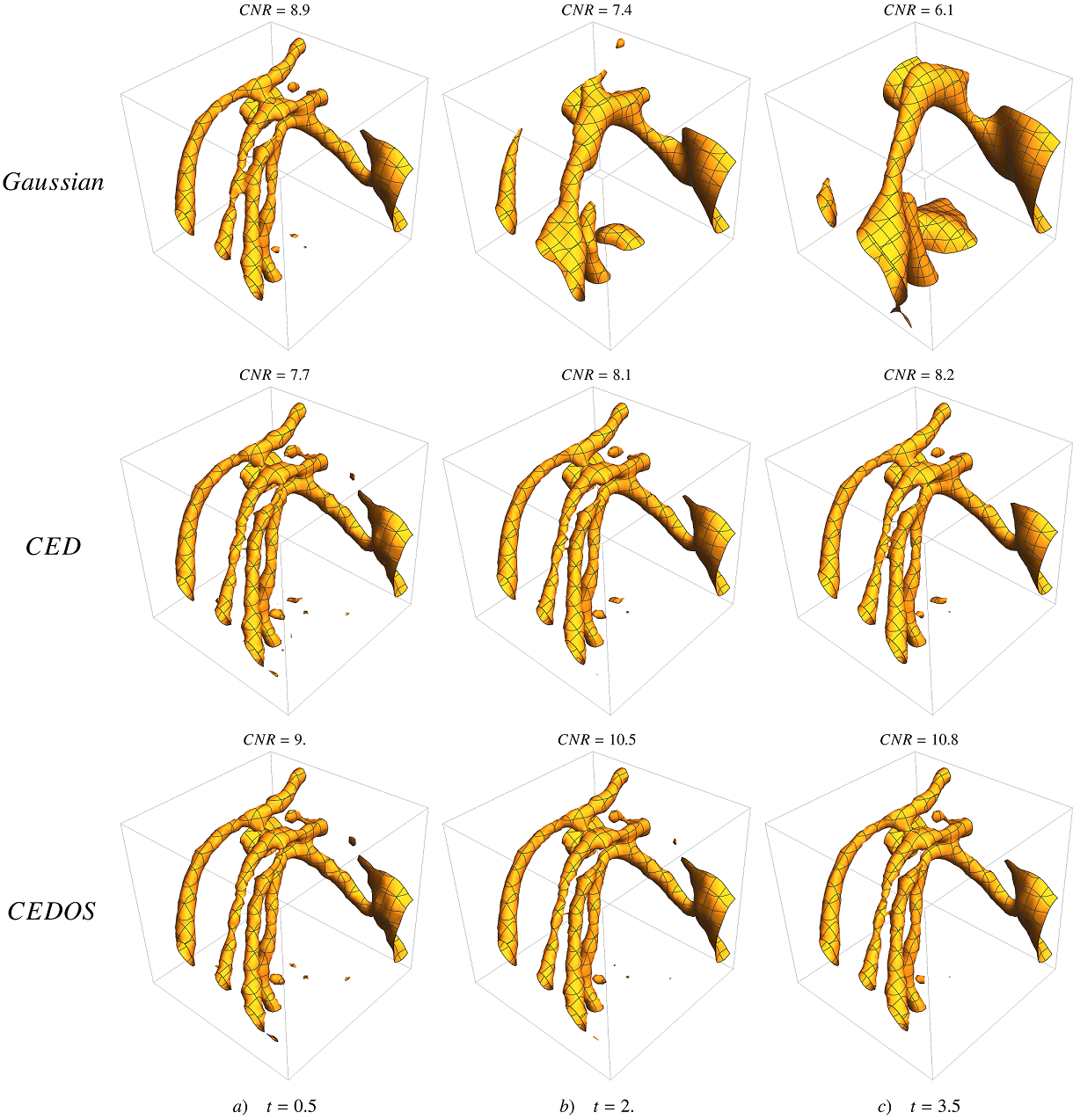}
	\caption{Results of coherence enhancing diffusion via orientation scores on dataset B.1, see Eq. \eqref{eq:CEDOS}. \emph{Top}: CEDOS for different amounts of diffusion time. \emph{Middle}: Result for CED. \emph{Bottom}: Result for isotropic Gaussian regularization. For all datasets, we show one iso-contour at $\mu_\textrm{BG}+ 0.7 (\mu_\textrm{FG}-\mu_\textrm{BG})$, where $\mu_\textrm{BG}$ and $\mu_\textrm{FG}$ are the mean of the background and foreground in the (processed) data determined using the selected regions used for determining the CNR. For a better impression of the full volume see Fig. \ref{fig:CEDOSvsGaussianPhilipsViewer}. We plot results for three different times according to Fig. \ref{fig:CEDOSvsGaussianCNR}, where case a) corresponds to optimal diffusion time for Gaussian regularization and case c) corresponds to optimal CEDOS which is also approximately equal to optimal CED time. We see that CEDOS preserves the complex vascular geometry.}
	\label{fig:CEDOSvsGaussian}
\end{figure*}


\begin{figure*}[p]
	\centering
	\includegraphics[width=0.66\textwidth]{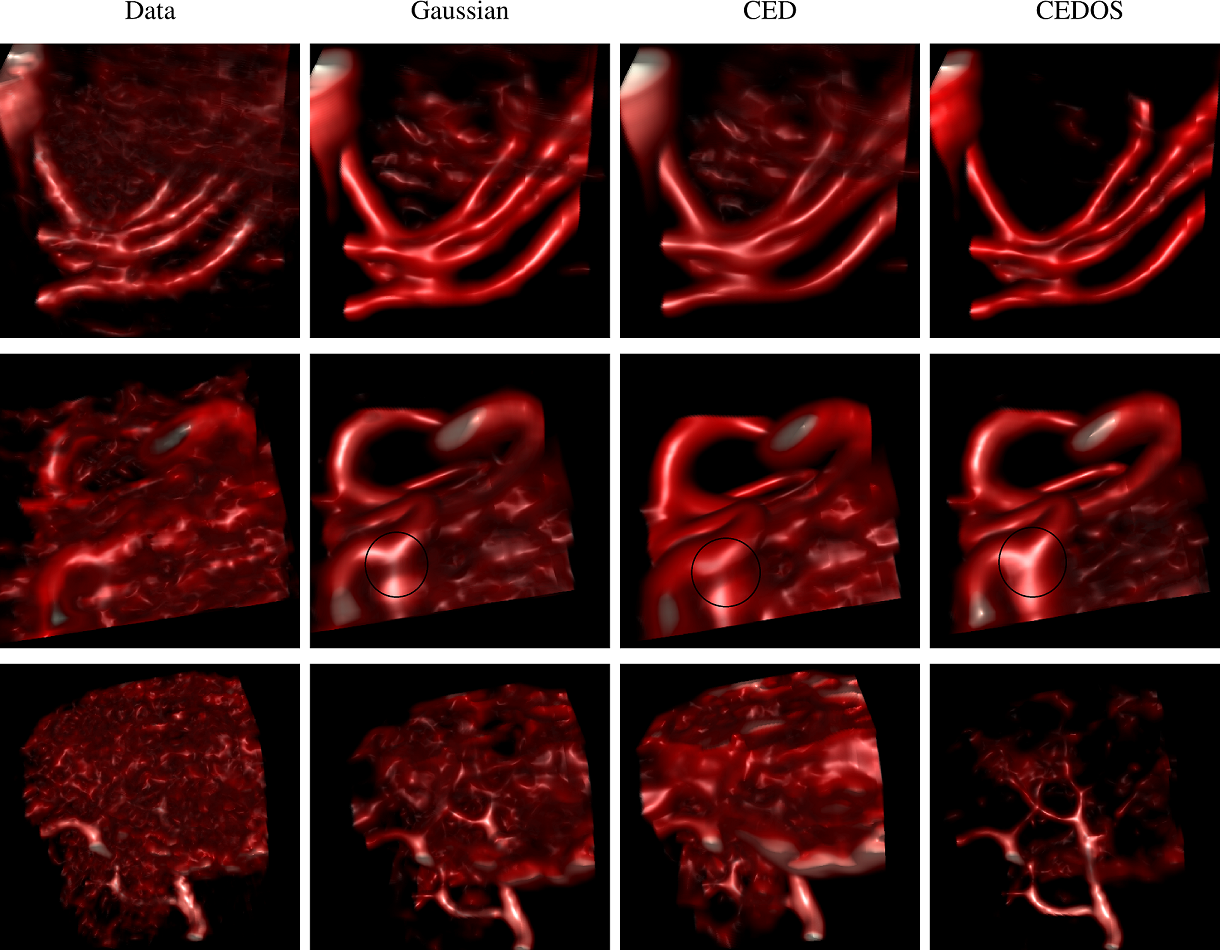}
	\caption{Volume rendering of the diffusion results for datasets B.1, B.2 and B.3 (from top to bottom) visualized in the Philips viewer \cite{Ruijters2006} using default settings in all cases. For all cases we used optimized diffusion time according to Fig. \ref{fig:CEDOSvsGaussianCNR}.}
	\label{fig:CEDOSvsGaussianPhilipsViewer}
\end{figure*}

\subsubsection{Influence of CEDOS on Vessel Edge Location}
In order to test the influence of our regularization method on vessel features we implemented a simple edge detection algorithm. We manually selected positions in the data and detect the edges in the vessel cross-sections by extracting radial profiles from the centerline outwards and looking for the minimum in first order gaussian derivative. Compared to Gaussian regularization, the vessel edge position is not influenced by our regularization method (Fig. \ref{fig:RadiiMeasurements}), which is highly important for applications which rely on accurate vessel lumen measurements, e.g. stent positioning and navigation of endovascular devices. The key explanation for this benefit is that at the vessel we get a very low $D_{11}$ (Eq. \ref{eq:DiffusionConstantD11}) and therefore we smooth only along the vessel.

\begin{figure*}[p]
	\centering
	\includegraphics[width=0.8\textwidth]{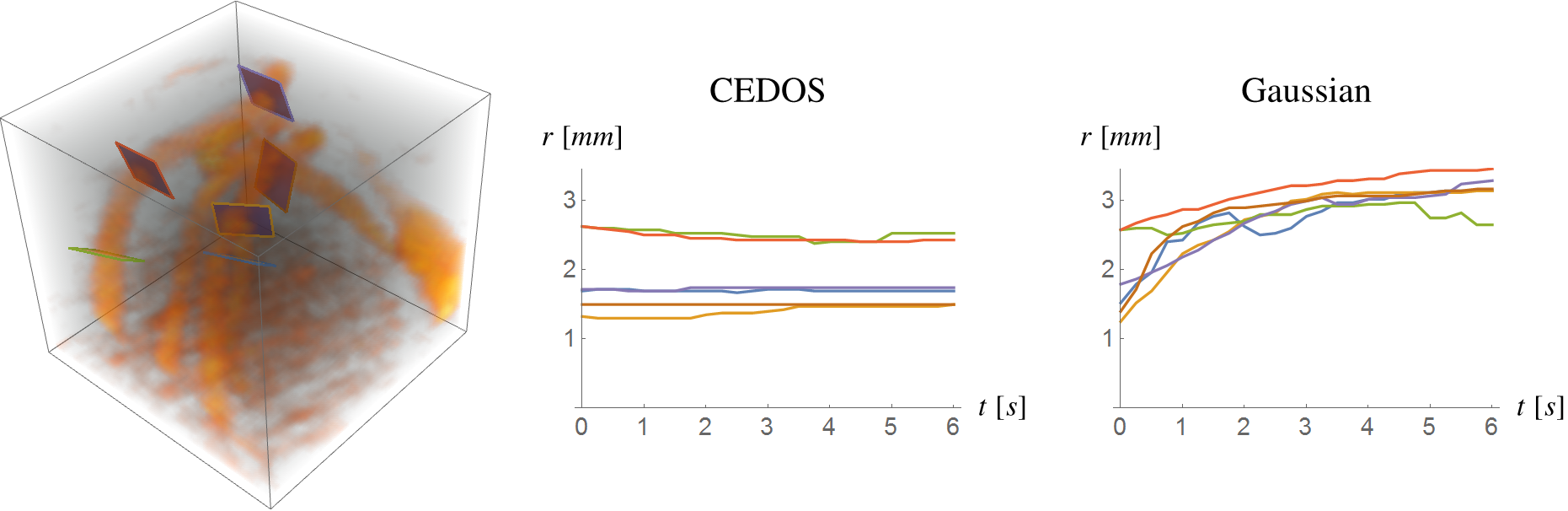}
	\includegraphics[width=0.35\textwidth]{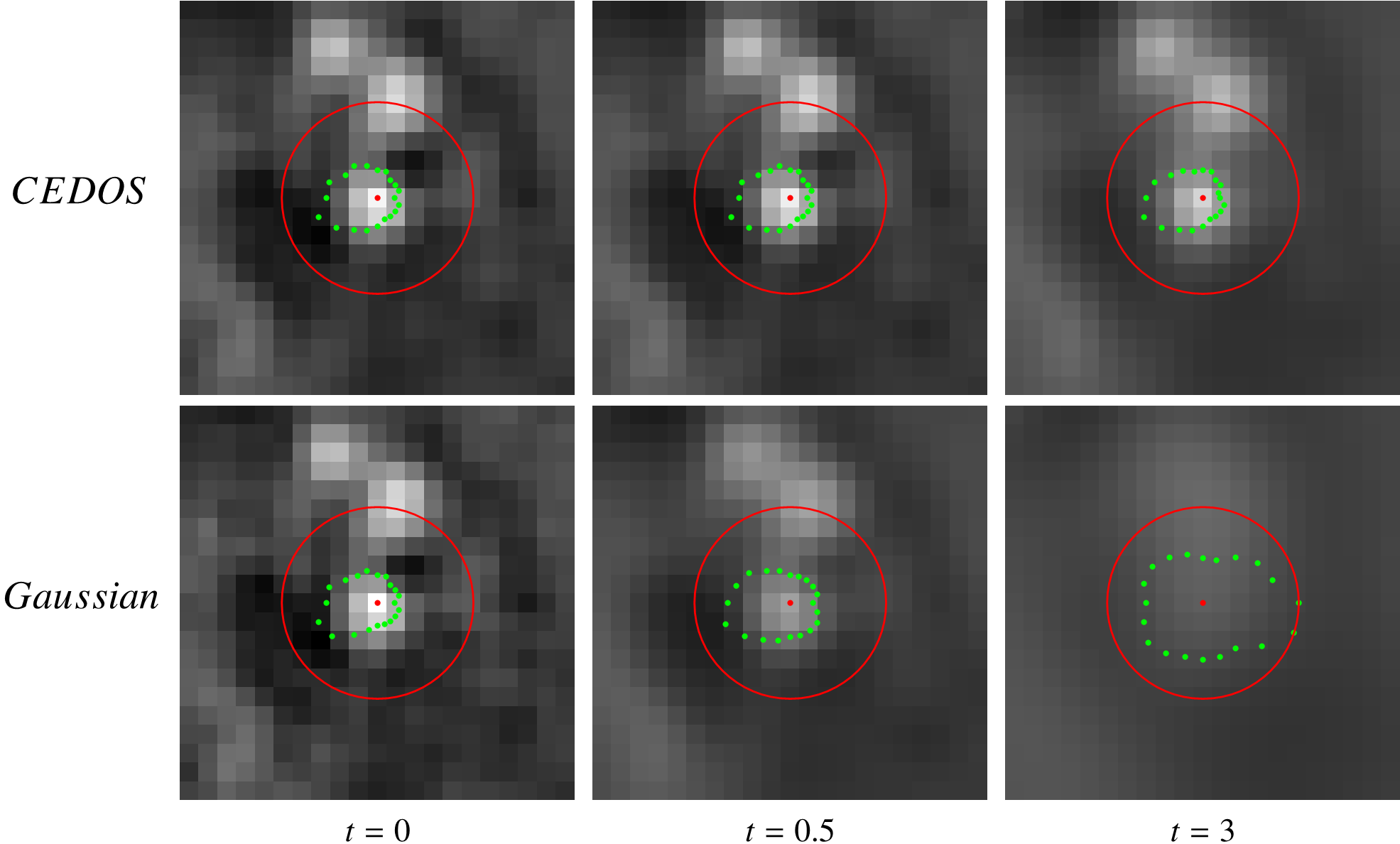}
	\caption{Measurement of vessel radius in vessel cross-sections after different amounts of diffusion in dataset B.1. \emph{Top left:} 3D Visualization of the data with the selected slices. \emph{Top middle:} Radii measurements for increasing diffusion time for CEDOS. \emph{Top right:} Radii measurements for increasing diffusion time for Gaussian regularization. The detected vessel width is not influenced by our regularization method while this does occur for Gaussian regularization. \emph{Bottom:} Cross-sections of one vessel for increasing diffusion time with detected vessel edge positions (green points) and search area for edge detection (red circle).}
	\label{fig:RadiiMeasurements}
\end{figure*}

\subsection{Tubularity Measure}\label{ssect:Tubularity}

\subsubsection{Background and Related Methods}
In this section we propose a tubularity measure based on the edge information in our orientation scores. This tubularity measure is then used for vessel-width measurements and could be used for vessel segmentation.

Tubularity measures are designed to have a high response on the centerline of tubular structures. One such tubularity measure is the image gradient flux filter \cite{BouixMedImageAnal2005} which is used for vessel segmentation. An extension of the gradient flux filter is the optimally oriented flux filter \cite{Law2008} which introduces the notion of oriented flux making the filter orientation sensitive.

A 2D tubularity measure based on orientation scores was proposed in \cite{ThesisBekkers}. The advantage of this tubularity measure is that it included non-linearity and that the implementation via orientation scores still has a response at crossing vessels. Here we propose an extension of this tubularity measure to 3D making use of 3D orientation scores.

\subsubsection{Tubularity via Orientation Scores}

For the tubularity measure we detect edges in the plane perpendicular to orientation $\vn$. Within this plane, the product of two opposite edges at radius $r$ and in-plane orientation $\theta$ at position $\vx\in\R^3$ is given by
\begin{multline} \label{tm}
	E_{\textrm{prod}}(\vx, \vn, r, \theta) = \Imagpos [U(\vx + r \vn^\bot(\theta), \vn^\bot(\theta))] \cdot \\   \Imagpos [U(\vx - r \vn^\bot(\theta), - \vn^\bot(\theta))],
\end{multline}
where $\Imagpos(z) = \max\{0,\Imag(z)\}$ and $\vn^\bot(\theta) = \cos \theta \,\ve_1 + \sin \theta \,\ve_2$, with $\{\ve_1,\ve_2\}$ an orthogonal basis for the orthogonal complement of $\langle \vn \rangle=\operatorname{span}\{\vn\}$.
The product of the two $\Imagpos$ edge responses in Eq.\!~(\ref{tm}) yields a better performance than taking the sum, as is done in \cite[Fig. 12.2]{ThesisBekkers} and \cite{Chen2014}.
The idea behind taking the product instead of the sum is that we need a high edge response in both directions.
See Fig. \ref{fig:TubularityExplanationSchematic} for a schematic visualization. Since for a real tube all edges should be present, and we do no want any response for e.g. plate structures, we take a minimum over the perpendicular orientations parametrized by $\theta$. To make the method more robust to deviations from exactly tubular structures, allowing slightly elliptical cross sections, we apply radial regularization and to make the method more robust to missing edge pieces we apply angular regularization before taking the minimum:
 \begin{multline}\label{eq:TubularityDefinition}
	V(\vx, \vn, r) = \min \limits_{\theta\in[0,2\pi)} \int \limits_{0}^\infty \int \limits_{0}^\pi K^\textrm{or}(\theta - \theta') \, K^\textrm{rad}(r, r') \cdot \\
	E_{\textrm{prod}}(\vx, \vn, r', \theta') \d \theta' \d r',
\end{multline}
where we deliberately do not add the Jacobian $(r' \d \theta' \d r')$ as higher radii should not gain a higher weight. For the kernel on orientation we use the diffusion kernel $K^\textrm{or} = G_{S^1}^{\sigma_{o}^{V}}$ and for $K^\textrm{rad}=K^\textrm{rad}_{\sigma_r}$ the radial regularization kernel given by
\begin{equation}\label{eq:TubularityRadialRegularizationKernel}
	K^\textrm{rad}_{\sigma_r}(r, r') = \frac{1}{\sqrt{2 \pi }  \sigma_r r'} e^{-\frac{\log ^2\left(\frac{r}{r'}\right)}{2 \sigma_r ^2} -\frac{\sigma_r ^2}{2}},
\end{equation}
which is normally used for temporal smoothing \cite{KoenderinkBiolCybern1988,LindebergJMIV2017}. We use this kernel since our radius has a similar one-sided domain $(0, \infty)$, requiring similar scaling relations. From the tubularity measure we extract the following features:
\begin{align}\label{eq:TubularityFeatures}
	s^t(\vx)	&= \max \limits_{\vn \in S^2, r \in \R^+} V(\vx, \vn, r), \\
	\vn^*(\vx)	&= \argmax \limits_{\vn \in S^2} \; \max \limits_{r \in \R^+} V(\vx, \vn, r), \\
	r^*(\vx)	&= \argmax \limits_{r \in \R^+} \; \max \limits_{\vn \in S^2} V(\vx, \vn, r).
\end{align}
Here $s^t(\vx)$ is the tubularity confidence which is a measure for how certain we are at least one tubular structure is present at position $\vx$. The features $\vn^*(\vx)$ and $r^*(\vx)$ are the orientation and radius of optimal tubularity response at position $\vx$.

\subsubsection{Results on Artificial Data}
For the validation of our tubularity measure we constructed 18 artificial datasets with a random tubular structure with randomly varying radius (Fig. \ref{fig:dataOverview} C.). For the tubularity measure we used the following settings: $\sigma_{o}^{V} = \pi/8,\sigma_r = 0.3$ and we discretized the $\theta$-integral in Eq. \eqref{eq:TubularityDefinition} using 8 orientations and radius $r$ from $1$ to $10$ pixels in steps of $0.5$ pixels.

As validation we compared the optimal radius to the ground-truth radius and inspect the tubularity confidence. The tubularity confidence is selective on the vessel centerline and the found optimal radius is a good estimation of the ground-truth radius, see Fig. \ref{fig:TubularityResultRadiusArtificialData} and Fig. \ref{fig:TubularityResultRadiusArtificialDataAll}.

\subsubsection{Results on 3D Rotational Angiography data of the Brain in Patients with Arteriovenous Malformation (AVM)}
We applied the tubularity measure to 3D Rotational Angiography of the Brain in Patients with Arteriovenous Malformation (Fig. \ref{fig:dataOverview} A.). The data was acquired using a Philips Allura Xper FC20 system, using a 3D Rotational Angiography backprojection algorithm (3DRA) to generate the final volumetric image. For the tubularity measure we used the following settings: $\sigma_{o}^{V} = \pi/8,\sigma_r = 0.3$ and we discretized the $\theta$-integral in Eq. \eqref{eq:TubularityDefinition} using 12 orientations and radius $r$ from $0.3$ to $5$ pixels in steps of $0.3$ pixels.
For the orientation score transformation we used $s_o=\frac{1}{2}(0.4)^2, s_\rho = \frac{1}{2}(5)^2$ and evaluated on a grid of $21 \times 21 \times 21$ pixels and default settings for the other parameters.

In Fig. \ref{fig:TubularityResultRealData} we show our tubularity measure for this medical data. The tubularity measure gives sharp responses for vessel centerlines. We also show optimal orientation $\vn^*(\vx)$ and a simple segmentation given by the 0-isocontour of the distance map $d(\vx,\cup B_{\vc_i,r^*(\vc_i)})=\min_{\vc_i} \left\{ \| \vx-\vc_i \|-r^*(\vc_i) \right\}$, where $\vc_i$ are the positions given by the $1\%$ quantile of highest responses.

\begin{figure*}[hp!]
	\centering
	\includegraphics[width=0.3 \textwidth]{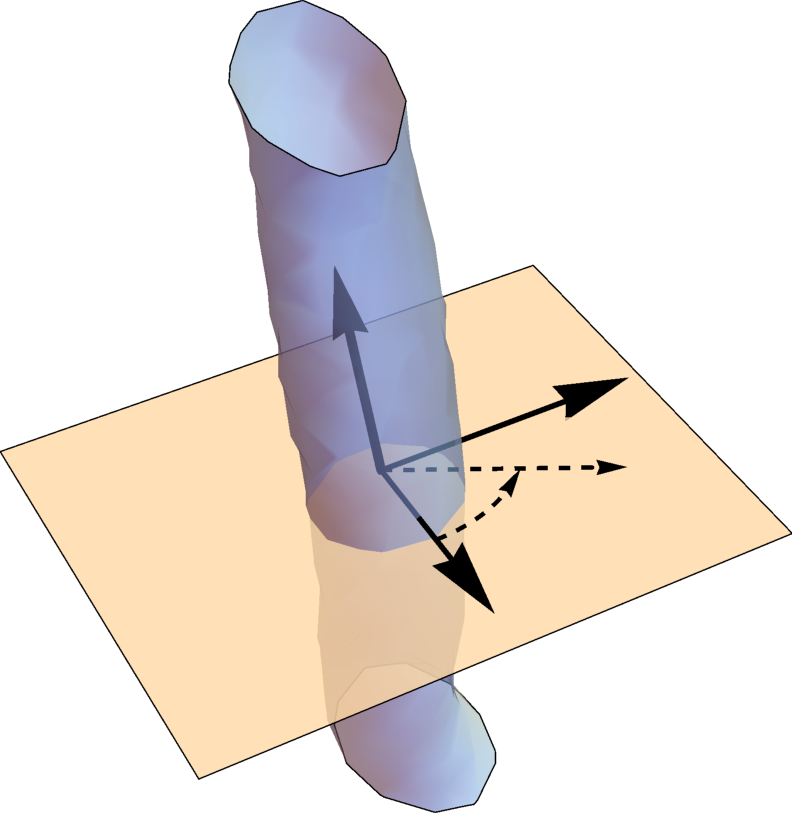}
	\includegraphics[width=0.4 \textwidth]{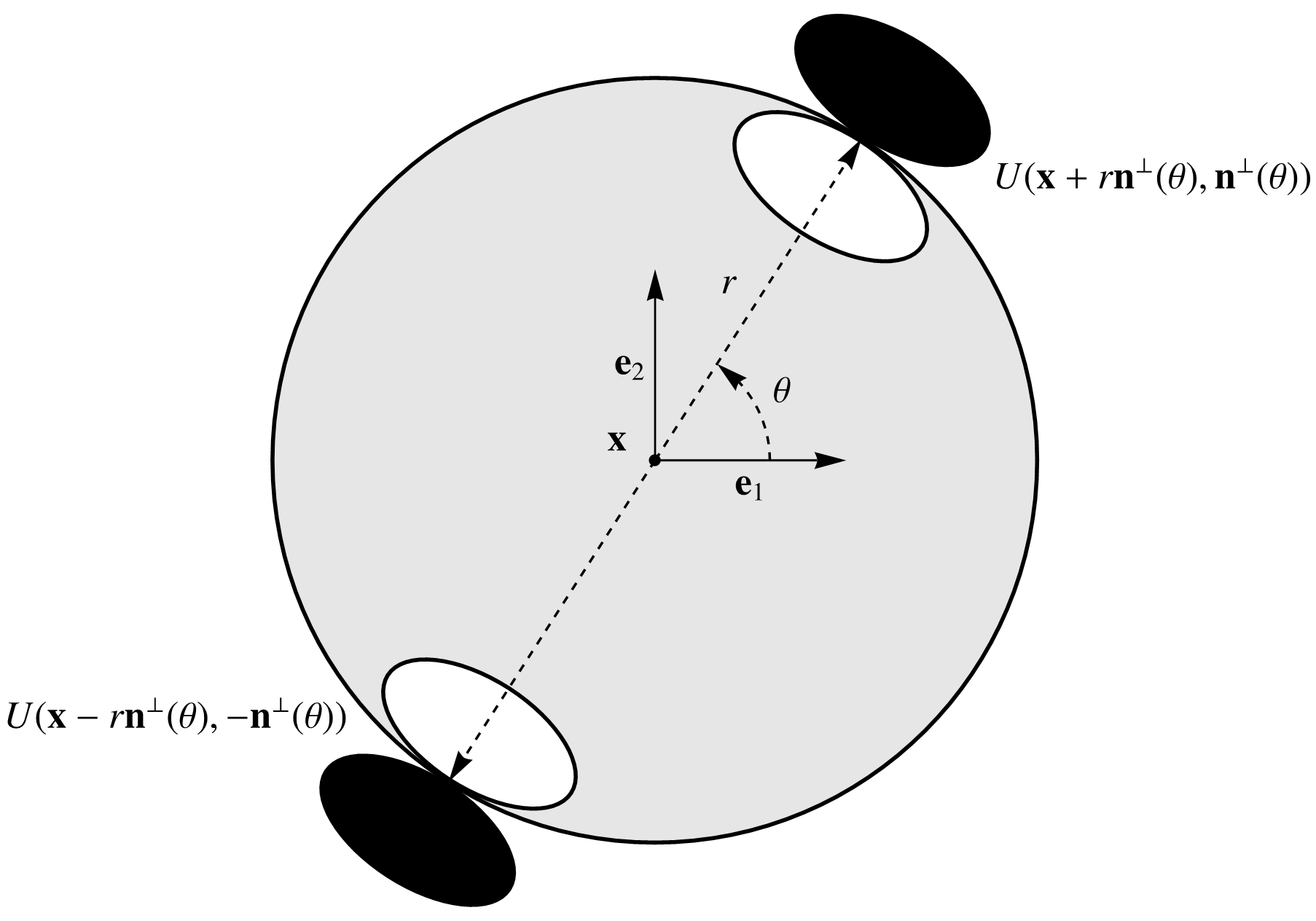}
	\caption{Schematic visualization of the edges used in the tubularity measure $V(\vx, \vn, r)$. \emph{Left:} A 3D iso-contour visualization of a vessel with orientation $\vn$. The coordinates $(\theta,r)$ are polar coordinates for the plane perpendicular to orientation $\vn$ spanned by $\ve_1$ and $\ve_2$. \emph{Right:} In this plane opposite edges are multiplied in $E_{\textrm{prod}}(\vx, \vn, r, \theta)$. The edge is expected to have outwards orientation $\vn^\perp (\theta)$.}
	\label{fig:TubularityExplanationSchematic}
\end{figure*}

\begin{figure*}[hp!]
	\centering
	\includegraphics[width=0.7 \textwidth]{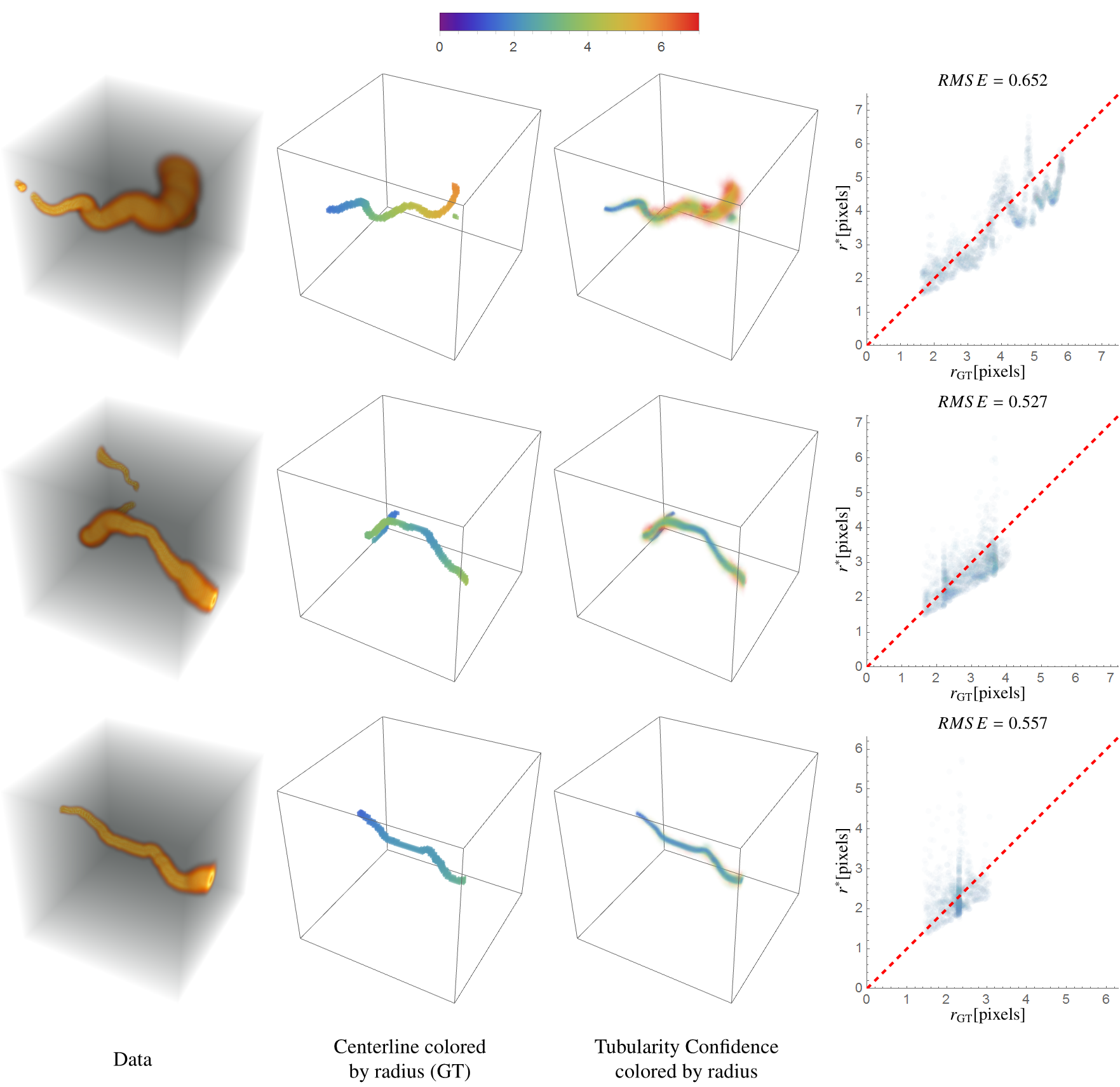}
	\caption{Tubularity on artificial datasets and comparison of measured radius against ground truth radius for datasets C.1, C.2 and C.3 (top to bottom). \emph{Left:} The data. \emph{Middle left:} The centerline with ground truth radius in color. \emph{Middle right:} The tubularity confidence $s^t(\vx)$  (max of tubularity over radius and orientation) with radius of max response $r^*(\vx)$ in color. \emph{Right:} Measured radius $r^*(\vx)$ against ground truth radius $r_\textrm{GT}$ on the ground truth centerline. The opacity of the plotted points is linearly scaled with the tubularity confidence.}
	\label{fig:TubularityResultRadiusArtificialData}
\end{figure*}

\begin{figure}[hb]
	\centering
	\includegraphics[width=0.7 \columnwidth]{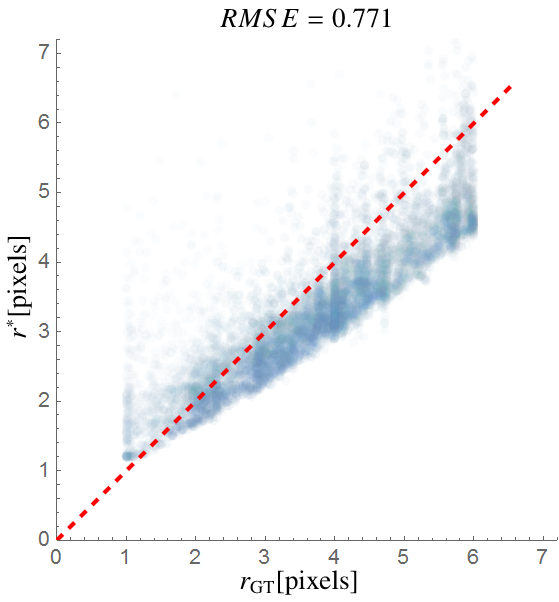}
	\caption{Measured radius $r^*(\vx)$ against ground truth radius on the ground truth centerline for all 18 datasets. The opacity of the plotted points is linearly scaled with the tubularity confidence.}
	\label{fig:TubularityResultRadiusArtificialDataAll}
\end{figure}

\begin{figure*}[htb!]
	\centering
	\includegraphics[width=0.8 \textwidth]{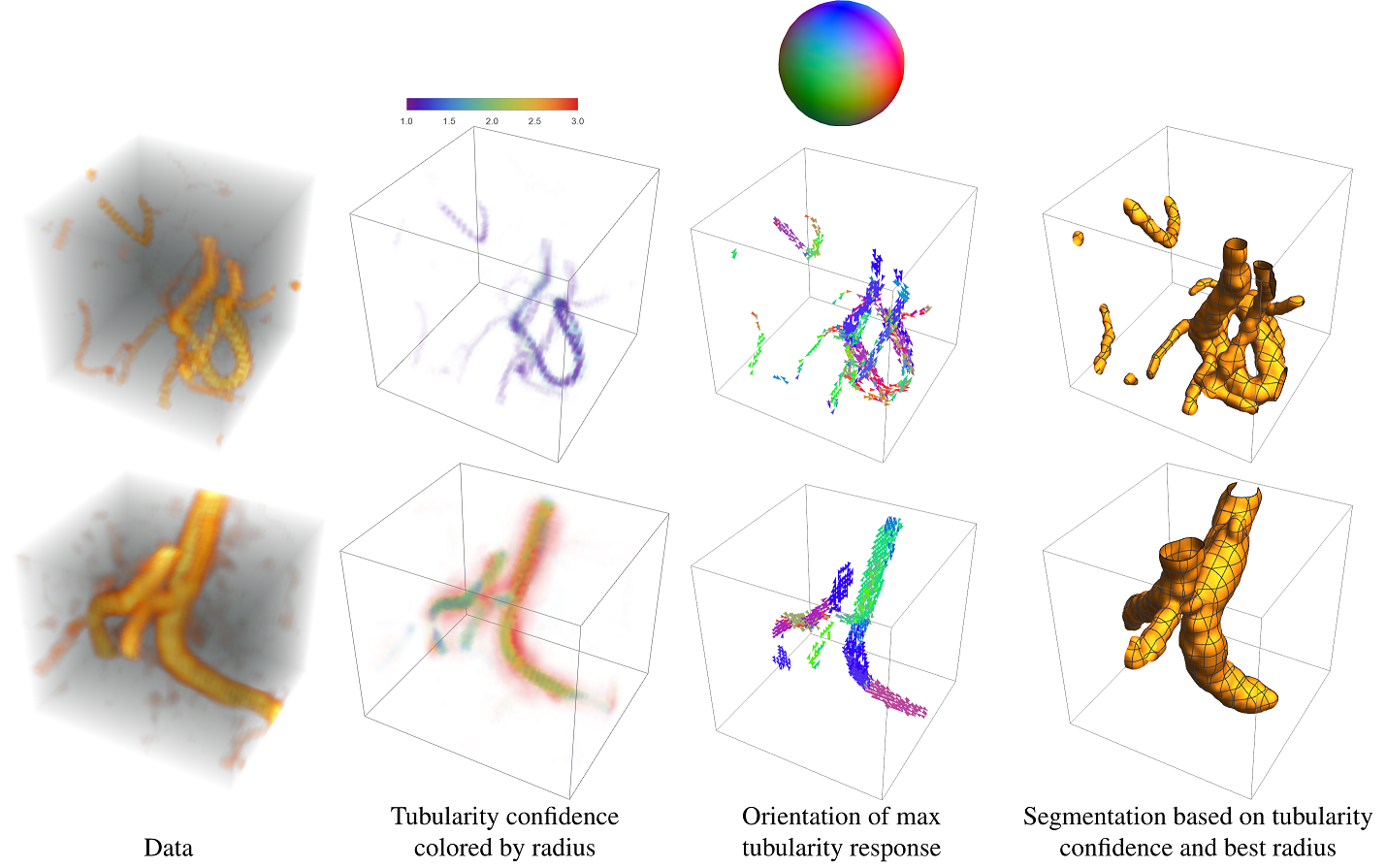}
	\caption{Tubularity on real data. \emph{Top:} Dataset A.1. \emph{Bottom:} Dataset A.2. \emph{Left:} The data. \emph{Middle left:} The projected tubularity (max over radius and orientation) colored by radius. \emph{Middle right:} Orientation of max response for $1\%$ quantile of highest responses in the tubularity confidence. \emph{Right:} Segmentation based on radius of max response for $1\%$ quantile of highest responses in the tubularity confidence. The plotted surface is the 0-isocontour of the distance map $d(\vx,\cup B_{\vc_i,r^*(\vc_i)})=\min_{\vc_i} \left\{ \| \vx-\vc_i \|-r^*(\vc_i) \right\}$, where $\vc_i$ are the positions given by the $1\%$ quantile of highest responses.}
	\label{fig:TubularityResultRealData}
\end{figure*}

\section{Conclusion}
We presented theory and filters for the 3D orientation score transformation which is valuable in handling complex oriented structures in 3D image data. Then we showed applications of this transformation.

First we proposed filters for a 3D orientation score transform. We presented two types of filters, the first uses a discrete Fourier transform, the second is designed in the 3D generalized Zernike basis which allowed us to find analytical expressions for the spatial filters. Both filters allowed for an invertible transformation. The filters and the quality of the reconstruction were assessed in Section \ref{sect:Experiments}, where we showed that the discrete filters approximate their analytical counterparts well. We also verified the invertibility of our transformation by showing data reconstructions of real medical data.

The orientation score transform was then used in two different applications. In the first we presented an extension of coherence enhancing diffusion via 3D orientation scores which we applied to real 3D medical data and showed our method effectively reduced noise while maintaining the complex vessel geometry. In the second application we propose a new non-linear tubularity measure via 3D orientation scores. The tubularity measure has sharp responses for vessel centerlines and we showed its use in radius extraction and complex vessel segmentation.

In this work basic applications of the tubularity measure are shown. Future work would include using the tubularity measure in more advanced vessel segmentation procedures. Furthermore many other applications exist for the 3D orientation scores and many techniques developed for 2D orientation scores can now be extended to 3D. First steps are presented in this paper, where the extension of 2D CEDOS and a 2D tubularity measure via 2D orientation scores are given.

\section*{Acknowledgements}

We would like to thank professors Jacques Moret and Laurent Spelle from the NEURI center, Department of NRI, Bic\^{e}tre University hospital Paris, France for providing the AVM dataset used in this article.

The research leading to the results of this article has received funding from the European Research Council (ERC) under the European Community’s 7th Framework Programme (FP7/20072014)/ERC grant agreement No. 335555 (Lie Analysis).

\appendix
\section{Steerable Orientation Score Transform} \label{app:steerable}
In this article we often rely on spherical harmonic decomposition of the angular part of proper wavelets in spatial and Fourier domain. As the choice of radial basis varies in Sec.~\ref{sect:WaveletZernike}, and since it does not affect steerable filter \cite{FreemanPAMI1991,ThesisFranken,ThesisReisert,ThesisAlmsick} properties, we simply write $\{g_{n}^{l}(r)\}_{n=0}^{\infty}$ for a radial orthonormal basis for $\mathbb{L}_{2}(\mathbb{R}^+, r^2 {\rm d}r)$ associated to each $2l+1$-dimensional $SO(3)$-irreducible subspace 
indexed by $l$. Then by the celebrated Bochner-Hecke theorem \cite{Faraut1987} this induces a corresponding orthonormal radial basis $\{\tilde{g}_n^l(\rho)\}_{n=0}^{\infty}$ in the Fourier domain which can be obtained by a Hankel-type of transform \cite[Ch.~7]{ThesisDuits}. We expand our wavelets in spherical harmonics $Y_{l}^{m}$ and ball-coordinates (cf. (\ref{eq:ballcoordinates})) accordingly:
\begin{equation}
\begin{array}{l}
\psi(\vx)= \sum \limits_{n=0}^{\infty} \alpha_{l}^{n} \, g^{l}_n(r)\; Y^{0}_{l}(\theta,\phi), \\
\hat{\psi}(\vomega)= \sum \limits_{n=0}^{\infty} \alpha_{l}^{n} \, \tilde{g}^{l}_n(\rho)\; Y^{0}_{l}(\vartheta,\varphi).
\end{array}
\end{equation}
\begin{remark}
In Section~\ref{sect:WaveletZernike}, we consider the modified Zernike basis in which case $g^{l}_n$ and $\tilde{g}^{l}_n$ are given by respectively by (\ref{eq:GeneralizedZernikeFunctionsOrthogonality}) and (\ref{eq:GeneralizedZernikeFunctions}), whereas for the harmonic oscillator basis one has $g^{l}_{n}= i^l (-1)^{n+l} \tilde{g}^{l}_n$ given respectively by (\ref{eq:WaveletHarmonicOscillator}) and (\ref{eq:WaveletHarmonicOscillator2}).
\end{remark}
We obtain steerability via finite series truncation at $n=N$ and $l=L$.
Then we rotate the steerable kernels via the Wigner-D functions $D^{l}_{0,m}(\gamma,\beta,0) \in \R$ (cf. Subsection~\ref{sssect:rotate}) and one obtains the following steerable
implementations of orientation scores:
{\small
\begin{equation}
\begin{array}{l}
U(\vx,\vn)= \\ \sum \limits_{n=0}^{N} \sum \limits_{l=0}^{L} \sum \limits_{m=\!-l}^l \overline{ \alpha_{l}^{n} \, D^{l}_{0,m}(\gamma,\beta,0) } \cdot
( (g_{n}^l \otimes Y_{l}^{m}) \star f)(\vx)= \\
\sum \limits_{n=0}^{N} \sum \limits_{l=0}^{L} \sum \limits_{m=-\!l}^l \overline{ \alpha_{l}^{n} \, D^{l}_{0,m}(\gamma,\beta,0) } \cdot
\mathcal{F}^{-1}\left[\overline{ \tilde{g}_{n}^l \otimes Y_{l}^{m} } \cdot  \hat{f}\right](\vx)
\end{array}
\end{equation}
}%
where $\vn=(\cos \gamma \sin \beta, \sin \gamma \sin \beta, \cos \beta)^T$, $\star$ denotes correlation, the overline denotes complex conjugation and $\otimes$ the function product, i.e., $(\tilde{g}_{n}^l \otimes Y_{l}^{m})(\vx) = \tilde{g}_{n}^l(r)\; Y_{l}^{m}(\theta,\phi)$.

\begin{table*}
\renewcommand{\arraystretch}{1.3}
\footnotesize
\section{Table of Notations \label{app:TableOfNotations}}
\begin{tabular}{|p{6.8cm}|p{7cm}|r|}
\hline
\textbf{Symbol}
& \textbf{Explanation}
& \textbf{Reference}
\\ \hline \hline
\multicolumn{3}{c}{} \vspace{-3mm} \\
\multicolumn{3}{c}{B.1 Spaces and Input Data}
\\ \hline
$\mathbb{R}^{3} \times S^{2}$
& Space of positions and orientations
& Page 1
\\ \hline
$\LL_2^{\varrho} (\R^3)=\{f \in \LL_2 (\R^3)| \textrm{supp}(\cF f) \subset B_{\varrho}\}$, with $\varrho>0$
& Space of frequency ball-limited images
& \eqref{eq:ballLimitedData}
\\ \hline
$B_{\rho}=\{\vx \in \R^3 \big| \| \vx \| < \rho \}$, with $\varrho>0$
& The ball of radius $\rho$
& \eqref{eq:ballLimitedData}
\\ \hline
\multicolumn{3}{c}{} \vspace{-3mm} \\
\multicolumn{3}{c}{B.2 Orientation Score Transformation}
\\ \hline
$\cW_\psi$
& Orientation score transformation
& \eqref{eq:Construction1}
\\ \hline
$\cW_\psi^{-1}$
& Data reconstruction via the inverse orientation score transformation
& \eqref{eq:Reconstruction1}
\\ \hline
$\psi$
& Wavelet used for the orientation score transformation
& \eqref{eq:Construction1}
\\ \hline
$\hat{\psi}$
& Fourier Transform of the wavelet used for the orientation score transformation
& \eqref{eq:Mpsi}
\\ \hline
$M_{\psi}$
& Factor used to quantify the stability of the transformation
& \eqref{eq:Mpsi}
\\ \hline
$N_{\psi}$
& Factor used to quantify the stability of the transformation when using the simplified reconstruction by integration
& \eqref{eq:Npsi}
\\ \hline
$\cW_\psi^d, (\cW_\psi^d)^{-1}, M_{\psi}^d, N_{\psi}^d$
& Discretized versions of $\cW_\psi, (\cW_\psi)^{-1}, M_{\psi}, N_{{}\psi}$
& \eqref{eq:construction1Discrete},\eqref{eq:Reconstruction1Discrete},\eqref{eq:MpsiDiscrete},\eqref{eq:NpsiDiscrete}
\\ \hline
$\d \sigma$ and $\Delta_i$
& Spherical area measure and discretized spherical area measure
& \eqref{eq:Reconstruction1},\eqref{eq:Reconstruction1Discrete}
\\ \hline
\multicolumn{3}{c}{} \vspace{-3mm} \\
\multicolumn{3}{c}{B.3 Coordinates}
\\ \hline
$\vx=(x,y,z)$
& Cartesian coordinates real space
& -
\\ \hline
$\vomega=(\omega_x,\omega_y,\omega_z)$
& Cartesian coordinates Fourier domain
& -
\\ \hline
$(r,\theta,\phi), \quad  \vx= (r \sin \theta \cos \phi,r \sin \theta \sin \phi,r \cos \theta)$
& Spherical coordinates real space
& \eqref{eq:CoordinatesFourierSpace}
\\ \hline
$(\rho,\vartheta,\varphi),\quad \vomega = (\rho \sin \vartheta \cos \varphi,\rho \sin \vartheta \sin \varphi,\rho \cos \vartheta)$
& Spherical coordinates Fourier domain
& \eqref{eq:CoordinatesFourierSpace}
\\ \hline
\multicolumn{3}{c}{} \vspace{-3mm} \\
\multicolumn{3}{c}{B.4 Wavelets}
\\ \hline
$g$
& Radial function of the cake filters
& \eqref{eq:radialFunctionOfPsi}
\\ \hline
$A$
& Orientation distribution used in wavelet construction
& \eqref{eq:orientationDistribution}
\\ \hline
$Y_l^m$
& Spherical Harmonics
& \eqref{eq:definitionSphericalHarmonics}
\\ \hline
$Z_{n,l}^{m,\alpha}$
& 3D Generalized Zernike Functions
& \eqref{eq:GeneralizedZernikeFunctions}
\\ \hline
$R_n^{l,\alpha}$
& Radial part of the 3D Generalized Zernike Functions
& \eqref{eq:RadialPartGeneralizedZernikeFunctions}
\\ \hline
$S_{n,l}^{\alpha}$
& Radial part of the inverse Fourier transformed 3D Generalized Zernike Functions
& \eqref{eq:radialPartZernikeSpatialDomain}
\\ \hline
\multicolumn{3}{c}{} \vspace{-3mm} \\
\multicolumn{3}{c}{B.5 Applications}
\\ \hline
$V(\vx, \vn, r)$
& Tubularity measure
& \eqref{eq:TubularityDefinition}
\\ \hline
$s^t(\vx)$
& Tubularity confidence
& \eqref{eq:TubularityFeatures}
\\ \hline
$\vn^*(\vx)$
& Orientation of maximum tubularity response
& \eqref{eq:TubularityFeatures}
\\ \hline
$r^*(\vx)$
& Radius of maximum tubularity response
& \eqref{eq:TubularityFeatures}
\\ \hline

\end{tabular}
\end{table*}

\bibliography{Library}{}

\bibliographystyle{spmpsciSimple}
\end{document}